\documentclass[10pt]{article}
\usepackage[margin=1in]{geometry}

\usepackage[bookmarks, colorlinks=true, plainpages = false, citecolor = blue,linkcolor=red,urlcolor = blue, filecolor = blue]{hyperref}
\usepackage{url}
\usepackage{amsmath,amsfonts,amsthm,amssymb,bm, verbatim,dsfont,mathtools}
\usepackage{algorithm,algorithmic,color,graphicx,appendix}
\usepackage{subfig}
\usepackage{etoolbox,tikz,xspace,todonotes,paralist,adjustbox}
\usepackage{tikz}
\usepackage{xr,xspace}
\usepackage{todonotes}
\usepackage{paralist}

\theoremstyle{plain}
\newtheorem{theorem}{Theorem}[section]
\newtheorem{lemma}{Lemma}
\newtheorem{proposition}[theorem]{Proposition}
\newtheorem{corollary}[theorem]{Corollary}
\newtheorem{conjecture}[theorem]{Conjecture}

\newtheorem{definition}{Definition}
\newtheorem{remark}[theorem]{Remark}
\newtheorem*{remark*}{Remark}

\renewenvironment{compactitem}
{\begin{itemize}}
{\end{itemize}}

\newcommand\gtrsom{\overset{\bm{.}}{\gtrsim}}
\newcommand\lesssom{\overset{\bm{.}}{\lesssim}}

\DeclareMathOperator{\polylog}{polylog}

\global\long\def\R{\mathcal{R}}
\global\long\def\A{\mathcal{A}}

\renewcommand{\hat}{\widehat}
\renewcommand{\tilde}{\widetilde}

\newcommand{\calG}{{\mathcal{G}}}

\begin{document}

\title{Statistical-Computational Tradeoffs in Planted Problems and Submatrix Localization with a Growing Number of Clusters and Submatrices%
\footnote{
This manuscript is accepted to the Journal of Machine Learning Research (JMLR) conditioned on minor revisions.
Partial results appeared at the International Conference on Machine Learning (ICML) 2014.}}

\date{}

\author{Yudong Chen\thanks{Y. Chen is
 with the Department of EECS, University of California, Berkeley. Email: \texttt{yudong.chen@eecs.berkeley.edu}.} \and Jiaming Xu\thanks{
 J. Xu is with the Department of ECE, University of Illinois at Urbana-Champaign. Email: \texttt{jxu18@illinois.edu}.}
}

\maketitle

\begin{abstract}%
We consider two closely related problems: planted clustering and submatrix localization.  The planted clustering problem assumes that a random graph is generated based on some underlying clusters of the nodes; the task is to recover these clusters given the graph. The submatrix localization problem concerns locating hidden submatrices with elevated means inside a large real-valued random matrix. Of particular interest is the setting where the number of clusters/submatrices is allowed to grow unbounded with the problem size. These formulations cover several classical models such as planted clique, planted densest subgraph, planted partition, planted coloring, and stochastic block model,
which are widely used for studying community detection and clustering/bi-clustering.

For both problems, we show that the space of the model parameters (cluster/submatrix size, cluster density, and submatrix mean) can be partitioned into four disjoint regions corresponding to decreasing statistical and computational complexities: (1) the \emph{impossible} regime, where all algorithms fail; (2) the \emph{hard} regime, where the computationally expensive Maximum Likelihood Estimator (MLE) succeeds; (3) the \emph{easy} regime, where the polynomial-time convexified MLE succeeds; (4) the \emph{simple} regime, where a simple counting/thresholding procedure succeeds. Moreover, we show that each of these algorithms provably fails in the previous harder regimes.

Our theorems establish the minimax recovery limit, which are tight up to constants and hold with a growing number of clusters/submatrices, and provide a stronger performance guarantee than previously known for polynomial-time algorithms.
Our study demonstrates the tradeoffs between statistical and computational considerations, and suggests that the minimax recovery limit may not be achievable by polynomial-time algorithms.
\end{abstract}


\section{Introduction}
In this paper we consider two closely related problems: {planted clustering} and {submatrix localization}, both concerning the recovery of hidden structures from a noisy random graph or matrix.
\begin{itemize}
\item \textbf{Planted Clustering:}  Suppose that out of a total of $ n $ nodes, $ rK $ of them are partitioned into $ r $ {clusters} of size $ K $, and the remaining $ n-rK $ nodes do not belong to any clusters; each pair of nodes is connected by an edge with probability $ p $ if they are in the same cluster, and with probability $ q $ otherwise. Given the adjacency matrix $ A$ of the graph, the goal is to recover the underlying clusters (up to a permutation of cluster indices). By varying the values of the model parameters, this formulation covers several classical models including planted clique,  planted coloring, planted densest subgraph, planted partition, and stochastic block model (cf.\ Definition~\ref{def:planted} and discussion thereafter).

\item \textbf{Submatrix Localization:} Suppose $A \in \mathbb{R}^{n_L\times n_R}$ is a random matrix with independent Gaussian entries with unit variance, where there are $ r $ submatrices of size $ K_L\times K_R $  with disjoint row and column supports, such that the entries inside these submatrices have mean $ \mu >0$, and the entries outside have mean zero. The goal is to identify the locations of these hidden submatrices given  $A$. This formulation generalizes the submatrix detection and bi-clustering models with a single bi-submatrix/cluster that are studied in previous work (cf.\ Definition~\ref{def:submatrix} and discussion thereafter).
\end{itemize}
We are particularly interested in the setting where the number $ r $ of clusters or submatrices may grow unbounded with the problem dimensions $ n $, $ n_L $, and $ n_R $ at an arbitrary rate. We may call this the \emph{high-rank} setting because $ r $ equals the rank of a matrix representation of the clusters and submatrices (cf.\ Definitions~\ref{def:planted} and~\ref{def:submatrix}). The other parameters $ K $, $ p$, $q$, and $\mu $ are also allowed to scale with $ n $ or $ ( n_L , n_R) $.

These two problems have been studied under various names such as \emph{community detection}, \emph{graph clustering/bi-clustering},
 and \emph{reconstruction in stochastic block models}, and have a broad range of applications.
They are used as generative models for approximating real-world networks and data arrays with natural cluster/community structures, such as social networks~\cite{Fortunato10}, gene expressions~\cite{shabalin2009submatrix}, and online ratings~\cite{Hajek13}.
They serve as benchmarks in the evaluation of algorithms for clustering~\cite{mathieu}, bi-clustering~\cite{balakrishnan2011tradeoff}, community detection~\cite{Newman04}, and other  network inference problems. They also provide a venue for studying the average-case behaviors of many graph theoretic problems including max-clique, max-cut, graph partitioning, and coloring~\cite{bollobas2004maxcut,Condon01}. The importance of these two problems are well-recognized in many areas across computer science, statistics, and physics~\cite{Yu11,arias2013community,nadakuditi,Decelle11,Mossel13,Lelarge13,anandkumar2013tensormixed,Bicke09,amini2013pseudo}.

The planted clustering and submatrix localization problems exhibit an interplay between \emph{statistical} and \emph{computational} considerations.
From a statistical point of view, we are interested in identifying the range of the model parameters for which the hidden structures---in this case the clusters and submatrices---can be recovered from the noisy data $A$. The values of the parameters $ n,r,K,p,q, \mu$ govern the statistical hardness of the problems: the problems become more difficult with smaller values of $p-q $, $ \mu $, $ K$, and larger $r$, because the observations are noisier and the sought-after structures are more complicated.
A statistically powerful algorithm is one that can recover the hidden structures for a large region of the model parameter space.

From a computational point of view, we are concerned with the running time of different recovery algorithms. An exhaustive search over the solution space (i.e., all possible clusterings or locations of the submatrices) may make for a statistically powerful algorithm, but is computationally intractable. A simpler algorithm with lower running time is computationally more desirable, but may succeed only in a smaller region of the model parameter space and thus has weaker statistical power.

Therefore, it is important to take a joint statistical-computational view to the planted clustering and submatrix localization problems,  and to understand the \emph{tradeoffs} between these two considerations. How do algorithms with different computational complexity achieve different statistical performance?  For these two problems, what is the \emph{information limit} (under what conditions on the model parameters does recovery become infeasible for any algorithm), and what is the \emph{computational limit} (when does it become infeasible for computationally tractable algorithms)?

The results on this paper sheds light on the above questions. For both problems, our results demonstrate, in a precise and quantitative way, the following phenomenon:
The parameter space can be partitioned into four disjoint regions, such that each region corresponds to statistically easier instances of the problem than the previous one, and recovery can be achieved by simpler algorithms with lower running time.
Significantly, there might exist a large gap between the statistical performance of computationally intractable algorithms and that of computationally efficient algorithms. We elaborate in the next two subsections.

\subsection{Planted Clustering: The Four Regimes}\label{sec:intro_planted}
For concreteness, we first consider the planted clustering problem in the setting $ r\ge 2 $, $ p>q $ and $p/q=\Theta(1)$.
This covers the standard planted bisection/partition/$ r$-disjoint-clique models.

The statistical hardness of cluster recovery is captured by the quantity $ \frac{(p-q)^2}{q(1-q)} $, which is essentially a measure of the Signal-to-Noise Ratio (SNR). Our main theorems identify the following four regimes of the problem defined by the value of this quantity. Here for simplicity, the results use the notation $ \gtrsom  $ and $\lesssom $, which
ignore constant and $\log n$ factors; our main theorems do capture the $\log n$ factors.

\begin{compactitem}
\item \textbf{The Impossible Regime: $ \frac{(p-q)^2}{q(1-q)} \lesssom \frac{1}{K} $}.
In this regime, there is no algorithm, regardless of its computational complexity, that can recover the clusters with a vanishing probability of error.
\item \textbf{The Hard Regime: $ \frac{1}{K}  \lesssom \frac{(p-q)^2}{q(1-q)} \lesssom \frac{ n }{K^2}$}. There exists a computationally expensive algorithm---specifically the Maximum Likelihood Estimator (MLE)---that recovers the clusters with high probability in this regime (as well as in the next two easier regimes; we omit such implications in the sequel). There is no known polynomial-time algorithm that succeeds in this regime.
\item \textbf{The Easy Regime: $  \frac{n}{K^2} \lesssom \frac{(p-q)^2}{q(1-q)} \lesssom \frac{\sqrt{n }}{K} $}. There exists a polynomial-time algorithm---specifically a convex relaxation of the MLE---that recovers the clusters with high probability in this regime. Moreover, this algorithm provably fails in the hard regime above.
\item \textbf{The Simple Regime: $ \frac{(p-q)^2}{q(1-q)} \gtrsom \frac{\sqrt{n}}{K} $}. A simple algorithm based on counting node degrees and common neighbors recovers the clusters with high probability in this regime, and provably fails outside this regime (i.e., in the hard and easy regimes).
\end{compactitem}

We illustrate these four regimes in Figure~\ref{fig:summary} assuming the scaling $ p=2q=\Theta(n^{-\alpha}) $ and $ K=\Theta(n^\beta) $ for
two constants $\alpha, \beta \in (0,1)$.
Here cluster recovery becomes harder with larger $\alpha$ and smaller $\beta$.
In this setting, the four regimes correspond to four disjoint and non-empty regions of the parameter space.
Therefore, a computationally more expensive algorithm leads to an \emph{order-wise} (polynomial in $n$) enhancement in the statistical power.
For example, when $ \alpha=1/4 $, the simple, polynomial-time, and computationally intractable algorithms succeeds for $\beta$ larger than $ 0.75 $, $ 0.625$, and $ 0.25 $, respectively. There is a similar hierarchy for the allowable sparsity of the graph, given by $ \alpha < 0.25 $, $\alpha < 0.5$, and $\alpha < 0.75$ assuming $\beta=0.75$.
\begin{figure}
\begin{center}
 \scalebox{1}{
\begin{tikzpicture}[scale = 2, font = \small, thick]
\draw[->] (0, 0) node [below left] {$0$}-- (2.2, 0) node[right]{$\alpha$};
\draw[->] (0, 0) -- (0, 2.3) node [right,align=left]{$\beta$};

\draw (2, 0) -- (2, 2);
\draw (0, 2) -- (2, 2);
\node [left] at (0, 2) {$1$};
\node [below] at (2, 0) {$1$};
\node [below] at (1, 0) {$1/2$};
\node [left] at (0, 1) {$1/2$};

\path[fill = black!20] (0, 0) -- (2, 2) -- (2, 0) -- cycle;
\path[fill= red!70] (0,0)--(0,1) --(2,2)-- cycle;
\path[fill=blue!70] (0,1)--(1,2) -- (2,2)-- cycle;
\path[fill=green] (0,1)--(0,2)--(1,2)--cycle ;


\node at (1.5,0.6) {impossible};
\node at (0.5,0.9) {hard};
\node at (1.2,1.8) {easy};
\node at (0.35,1.8) {simple};
\end{tikzpicture}}
\end{center}
\caption{\small Illustration of the four regimes. The figure applies to the planted clustering problem with $p=2q=\Theta(n^{-\alpha})$ and $K=\Theta(n^{\beta})$, as well as to
the submatrix localization problem with $ n_L = n_R = n $, $ \mu^2= \Theta(n^{-\alpha}) $ and $K_L = K_R =\Theta(n^{\beta})$.
}
\label{fig:summary}
\end{figure}
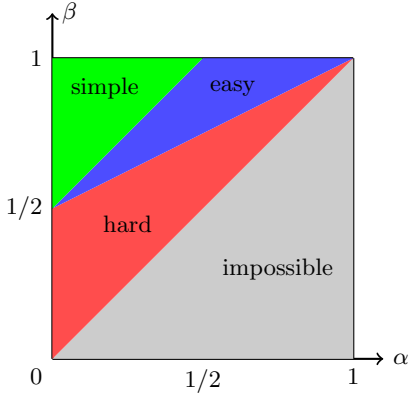

The results in the impossible and hard regimes together establish the \emph{minimax recovery boundary} of the planted clustering problem, and show that the MLE is statistically order-optimal. These two regimes are separated by an ``information barrier'': in the impossible regime the graph does not carry enough information to distinguish different cluster structures, so recovery is statistically impossible.

Our performance guarantees for the convexified MLE
improve the best known results for polynomial time algorithms in terms of the scaling, particularly in the setting when the number of clusters are allow to grow with $ n $.
We conjecture that no polynomial-time algorithm can perform significantly better and succeed in the hard regime, i.e., the convexified MLE achieves the \emph{computational limit} order-wise. While we do not prove the conjecture, there are many  supporting evidences; cf.\ Section~\ref{sec:easy}.
For instance, there is a ``spectral barrier'', determined by the spectrum of an appropriately defined noise matrix, that prevents the convexified MLE and spectral clustering algorithms from succeeding in the hard regime. In the special setting with a single cluster, the work by \cite{ma2013submatrix,HajekWuXu14} proves that no polynomial-time algorithm can reliably recover
the cluster if $\beta< \alpha/4+1/2$ conditioned on the planted clique hardness hypothesis.

The simple counting algorithm fails outside the simple regime due to a ``variance barrier'' which is associated with the fluctuations of the node degrees and the numbers of common neighbors. The simple algorithm is statistically order-wise weaker than the convexified MLE in separating different clusters.

\paragraph*{General results}

Our main theorems apply beyond the special setting above and allow for general values of $ p $, $ q $, $ K $, and $r$. The four regimes and the statistical-computational tradeoffs can be observed for a broad spectrum of planted problems, including planed partition, planted coloring, planted $r $-disjoint-clique and planted densest-subgraph models. Table~\ref{tab:four} summarizes the implications of our results for some of these models. More precise and general results are given in Section~\ref{sec:main}.

\begin{table*}
\begin{adjustbox}{center}
\small
\begin{tabular}{|r|ccc|}
\hline
 & %
\begin{tabular}{c}
\bf Planted $r$-Disjoint-Clique\tabularnewline
\footnotesize $1=p>q \ge 0,r\ge1$ \\
\end{tabular} & %
\begin{tabular}{c}
\bf Planted Partition\tabularnewline
\footnotesize $1\ge p>q\ge0$, $rK=n$ \\
\end{tabular} & %
\begin{tabular}{c}
\bf Planted Coloring\\
\footnotesize $0=p<q \le 1$, $rK=n$ \\
\end{tabular}\\
\hline
\begin{tabular}{r}
\bf Impossible\\
\footnotesize Thm \ref{thm:Impossible}, Cor \ref{cor:impossible}\\
\end{tabular} & $K \lesssim \left( \frac{q}{1-q} \vee \frac{1}{\log (1/q)} \right) \log n$ & $(p-q)^{2}\lesssim\frac{p(1-q)\log n}{K}$ & $q\lesssim\frac{\log n}{K}$\\
\hline
\begin{tabular}{r}
\bf MLE\\
\footnotesize Thm~\ref{thm:MLE}, Cor \ref{cor:hard}\\
\end{tabular} & $K\gtrsim \left( \frac{q}{1-q} \vee \frac{1}{\log(1/q)} \right) \log n$ & $(p-q)^{2}\gtrsim\frac{p(1-q)\log n}{K}$ & $q\gtrsim\frac{\log n}{K}$\\
\hline
\begin{tabular}{r}
\bf Convexified MLE\\
\footnotesize Thm~\ref{thm:CVX}\\
\end{tabular} & $K\gtrsim \frac{\log n}{1-q} + \sqrt{\frac{qn}{1-q}} $ & $(p-q)^{2}\gtrsim\frac{p(1-q)\log n}{K}+\frac{q(1-q)n}{K^{2}}$ & $q\gtrsim\frac{\log n}{K}+\frac{(1-q)n}{K^{2}}$\\
\hline
\begin{tabular}{r}
\bf Simple Counting\\
\footnotesize Thm \ref{thm:Simple}, Rem~\ref{rem:count_non_neghber}\\
\end{tabular} & $K\gtrsim \frac{\log n}{1-q} + \sqrt{\frac{qn \log n}{1-q}}$ & $(p\!-\!q)^{4}\gtrsim\left[\frac{p^{2}(1-q)}{K}\!+\!\frac{nq(1-q) (q \vee p^2)}{K^{2}}\right]\log n$ & $q^{2}\gtrsim\frac{(1-q)n\log n}{K^{2}}$\\
\hline
\end{tabular}
\end{adjustbox}
\caption{\label{tab:four}\small Our results specialized to different planted models. Here the notation $ \gtrsim  $ and $\lesssim $ ignore constant factors. This table
shows the \emph {necessary conditions} for any algorithm to succeed under a mild assumption $K \gtrsim \log (rK)$,
as well as the \emph{sufficient} conditions under which the algorithms in this paper succeed, thus  corresponding to the four regimes described in Section~\ref{sec:intro_planted}. The relevant theorems/corollaries are also listed. The conditions for convexified MLE and simple counting can further be shown to be also \emph{necessary} in a broad range of settings; cf.\ Theorems~\ref{thm:cvx_converse} and~\ref{thm:SimpleConverse}. The results in this table are not the strongest possible; see the referenced theorems for more precise statements. }
\end{table*}

\subsection{Submatrix Localization: The Four Regimes}

Similar results hold for the submatrix localization problem. Consider the setting with $ n_L=n_R=n $ and $ K_L=K_R=K $.
The statistical hardness of  submatrix localization is captured by the quantity $\mu^2$, which is again a measure of the SNR.
In the high SNR setting with $\mu^2=\Omega(\log n)$, the submatrices can be trivially identified by element-wise thresholding. In the more interesting  low SNR setting with $\mu^2 =O(\log n)$, our main theorems identify the following four regimes, which have the same meanings as before:
\begin{compactitem}
\item \textbf{The Impossible Regime: $ \mu^2  \lesssom \frac{1}{K} $}. All algorithm fail in this regime.
\item \textbf{The Hard Regime:  $ \frac{1}{K} \lesssom \mu^2 \lesssom  \frac{n}{K^2} $}. The computationally expensive MLE  succeeds, and it is conjectured that no polynomial-time algorithm succeeds here.
\item \textbf{The Easy Regime: $ \frac{n}{K^2}\lesssom \mu^2 \lesssom \frac{\sqrt{n}}{K} $}. The polynomial-time convexified MLE succeeds, and provably fails in the hard regime.
\item \textbf{The Simple Regime: $ \frac{\sqrt{n} }{K}  \lesssom \mu^2 \lesssom 1 $}. A simple thresholding algorithm succeeds, and provably fails  outside this regime.
\end{compactitem}
We illustrate these four regimes in Figure~\ref{fig:summary}  assuming  $\mu^2 =\Theta(n^{-\alpha})$ and $ K=\Theta(n^\beta) $.
In fact, the results above hold in the more general setting where the entries of  $ A $ are \emph{sub-Gaussian}.

\subsection{Discussions}\label{sec:contribution}

This paper presents a systematic study of planted clustering and submatrix localization with a growing number of clusters/submatrices. We provide sharp characterizations of the {minimax recovery boundary} with the lower and upper bounds matching up to constants. We also give improved performance guarantees for convex optimization approaches and the simple counting/thresholding algorithms. In addition, complementary results are given on the \emph{failure conditions} for these algorithms, hence characterizing their performance limits. Our analysis addresses several challenges that arise in the high-rank setting. The results in this paper highlight the similarity between planted clustering and submatrix localization, and place under a unified framework several classical problems such as planted clique, partition, coloring, and densest graph.

The central theme of our investigation is the interaction between the statistical and the computational aspects in the problems, i.e.,  {how to handle more noise and more complicated structures using more computation}. Our study parallels a recent line of work that takes a joint statistical and computational view on inference problems~\cite{balakrishnan2011tradeoff,oymak2012simultaneously,berthet2013lowerSparsePCA,chandrasekaran2013tradeoff,ma2013submatrix}; several of these works are closely related to special cases of the planted clustering and bi-clustering models. In this sense, we investigate two specific but fundamental problems, and we expect that the phenomena and principles described in this paper are relevant more generally.
Below we provide additional discussions, and comment on the relations with existing work.

\paragraph*{High rank vs.\ rank one.}
Several recent works investigate the problems of single-submatrix detection/localization~\cite{kolar2011submatrix,arias2011anomalous},
planted densest subgraph detection~\cite{arias2013community} and sparse principal component analysis (PCA)~\cite{amini2009sparsePCA} (cf.\ Section~\ref{sec:related} for a literature review).
Even earlier is the extensive study of the statistical/computational hardness of Planted Clique. The majority of these works focus on the \emph{rank-one} setting with a single clique, cluster, submatrix or principal component. This paper considers the more general \emph{high-rank} setting where the number $ r $ of clusters/submatrices may grow quickly with the problem size.
This setting is important in many empirical networks~\cite{Leskovec08,Yu11}, and poses significant challenges to the analysis.
Moreover, there are qualitative differences between these two settings. We discuss one such difference in the next paragraph.

\paragraph*{The power of convex relaxations.} In the previous work on the rank-one case of the submatrix detection/localization problem~\cite{ma2013submatrix,balakrishnan2011tradeoff} and the sparse PCA problem~\cite{Vilenchik13}, it is shown that simple algorithms based on averaging/thresholding have order-wise similar statistical performance as more sophisticated convex optimization approaches. In contrast, for the problems of finding multiple clusters/submatrices, we show that convex relaxation approaches are statistically much more powerful than the simple counting/thresholding algorithm. Our analysis reveals that the power of convex relaxations lies in \emph{separating different clusters/submatrices}, but not in identifying a single cluster/submatrix. Our results thus provide one explanation for the (somewhat curious) observation in previous work regarding the lack of benefit of using sophisticated methods, and demonstrate a finer spectrum of computational-statistical tradeoffs.

\paragraph*{Detection vs.~estimation.} Several recent works on planted densest subgraph and submatrix detection have focused on the \emph{detection} or \emph{hypothesis testing} version of the problems, i.e., detecting the existence of  a dense cluster or an elevated submatrix (cf. Section~\ref{sec:related} for literature review). In this paper, we study the (support) \emph{estimation} version of the problems, where the goal is to find the precise locations of the clusters/submatrices.  In general estimation appears to be harder than detection. For example, if we consider the scalings of $ \mu $ and $ K $ in Figure~\ref{fig:summary} of this paper, and compare with Figure 1 in~\cite{ma2013submatrix} which studies submatrix detection, we see that the minimax localization boundary is $\beta= \alpha$, whereas the minimax detection boundary is at a higher value $\beta = \min \{\alpha, \alpha/4+1/2\}$. For the planted densest subgraph problem, we see a similar gap between the minimax detection and estimation boundaries if we compare our results with results in~\cite{arias2013community,HajekWuXu14}.
In addition, it is shown in~\cite{ma2013submatrix,HajekWuXu14} that if $\beta>\alpha/4+1/2$, the planted submatrix or densest subgraph can be detected in linear time; if $\beta< \alpha/4+1/2$, no polynomial-time test exists assuming the hardness of the planted clique detection problem.
For estimation, we  prove the sufficient condition $ \beta>\alpha/2+1/2$, which is the best known performance guarantee for polynomial-time algorithms---again we see a gap between detection and estimation.
For detecting a sparse principal component, see the seminar work~\cite{berthet2013lowerSparsePCA} for proving computational lower bounds  conditioned on the hardness of Planted Clique.

\paragraph*{Extensions.} It is a simple exercise to extend our results to a  variant of the planted clustering model where the graph adjacency matrix has sub-Gaussian entries instead of Bernoulli, corresponding to a weighted graph clustering problem. Similarly, we can also extend the submatrix location problem to the setting with Bernoulli entries, which is the  bi-clustering problem on an unweighted graph and covers the \emph{planted bi-clique} problem~\cite{Feldman2012statAlg,ames2011plantedclique} as a special case.

\subsection{Related Work}\label{sec:related}

There is a large body of literature, from the physics, computer science and statistics communities, on models and algorithms for graph clustering and bi-clustering, as well as on their various extensions and applications.
A complete survey is beyond the scope of this paper. Here we focus on theoretical work on  planted clustering/submatrix localization concerning \emph{exact recovery} of the clusters/submatrices.
Detailed comparisons of existing results with ours are provided after we present each of our theorems in Sections~\ref{sec:main} and~\ref{sec:submatrix}. We emphasize that our results are \emph{non-asymptotic} and applicable to finite values of $ n,n_L $ and $ n_R $, whereas some of the results below require $ n \to \infty$.

\paragraph*{Planted Clique, Planted Densest Subgraph} The planted clique model ($ r=1$, $p=1$, $q=1/2 $) is the most widely studied planted model. If the clique has size $ K =o(\log n) $, recovery is impossible as the random graph $ \calG(n,1/2) $ will have a clique with at least the same size; if $ K=\Omega(\log n) $, an exhaustive search succeeds~\cite{Alon98}; if $ K=\Omega(\sqrt{n}) $, various polynomial-time algorithms work~\cite{Alon98,Dekel10,Deshpande12}; if $K=\Omega(\sqrt{n \log n})$, the nodes in the clique can be easily identified by counting degrees~\cite{Kucera95}. It is an open problem to find polynomial-time algorithms which succeed in the  regime with $K=o(\sqrt{n})$, and it is believed that this cannot be done~\cite{Hazan2011Nash,Juel00cliqueCrypto,alon2007testing,Feldman2012statAlg}. The four regimes above can be considered as a special case of our results for the general planted clustering model.
The planted densest subgraph model generalizes the planted clique model by allowing general values of $ p $ and $ q $. The detection version of this problem is studied in~\cite{arias2013community, verzelen2013sparse}, and  conditional computational hardness results are obtained in \cite{HajekWuXu14}.

\paragraph*{Planted $ r $-Disjoint-Cliques, Partition, and Coloring} Subsequent work considers the setting with $ r\ge 1 $ planted cliques~\cite{McSherry01}, as well as the planted partition model (a.k.a.\ stochastic block model) with general values of $ p > q $~\cite{Condon01,Holland83}.  A subset of these results allow for growing values $ r $. Most existing work focuses on the recovery performance of specific polynomial-time algorithms. The state-of-the-art recovery results for planted $r $-disjoint-clique are given in~\cite{McSherry01,Chen12,ames2010kclique}, and for planted partition in~\cite{Chen12,anandkumar2013tensormixed,cai2014robust}; see~\cite{chen2014improved} for a survey of these results.
The setting with $ p<q $ is sometimes called the \emph{heterophily} case, with the planted coloring model ($ p=0 $) as an important special case~\cite{alon1997coloring3,coja2004coloringSemirandom}.
Our performance guarantees for the convexified MLE (cf.\ Table~\ref{tab:four}) improve upon the previously known results for polynomial-time algorithms.
Also, particularly when the number of clusters $ r $ is allowed to scale arbitrarily with $ n $, matching upper and lower bounds for the information-theoretic limits were previously unknown. This paper identifies the minimax recovery thresholds for general values of $p,q, K$ and $r$, and shows that they are achieved by the MLE. Our results also suggest that polynomial-time algorithms may not be able to achieve these thresholds in the growing $ r $ setting with the cluster size $ K $ sublinear in $ n $.

\paragraph*{Converse Results for Planted Problems} Complementary to the \emph{achievability} results, another line of work focuses on \emph{converse} results, i.e., identifying necessary conditions for recovery, either for any algorithm, or for any algorithm in a specific class.
For the planted partition model with $K=\Theta(n)$, necessary conditions for any algorithm to succeed are obtained in~\cite{ChaudhuriGT12, Chen12,balakrishnan2011threshold,Abbe14} using information-theoretic tools. For spectral clustering algorithms and convex optimization approaches, more stringent conditions are shown to be needed~\cite{nadakuditi,vinayak2013sharp}. We generalize and improve upon the existing work above.

\paragraph*{Sharp Exact Recovery Thresholds with a Constant Number of Clusters}
Since the conference version of this paper is published~\cite{ChenXu14}, a number of papers have appeared on the information-theoretic
limits of exact recovery under the stochastic block model.
Under the special setting with $r=2$ and $K=n/2$,  the recovery threshold
 \emph{with sharp constants} is identified in \cite{Abbe14} for $p, q= O(\log n /n)$, and in \cite{Mossel14} for
general scalings of $p,q$. Very recently, \cite{Abbe15} proved the sharp recovery threshold for
the more general case where $r=O(1)$, $K=\Theta(n)$ and the in-cluster and cross-cluster edge probabilities
are heterogeneous and scale as $\log n /n$. Notably, when the number of clusters $ r $ is bounded, sharp recovery thresholds may be achieved by polynomial-time algorithms, in particular, by
the semi-definite programming relaxation of the maximum likelihood estimator~\cite{HajekWuXu14SDP,HajekWuXu14SDP2}. Our results are optimal up to absolute constant factors, but are non-asymptotic and apply to a growing number of clusters/submatrices of size sublinear in $n$.

\paragraph*{Approximate Recovery} While not the focus of this paper, approximate cluster recovery (under various criteria)  has also been studied, e.g.,  for planted partition with $ r=O(1) $ clusters in~\cite{Mossel12,Mossel13,Massoulie13,yun2014adaptive,Decelle11}. These results are not directly comparable to ours, but  often the approximate recovery conditions differ from  the exact recovery conditions by a $ \log n $ factor. When  constant factors are concerned, the existence of a hard regime is also conjectured in~\cite{Decelle11,Mossel12}.

\paragraph*{Submatix Localization}
The statistical and computational tradeoffs in locating a single submatrix (i.e., $ r=1 $) are studied in~\cite{balakrishnan2011tradeoff,kolar2011submatrix}, where the information limit is shown to be achieved by a computationally intractable algorithm order-wise. The success and failure conditions for various polynomial-time procedures are also derived. The work~\cite{ames2012clustering} focuses on success conditions for a convex relaxation approach; we improve the results particularly in the high-rank setting.
The single-submatrix \emph{detection} problem is studied in~\cite{butucea2011submatrix,shabalin2009submatrix,sun2013anova,arias2011anomalous,bhamidi2012energy}, and
the recent work by~\cite{ma2013submatrix} establishes the conditional hardness for this problem.

\subsection{Paper Organization and Notation}
The remainder of this paper is organized as follows.
In Section~\ref{sec:main} we set up the planted clustering model and present our main theorems for the impossible, hard, easy, and simple regimes. In Section~\ref{sec:submatrix} we turn to the submatrix localization problem and provide the corresponding theorems for the four regimes. Section~\ref{sec:conclusion} provides a brief summary with a discussion of future work. We prove the main theorems for planted clustering and submatrix localization in Sections~\ref{sec:proof} and~\ref{sec:proof_bi}, respectively.

\paragraph*{Notation}
Let $a \vee b=\max\{a,b\}$ and  $a \wedge b = \min \{ a, b \}$, and $ [m]=\{1,2,\ldots, m\} $ for any positive integer $ m$.
We use $c_1, c_2$ etc.\ to denote absolute numerical constants whose values can be made explicit and are independent of the model parameters. We use the standard big-O notations: for two sequences $\{a_n\},\{b_n\}$, we write $a_n \lesssim b_n$ or $a_n=O(b_n)$ to mean $a_n \le c_1 b_n$ for an absolute constant $c_1$ and all $ n $. Similarly,  $a_n \gtrsim b_n$ means $a_n=\Omega(b_n)$, and $a_n  \asymp b_n$ means $a_n=\Theta(b_n)$.

\section{Main Results for Planted Clustering}\label{sec:main}
The {planted clustering} problem is defined by five parameters $ n,r,K \in \mathbb{N}$ and $p, q  \in [0,1] $ such that $n \ge rK$.
\begin{definition}[Planted Clustering]\label{def:planted}
Suppose $n$ nodes (which are identified with $ [n] $) are divided into two subsets $V_1$ and $V_2$ with $|V_1|=rK$ and $|V_2|=n-rK$. The nodes in $V_1$ are partitioned into $r$ disjoint clusters $C^\ast_1, \ldots, C^\ast_r$ (called \emph{true clusters}), where $|C^\ast_m|=K$  for each $  m \in [r] $ and $\bigcup_{m=1}^r C^\ast_m=V_1$. Nodes in $V_2$ do not belong to any of the clusters and are called \emph{isolated nodes}. A random graph is generated based on the cluster structure: for each pair of nodes and independently of all others, we connect them by an edge with probability $p$ (called \emph{in-cluster edge density}) if they are in the same cluster, and otherwise with probability $q$ (called \emph{cross-cluster edge density}).
\end{definition}
We emphasize again that the values of $ p $, $ q $, $ r $, and $ K $ are allowed to be functions of $n$.  The goal is to exactly recover the true clusters $\{C^\ast_m\}_{m=1}^r$ up to a permutation of cluster indices given the random graph.

The model parameters $ (p, q, r, K) $ are assumed to be known to the algorithms. This assumption is often not necessary and can be relaxed~\cite{Chen12,arias2013community}. It is also possible to allow for non-uniform cluster sizes~\cite{ailon2013breaking}, and heterogeneous edge probabilities~\cite{cai2014robust} and node degrees~\cite{ChaudhuriGT12,Chen12}. These extensions are certainly important in practical applications; we do not delve into them, and point to the referenced papers above and the references therein for work in this direction.

To facilitate subsequent discussion, we introduce a matrix representation of the planted clustering problem. We represent the true clusters $\{C^\ast_m\}_{m=1}^r$ by a {\it cluster matrix} $Y^\ast \in \{0,1\}^{n\times n}$, where $Y^\ast_{ii}=1$ for $i \in V_1$, $Y^\ast_{ii}=0$ for $i \in V_2$, and $Y^\ast_{ij}=1$ if and only if nodes $i$ and $j$ are in the same true cluster. Note that the rank of $Y^\ast$ equals $r$, hence the name of the high-rank setting. The adjacency matrix of the graph is denoted as $ A $, with the convention  $ A_{ii}=0,\forall i\in[n]$. Under the planted clustering model, we have $ \mathbb{P}(A_{ij}=1)=p $ if $ Y_{ij}^*=1 $ and $ \mathbb{P}(A_{ij}=1)=q $ if $ Y^*_{ij}=0 $ for all $i \neq j$. The problem reduces to recovering $Y^\ast$ given $A$.

The planted clustering model generalizes several classical planted models.
\begin{compactitem}
\item {\bf Planted $r$-Disjoint-Clique}~\cite{McSherry01}. Here $p=1$ and $0<q<1$, so  $ r $ cliques of size $ K $ are planted into an Erd\H{o}s-R\'enyi random graph $ G(n,q) $.  The special case with $r=1$ is known as the {\it planted clique} problem~\cite{Alon98}.
\item {\bf Planted Densest Subgraph}~\cite{arias2013community}. Here $0<q<p<1$ and $r=1$, so there is a subgraph of size $K$ and density $ p $ planted into a $G(n,q)$ graph.
\item {\bf Planted Partition}~\cite{Condon01}. Also known as the {\it stochastic blockmodel}~\cite{Holland83}. Here $n=rK$ and $p,q \in (0,1)$.  The special case with $r=2$ can be called {\it planted bisection}~\cite{Condon01}. The case with $p<q$ is sometimes called {\it planted noisy coloring} or {\it planted $ r $-cut}~\cite{Decelle11,bollobas2004maxcut}.
\item {\bf Planted $r$-Coloring}~\cite{alon1997coloring3}. Here $n=rK$ and $0=p<q<1$, so each cluster corresponds to a group of disconnected nodes that are assigned with the same color.
\end{compactitem}

\paragraph*{Reduction to the $ p>q $ case.} For clarity we shall focus on the homophily setting with $ p>q $; results for the $ p<q $ case are similar. In fact, any achievability or converse result for the $ p>q $ case immediately implies a corresponding result for $ p<q $. To see this, observe that if the graph $ A $ is generated from the planted clustering model with $ p<q $, then the flipped graph $ A':= J -A-I $ ($J$ is the all-one matrix and $I$ is the identity matrix) can be considered as generated with in/cross-cluster edge densities $ p'=1-p $ and $ q'=1-q $, where $ p'>q' $. Therefore, a problem with $ p<q $ can be reduced to one with $ p'>q' $. Clearly the reduction can also be done in the other direction.

\subsection{The Impossible Regime: Minimax Lower Bounds}\label{sec:impossible}
In this section, we characterize the necessary conditions for cluster recovery.
Let $\mathcal{Y}$ be the set of cluster matrices corresponding to $ r $ clusters of size $ K $; i.e.,
\begin{align*}
\mathcal{Y} &= \left\{Y\in \{0,1\}^{n\times n} \right.
|\textrm{there exist disjoint clusters $\{C_m\}_{m=1}^r$ such that $|C_m|=K,\forall m\in[r],$}\\
&\qquad\qquad\qquad\qquad\; \left.\textrm{and $ Y $ is the corresponding cluster matrix} \right\}.
\end{align*}
We use $ \hat{Y} \equiv \hat{Y}(A)$ to denote an estimator which takes as input the graph $ A $ and outputs an element of $ \mathcal{Y} $ as an estimate of the true $ Y^* $. Our results are stated in terms of the Kullback-Leibler (KL) divergence between two Bernoulli distributions with means $u$ and $v$, denoted by $D(u \Vert v ):=u\log\frac{u}{v}+(1-u)\log\frac{1-u}{1-v}$.  The following theorem gives a lower bound on the minimax error probability of recovering~$ Y^* $.
\begin{theorem}[Impossible] \label{thm:Impossible}
Suppose $128 \le K \le n/2$. Under the planted clustering model with $p>q$, if one of the following two conditions holds:
\begin{align}
K \cdot D( q \| p ) & \le \frac{1}{192} \left[\log (rK) \wedge K\right],
\label{eq:KLimpossible1}\\
K \cdot D( p \| q ) & \le \frac{1}{192} \log n,
\label{eq:KLimpossible2}
\end{align}
then
\begin{equation*}
\inf_{\hat{Y}} \sup_{Y^*\in \mathcal{Y}} \mathbb{P} \left[\hat{Y} \neq Y^*\right] \ge \frac{1}{4},
\end{equation*}
where the infimum ranges over all measurable function of the graph.
\end{theorem}
The theorem shows it is fundamentally impossible to recover the clusters with success probability close to 1 in the regime where~\eqref{eq:KLimpossible1} or~\eqref{eq:KLimpossible2} holds, which is thus called the {\it impossible regime}.
This regime arises from an \emph{information/statistical barrier}:
The KL divergence on the LHSs of~\eqref{eq:KLimpossible1} and~\eqref{eq:KLimpossible2} determines how much information of $ Y^* $ is contained in the data  $ A $. If the in-cluster and cross-cluster edge distributions are close (measured by the KL divergence) or the
cluster size is small, then  $ A $ does not carry enough information to distinguish different cluster matrices.

It is sometimes more convenient to use the following corollary, derived by upper-bounding the KL divergence in~\eqref{eq:KLimpossible1} and~\eqref{eq:KLimpossible2} using its Taylor expansion. This corollary was used when we overviewed our results in Section~\ref{sec:intro_planted}. See table~\ref{tab:four} for its implications for specific planted models.

\begin{corollary} \label{cor:impossible}
Suppose  $ 128 \le K \le n/2$. Under the planted clustering model with $p>q$, if any one of the following three conditions holds:
\begin{align}
K(p-q)^2 & \le \frac{1}{192} q (1-q) \log n , \label{eq:impossiblesimle_1}\\
K p & \le \frac{1}{193}  \left[\log (rK) \wedge K\right],  \label{eq:impossiblesimle_2}\\
K p \log  \frac{p}{q} & \le \frac{1}{192} \log n, \label{eq:impossiblesimle_3}
\end{align}
then $ \inf_{\hat{Y}} \sup_{Y^*\in \mathcal{Y}} \mathbb{P} [\hat{Y} \neq Y^*] \ge \frac{1}{4}.$
\end{corollary}

Note the asymmetry between the roles of $ p $ and $ q $ in the conditions~\eqref{eq:KLimpossible1} and~\eqref{eq:KLimpossible2}; this is  made apparent in Corollary~\ref{cor:impossible}. To see why the asymmetry is natural, recall that by a classical result of~\cite{grimmett1975colouring}, the largest clique in a random graph $ G(n,q) $ has size $k_q=\Theta(\log n/\log(1/q) )$ almost surely. Such a  clique cannot be distinguished from a true cluster if $ K\lesssim k_q  $, even when $ p=1 $. This is predicted by the condition~\eqref{eq:impossiblesimle_3}. When $ q=0 $, cluster recovery requires $ p\gtrsim \frac{\log(rK)}{K} $ to ensure all true clusters are connected within themselves, matching the condition~\eqref{eq:impossiblesimle_2}.
The  term  $ K $ on the RHS of~\eqref{eq:KLimpossible1} and~\eqref{eq:impossiblesimle_2} is relevant only when $ K\le \log(rK) $. Potential improvement on this term is left to future work.

\paragraph*{Comparison to previous work} When $ r=1 $ and $q=1/2$, our results recover the $ K=\Theta(\log n) $ threshold for the classical planted clique problem. For planted partition with $ r=O(1) $ clusters of size $ K=\Theta(n) $ and $p/q=\Theta(1)$, the work in~\cite{ChaudhuriGT12,Chen13} establishes the necessary condition $ p-q\lesssim \sqrt{p/n} $; our result is stronger by a logarithmic factor. The work in~\cite{Abbe14} also considers planted partition with $ r=2 $ and focus on the special case with the scaling $ p,q=\Theta(\log(n)/n) $; they establish the condition $ p+q-2\sqrt{pq} <2\log(n)/n$, which is consistent with our results up to constants in this regime. Compared to previous work, we handle the more general setting where $ p,q  $ and $ r $ may scale arbitrarily with $n$.

\subsection{The Hard Regime: Optimal Algorithm}\label{sec:hard}
In this subsection, we characterize the sufficient conditions for cluster recovery
which match the necessary conditions given in Theorem~\ref{thm:Impossible} up to constant factors.
We consider the Maximum Likelihood Estimator of $ Y^* $ under the planted clustering model, which we now derive. The log-likelihood of observing the graph~$ A $ given a cluster matrix $ Y\in \mathcal{Y} $ is
\begin{align}
\log \mathbb{P}_Y(A)
&= \log \prod_{i<j} p^{A_{ij}Y_{ij}} q^{A_{ij} (1-Y_{ij})} (1-p)^{(1-A_{ij}) Y_{ij} }(1-q) ^{(1-A_{ij}) (1-Y_{ij})}\nonumber\\
&=\log \frac{p(1\!-\!q)}{q(1\!-\!p)}\sum_{i<j}\! A_{ij} Y_{ij} + \log \frac{1\!-\!p}{1\!-\!q} \sum_{i<j}\! Y_{ij}  + \log \frac{q}{1\!-\!q} \sum_{i<j}\! A_{ij} +\sum_{i<j}\! \log (1\!-\!q).\label{eq:loglikelihood}
\end{align}
Given $ A $, the MLE maximizes the the log-likelihood over the set $ \mathcal{Y} $ of all possible cluster matrices.
Note that $\sum_{i<j}Y_{ij} = r \binom{K}{2}$ for all $ Y\in\mathcal{Y} $, so the last three terms in~\eqref{eq:loglikelihood} are independent of~$ Y $. Therefore, the MLE for the $ p>q $ case is given as in Algorithm~\ref{alg:mle}.
\begin{algorithm}[H]
\caption{Maximum Likelihood Estimator ($ p>q $)}\label{alg:mle}
\begin{align}
\hat{Y} = \arg\max_{Y}  & \sum_{i,j} A_{ij} Y_{ij} \label{MLE} \\
\text{s.t.	} & \; Y \in \mathcal{Y}. \label{eq:cluster}
\end{align}
\end{algorithm}
\noindent
Algorithm~\ref{alg:mle} is equivalent to finding $ r $ disjoint clusters of size$ K $ that maximize the number of edges inside the clusters (similar to Densest $ K $-Subgraph), or minimize the number of edges outside the clusters (similar to Balanced Cut) or the disagreements between $ A $ and $ Y $ (similar to Correlation Clustering in~\cite{bansal2004correlation}). Therefore, while Algorithm~\ref{alg:mle} is derived from the planted clustering model, it is in fact quite general and not tied to the modeling assumptions. Enumerating over the set $ \mathcal{Y} $ is computationally intractable in general since $ \vert \mathcal{Y} \vert= \Omega(e^{rK}) $.

The following theorem provides a success condition for the MLE.
\begin{theorem}[Hard] \label{thm:MLE}
Under the planted clustering model with $p>q$, there exists a universal constant $ c_1$ such that for any $ \gamma \ge 1 $, the optimal solution $ \hat{Y} $ to the problem~\eqref{MLE}--\eqref{eq:cluster} is unique and equal to  $ Y^* $ with probability at least $ 1-16 (\gamma rK)^{-1}- 256 n^{-1} $ if both of the following hold:
\begin{equation}
\begin{aligned}
K \cdot D(q \Vert p) &\ge c_1 \log (\gamma rK),\\
   K \cdot D( p \Vert q) &\ge c_1 \log n.
\end{aligned}
\label{eq:ConditionHard}
\end{equation}
\end{theorem}

We refer to the regime in which the condition~\eqref{eq:ConditionHard} holds but~\eqref{eq:Condition} below fails as the {\it hard regime}, as clustering is statistically possible but conjectured to be computationally hard (cf.\ Conjecture~\ref{ConjectureHardness}).
The conditions~\eqref{eq:ConditionHard} above and~\eqref{eq:KLimpossible1}--\eqref{eq:KLimpossible2} in Theorem~\ref{thm:Impossible} match up to a constant factor under the mild assumption~$K \ge \log (rK)$. This establishes the minimax recovery boundary for planted clustering and the minimax optimality of the MLE up to constant factors.

By lower bounding the KL divergence, we obtain the following corollary, which is sometimes more convenient to use. See Table~\ref{tab:four} for its implications for specific planted models.
\begin{corollary} \label{cor:hard}
For planted clustering with $p>q$, there exists a universal constant $ c_2$ such that for any $ \gamma\ge 1$, the optimal solution $ \hat{Y} $ to the problem~\eqref{MLE}--\eqref{eq:cluster} is unique and equal to $ Y^* $ with probability at least $ 1-16 (\gamma rK)^{-1}- 256 n^{-1} $ provided
\begin{align}
K(p-q)^2 \ge c_2 q (1-q)  \log n, \quad  K p \ge c_2 \log (\gamma rK)  \quad\text{and}\quad Kp \log \frac{p}{q} \ge c_2 \log n.  \label{eq:hardsimple}
\end{align}
\end{corollary}

The condition~\eqref{eq:hardsimple} can be simplified to  $K(p-q)^2 \gtrsim q (1-q) \log n$ if $q=\Theta(p)$, and to $Kp \log \frac{p}{q} \gtrsim \log n, K p \gtrsim \log (rK)$ if $q=o(p)$. These match the converse conditions in Corollary~\ref{cor:impossible} up to constants.

\paragraph*{Comparison to previous work} Theorem~\ref{thm:MLE} provides the first minimax results tight up to constant factors when the number of clusters is allowed to grow, potentially at a \emph{nearly-linear} rate  $ r=O(n/\log n) $. Interestingly, for a fixed cluster size, the recovery boundary~\eqref{eq:ConditionHard} depends only weakly on the number of clusters~$ r $ though the logarithmic term.
For $ r=1 $ and $ p=2q=1 $, we recover the recovery boundary for planted clique $ K\asymp \log n $. For the planted densest subgraph model where $p/q=\Theta(1)$, $p$ bounded away from $1$ and $Kq \gg 1$, the minimax \emph{detection} boundary is shown in \cite{arias2013community} to be $ \frac{(p-q)^2}{q} \asymp \min \{ \frac{1}{K}\log \frac{n}{K}, \frac{n^2}{K^4} \}$; our results show that the minimax \emph{recovery} boundary is $\frac{(p-q)^2}{q} \asymp \frac{\log n}{K}$,
which is strictly above the detection boundary because $\frac{n^2}{K^4}$ can be much smaller than $\frac{\log n}{K}.$
For the planted bisection model with two equal-sized clusters: if $ p,q=\Theta(\log(n)/n) $,  the sharp recovery boundary is found in \cite{Abbe14} and \cite{Mossel14} to be $ K( \sqrt{p}-\sqrt{q})^2> \log n $, which is consistent with our results up to constants; if $p,q=O(1/n)$,  the correlated recovery limit is shown in \cite{Mossel12,Massoulie13, Mossel13} to be $K(p-q)^2>p+q$, which is consistent with our results up to a logarithmic factor.

\subsection{The Easy Regime: Polynomial-Time Algorithms}\label{sec:easy}
In this subsection, we present a polynomial-time algorithm for the planted clustering problem and show that it succeeds in the easy regime described in the introduction.

Our algorithm is based on taking a convex relaxation of the MLE in Algorithm~\ref{alg:mle}. Note that the objective function~\eqref{MLE} in the MLE is linear, but the constraint $ Y\in \mathcal{Y} $ involves a set $ \mathcal{Y} $ that is discrete, non-convex and exponentially large. We replace this non-convex constraint with a trace norm (a.k.a.\  nuclear norm) constraint and a set of linear constraints. This leads to the convexified MLE given in Algorithm~\ref{alg:convex}. Here the trace norm $ \Vert Y \Vert_* $ is defined as the sum of the singular values of $ Y $. Note that the true $ Y^* $ is feasible to the optimization problem~\eqref{eq:CVX}--\eqref{eq:linear} since $ \Vert Y^*\Vert_* = \text{trace}(Y^*) = rK.  $

\begin{algorithm}[H]
\caption{Convexified Maximum Likelihood Estimator ($ p>q $)\label{alg:convex}}
\begin{align}
\hat{Y} = \arg\max_{Y}  & \sum_{i,j} A_{ij} Y_{ij} \label{eq:CVX} \\
\text{s.t.	} & \; \Vert Y \Vert_* \le rK, \label{eq:norm}\\
 & \;  \sum_{i,j} Y_{ij} = rK^2, \quad 0\le Y_{ij} \le 1, \forall i,j. \label{eq:linear}
\end{align}
\end{algorithm}
\noindent  The optimization problem in Algorithm~\ref{alg:convex} is a semidefinite program (SDP) and can be solved in polynomial time by standard interior point methods or various fast specialized algorithms such as ADMM; e.g., see~\cite{jalali2012maxnorm,ames2012clustering}. Similarly to Algorithm~\ref{alg:mle}, this algorithm is not strictly tied to the planted clustering model as it can also be considered as a relaxation of Correlation Clustering or Balanced Cut. In the case where the values of $ r $ and $ K $ are unknown, one may replace the hard constraints~\eqref{eq:norm} and~\eqref{eq:linear} with an appropriately weighted objective function; cf.~\cite{Chen12}.

The following theorem provides a sufficient condition for the success of the convexified MLE. See Table~\ref{tab:four} for its implications for specific planted models.
\begin{theorem}[Easy] \label{thm:CVX}
Under the planted clustering model with $p>q$, there exists a universal constant $ c_1$ such that with probability at least $ 1-n^{-10} $, the optimal solution to the problem~\eqref{eq:CVX}--\eqref{eq:linear} is unique and equal to $ Y^* $ provided
\begin{align}
 K^2 (p-q)^2 \ge c_1  \left[ p (1-q) K\log n + q(1-q)n \right]. \label{eq:Condition}
 \end{align}
\end{theorem}

When $ r=1 $, we refer to the regime where the condition~\eqref{eq:Condition} holds and~\eqref{eq:ConditionSimple} below fails as the \emph{easy regime}. When $ r>1 $, the easy regime is where~\eqref{eq:Condition} holds and~\eqref{eq:ConditionSimple} or~\eqref{eq:ConditionSimple2} below fails.

If $ p/q=\Theta(1) $,
it is easy to see that the smallest possible cluster size allowed by~\eqref{eq:Condition} is $ K =\Theta(\sqrt{n}) $ and the largest number of clusters is $ r=\Theta(\sqrt{n}) $, both of which are achieved when $ p,q,|p-q|=\Theta(1) $. This generalizes the tractability threshold $ K=\Omega(\sqrt{n}) $ of the classic planted clique problem. If $ q=o(p)$ (we call it the high SNR setting), the condition~\eqref{eq:Condition} becomes to $Kp \gtrsim \max \{ \log n, \sqrt{qn} \}$. In this case, it is possible to go beyond the $ \sqrt{n} $ limit on the cluster size. In particular, when $p=\Theta(1)$, the smallest possible cluster size is $K = \Theta(\log n\vee \sqrt{qn} )$, which can be much smaller than $\sqrt{n}$.

\begin{remark}
Theorem~\ref{thm:CVX} immediately implies guarantees for other tighter convex relaxations. Define the sets $ \mathcal{B} := \{Y\vert Eq.\eqref{eq:linear} \text{ holds} \} $ and
\begin{align*}
\mathcal{S}_1 &:= \{Y \;|\; \|Y\|_* \le rK\},\\
\mathcal{S}_2 &:= \{Y \;|\; Y \succeq 0; \text{trace}(Y) = rK\}.
\end{align*}
The constraint in Algorithm~\ref{alg:convex} corresponds to $ Y \in \mathcal{S}_1 \cap \mathcal{B} $, while
$ Y \in \mathcal{S}_2 \cap\mathcal{B}$ is the constraint in the standard SDP relaxation. Clearly
$ \left(\mathcal{S}_1 \cap \mathcal{B}\right)
\supseteq \left(\mathcal{S}_2 \cap \mathcal{B}\right)
\supseteq \mathcal{Y}. $
Therefore, if we replace the constraint~\eqref{eq:norm} with $ Y \in \mathcal{S}_2 $, we obtain a \emph{tighter} relaxation of the MLE, and  Theorem~\ref{thm:CVX} guarantees that it also succeeds to recover $ Y^* $ under the condition~\eqref{eq:Condition}. The same is true if we consider other tighter relaxations, such as those involving the triangle inequalities~\cite{mathieu}, the row-wise constraints $ \sum_{j}Y_{ij}\le K,\forall i $~\cite{ames2012clustering}, the max norm~\cite{jalali2012maxnorm} or the Fantope constraint~\cite{vu2013fantope}. For the purpose of this work, these variants of the convex formulation make no significant difference, and we choose to focus on~\eqref{eq:CVX}--\eqref{eq:linear} for generality.
\end{remark}

\subsubsection*{Converse for the trace norm relaxation approach} We have a partial converse to the achievability result in Theorem~\ref{thm:CVX}. The following theorem characterizes the conditions under which the trace norm relaxation~\eqref{eq:CVX}--\eqref{eq:linear} provably fails with high probability;
we suspect the standard SDP relaxation with the constraint $ Y \in \mathcal{S}_2 \cap\mathcal{B}$ also fails with high probability under the same conditions, but we do not have a proof.

\begin{theorem}[Easy, Converse]\label{thm:cvx_converse}
Under the planted clustering model with $p>q$, for any constant $ 1>\epsilon_0> 0 $, there exist positive universal constants $ c_1, c_2$ for which the following holds. Suppose $ c_1 \log n \le K \le \frac{n}{2}$, $q \ge c_1 \frac{\log n}{n}$  and  $p \le 1- \epsilon_0$.  If
\[K^2 (p-q)^2  \le c_2  ( Kp + q n), \]
then with probability at least $ 1-n^{-10} $, $Y^\ast$ is not an optimal solution of the program~\eqref{eq:CVX}--\eqref{eq:linear}.
\end{theorem}

Theorem~\ref{thm:cvx_converse} proves the failure of our trace norm relaxation that has access to the \emph{exact} number and sizes of the clusters. Consequently, replacing the constraints~\eqref{eq:norm} and~\eqref{eq:linear} with a Lagrangian penalty term in the objective would not help for \emph{any} value of the Lagrangian multipliers.
Under the assumptions of Theorems~\ref{thm:CVX} and~\ref{thm:cvx_converse},  by ignoring log factors,  the \emph{sufficient and necessary} condition for the success of our convexified MLE is
\begin{equation}\label{eq:poly_boundary}
 \frac{p}{K(p-q)^2 } + \frac{q n}{K^2(p-q)^2} \lesssom 1,
\end{equation}
whereas the success condition~\eqref{eq:hardsimple} for the MLE simplifies to
\begin{align*}
\frac{p}{K(p-q)^2} \lesssom 1.
\end{align*}
We see that the convexified MLE is statistically sub-optimal due to the extra second term in~\eqref{eq:poly_boundary}.
This term
is responsible for the $ K=\Omega(\sqrt{n}) $  threshold on the cluster size for the tractability of planted clique.
The term has an interesting interpretation. Let $\widetilde{A}:=A-q\mathbf{1}\mathbf{1}^\top + qI$ be the centered adjacency matrix. The matrix $E:= (Y-\mathbf{1}\mathbf{1}^\top) \circ (\widetilde{A}- \mathbb{E}\widetilde{A})$,\footnote{Here $ \circ $ denotes the element-wise product.} i.e., the deviation $ \widetilde{A} -\mathbb{E}\widetilde{A}$ restricted to the inter-cluster node pairs, can be viewed as the ``cross-cluster noise matrix''. Note that the squared largest singular value of the matrix $\mathbb{E}\widetilde{A}=(p-q)Y^\ast$ is $K^2(p-q)^2$, whereas the squared largest singular value of $ E $ concentrates around $\Theta(qn)$ (see e.g., \cite{Chattergee12}). Therefore, the second term $ \frac{q n}{K^2(p-q)^2} $ in~\eqref{eq:poly_boundary} is the ``spectral noise-to-signal ratio'' that determines the performance of the convexified MLE. In fact, our proofs for Theorems~\ref{thm:CVX} and~\ref{thm:cvx_converse} build on this intuition.

\paragraph*{Comparison to previous work} We refer to~\cite{Chen12} for a survey of the performance of state-of-the-art polynomial-time algorithms under various planted models. Theorem~\ref{thm:CVX} matches and in many cases improves upon existing results in terms of the scaling. For example, for planted partition, the previous best results are $ (p-q)^2\gtrsim p(K\log^4n+n)/K^2 $ in~\cite{Chen12} and $ (p-q)^2 \gtrsim pn\polylog n/K^2 $ in~\cite{anandkumar2013tensormixed}. Theorem~\ref{thm:CVX} removes some extra $ \log n $ factors, and is also order-wise better when $ q=o(p) $ (the high SNR case) or $ 1-q=o(1) $. For planted $ r $-disjoint-clique, existing results require $ 1-q $ to be $ \Omega((rn+rK\log n)/K^2 )$~\cite{McSherry01}, $ \Omega(\sqrt{n}/K) $~\cite{ames2010kclique} or $ \Omega((n+K\log^4n)/K^2) $~\cite{Chen12}. We improve them to $ \Omega((n+K\log n)/K^2) $.

Our converse result in Theorem~\ref{thm:cvx_converse} is inspired by, and improves upon, the recent work in~\cite{vinayak2013sharp}, which focuses on the special case $p>1/2>q$, and considers a convex relaxation approach that is equivalent to our relaxation~\eqref{eq:CVX}--\eqref{eq:linear} but without the additional equality constraint in~\eqref{eq:linear}. The approach is shown to fail when $K^2(p-\frac{1}{2})^2 \lesssim qn$. Our result is stronger in the sense that it applies to a tighter relaxation and a larger region of the parameter space.

\subsubsection*{Limits of polynomial-time algorithms}
By comparing the recovery limit established in Theorems~\ref{thm:Impossible} and~\ref{thm:MLE} with the performance limit of our convex method established in Theorem~\ref{thm:CVX},
we get two strikingly different observations. On one hand, if $pK \log n =\Omega(nq)$ and $\log K =\Omega( \log n)$,
the recovery limit and performance limit of our convex method coincide up to constant factors at $K(p-q)^2 \asymp p(1-q) \log n$.
Thus, the convex relaxation is tight and the hard regime disappears up to constants, even though the hard regime may still exist when constant factors are concerned~\cite{Mossel12,Decelle11}.
In this case, we get a computationally efficient and statistically order-optimal estimator.
On the other hand, if $pK\log n=o(nq)$, there exists a substantial gap between  the information limit and performance limit of  our convex method.
We conjecture that no polynomial-time algorithm has order-wise better statistical performance than the convexified MLE and succeeds significantly beyond the condition~\eqref{eq:Condition}.
\begin{conjecture} \label{ConjectureHardness}
For any constant $ \epsilon >0 $,  there is no algorithm with running time polynomial in $ n $ that, for all $ n $ and with probability at least $ 1/2 $, outputs the true $Y^\ast$ of the planted clustering problem with $ p>q $ and
\begin{align}
(p-q)^2 K^2 \le   n^{-\epsilon} \left( K p(1-p) + q(1-q)n \right). \label{ConjectureHardnessCondition}
\end{align}
\end{conjecture}
If the conjecture is true, then in the asymptotic regime $p=2q=n^{-\alpha}$ and $K=n^\beta$,
the \emph{computational limit} for the cluster recovery is given by $\beta= \frac{\alpha}{2} + \frac{1}{2}$, i.e.,
the boundary between the green regime and red regime in Fig.~\ref{fig:summary}.

A rigorous proof of Conjecture~\ref{ConjectureHardness} seems difficult with current techniques.
There are other possible convex formulations for planted clustering. The space of possible polynomial-time algorithms is even larger. It is impossible for us to study each of them separately and obtain a converse result as in Theorem~\ref{thm:cvx_converse}.
There are however several evidences that support the conjecture:
\begin{compactitem}
\item The special case with $ p=2q=1$ corresponds to the $ K=o(\sqrt{n}) $ regime for the classical Planted Clique problem, which is conjectured to be computationally hard~\cite{alon2007testing,rossman2010clique,Feldman2012statAlg}, and was used as an assumption for proving other hardness results~\cite{Hazan2011Nash,Juel00cliqueCrypto,koiran2012rip}. Conjecture~\ref{ConjectureHardness} can be considered as \emph{a generalization of the Planted Clique conjecture} to the setting with multiple clusters and general values of $ p $ and $ q $, and may be used to study the computational hardness of other problems~\cite{chen2013incoherence_arxiv}.
\item It is shown in  \cite{HajekWuXu14} that for the special setting with a single cluster, no polynomial-time algorithm can reliably recover
the planted cluster if $\beta< \alpha/4+1/2$ conditioned on the planted clique hardness hypothesis. Here the planted clique hardness hypothesis refers to the statement that for any fixed constants $\gamma > 0$ and $\delta > 0$,
there exist no randomized polynomial-time tests to distinguish an Erd\H{o}s-R\'enyi random graph $\calG(n, \gamma)$ and a planted
clique model which is obtained by adding edges to $K=n^{1/2-\delta}$
vertices chosen uniformly from $\calG(n, \gamma)$ to form a clique.

\item As discussed earlier,
 if~\eqref{ConjectureHardnessCondition} holds, then the graph spectrum is dominated by noise and fails to reveal the underlying cluster structure. The condition~\eqref{ConjectureHardnessCondition} therefore represents a ``spectral barrier'' for clustering.
The work in~\cite{nadakuditi} uses a similar spectral barrier argument to prove the failure of a large class of algorithms that rely on the graph spectrum; our Theorem~\ref{thm:cvx_converse} shows that the convexified MLE fails for a similar reason.
\item In the sparse graph case with $p,q=O(1/n)$, it is argued in~\cite{Decelle11}, using non-rigorous but deep arguments from statistical physics, that it is intractable to achieve the correlated recovery under Condition~\eqref{ConjectureHardnessCondition}.
\end{compactitem}

\subsection{The Simple Regime: A Counting Algorithm}\label{sec:trivial}
In this subsection, we consider a simple recovery procedure in Algorithm~\ref{alg:counting}, which is based on counting node degrees and common neighbors.
\begin{algorithm}[h!]
\caption{A Simple Counting Algorithm}\label{alg:counting}
\begin{compactenum}
\item (Identify isolated nodes) For each node $ i $, compute its degree $ d_i $. Declare $ i $ as isolated if $ d_i < \frac{(p-q)K}{2} + qn. $
\item (Identify clusters when $r>1$)
For every pair of non-isolated nodes $i,j$, compute the number of common neighbors $S_{ij}:= \sum_{k:k\neq i,k\neq j} A_{ik}A_{jk}$, and
assign them into the same cluster if $S_{ij}>\frac{(p-q)^2K}{3} + 2Kpq + q^2 (n-2K)$. Declare error if inconsistency found.
\end{compactenum}
\end{algorithm}

We note that steps 1 and~2 of Algorithm~\ref{alg:counting} are considered in~\cite{Kucera95} and~\cite{DyerFrieze89} respectively for the special cases of recovering a single planted clique or two planted clusters. Let $ E $ be the set of edges. It is not hard to see that step 1 runs in time $O(|E|)$ and step 2 runs in time $O(n|E|)$, since each node only needs to look up its local neighborhood up to distance two. It is possible to achieve even smaller expected running time using clever data structures.

The following theorem provides sufficient conditions for the simple counting algorithm to succeed. Compared to the previous work in~\cite{Kucera95,DyerFrieze89}, our results apply to  general values of $ p,q $, $ r $, and $ K $. See Table~\ref{tab:four} for its implications for specific planted models.
\begin{theorem}[Simple] \label{thm:Simple}
For planted clustering with $p>q$, there exist universal constants $ c_1, c_2 $ such that Algorithm~\ref{alg:counting} correctly finds the isolated nodes with probability at least $ 1-2 n^{-1} $ if
\begin{align}
K^2(p-q)^2\ge c_1  [K p(1-q)  + n q(1-q) ]  \log n ,
\label{eq:ConditionSimple}
\end{align}
and finds the clusters with probability at least $1-4n^{-1}$ if further
\begin{align}
K^2(p-q)^4\ge c_2 [K  p^2 (1-q^2) + n q^2 (1-q^2) ] \log n.
\label{eq:ConditionSimple2}
\end{align}
\end{theorem}
\begin{remark}\label{rem:count_non_neghber}
If $p,q \to 1$ as $n \to \infty$, we can obtain slightly better performance by counting the common \emph{non}-neighbors in Step 2, which succeeds under condition~\eqref{eq:ConditionSimple2} with $p$ and $q$ replaced by $1-p$ and $1-q$, respectively, i.e., the RHS of (18) simplifies to $c_2n(1-q)^2 \log n$.
\end{remark}
In the case with a single clusters $r=1$,  we refer to the regime where  the condition~\eqref{eq:ConditionSimple} holds as the {\it simple regime}; in the case with $r>1$,  the {simple regime} is where both conditions~\eqref{eq:ConditionSimple}
and~\eqref{eq:ConditionSimple2} hold. It is instructive to compare these conditions with the success condition~\eqref{eq:Condition} for the convexified MLE. The condition~\eqref{eq:ConditionSimple} has an additional $ \log n $ factor on the RHS. This means when $ r=1 $ and the only task is to find the isolated nodes, the counting algorithm performs nearly as well as the sophisticated convexified MLE. On the other hand, when $ r>1 $ and one needs to distinguish between different clusters, the convexified MLE order-wise outperforms the counting algorithm whenever $ p/q=\Theta(1) $, as the condition~\eqref{eq:ConditionSimple2} is order-wise more restrictive than~\eqref{eq:Condition}. Nevertheless, when $ p,q,p-q=\Theta(1) $, both algorithms can recover $ \tilde{O}(\sqrt{n}) $ clusters of size $ \tilde{\Omega}(\sqrt{n}) $, making the simple counting algorithm a legitimate candidate in such a setting and a benchmark to which other algorithms can be compared with.

In the high SNR case with $q=o(p)$, the counting algorithm can recover clusters with size much smaller than $ \sqrt{n} $; e.g., if $p=\Theta(1)$ and $q=o(1)$, it only requires  $ K \gtrsim \max \{ \log n , \sqrt{qn\log n} \}$.

\subsubsection*{Converse for the counting algorithm} We have a (nearly-)matching converse to Theorem~\ref{thm:Simple}. The following theorem characterizes when the counting algorithm provably fails.
\begin{theorem}[Simple, Converse]\label{thm:SimpleConverse}
Under the planted clustering model with $p>q$, for any constant $ 0<\epsilon_0<1 $, there exist universal constants $c_1, c_2>0 $ for which the following holds. Suppose $K \le \frac{n}{2}$, $p\le 1-\epsilon_0$, $ q \ge c_1 \log n /n$ and $K p^2   + n q^2  \ge c_1 \log n$.
Algorithm~\ref{alg:counting} fails to correctly identify all the isolated nodes with probability at least $1/4$ if
\begin{align}\label{eq:simple_fail_1}
K^2 (p-q)^2 < c_2  \left[ ( K p + n q  ) \log (rK) + n q\log (n-rK) \right],
\end{align}
and fails to correctly recover all the clusters with probability at least $1/4$ if
\begin{align}\label{eq:simple_fail_2}
K^2 (p-q)^4 < c_2 ( K p^2+ n q^2 ) \log (rK).
\end{align}
\end{theorem}
\begin{remark}
Theorem \ref{thm:SimpleConverse} requires a technical condition $K p^2   + n q^2  \ge c_1 \log n$, which is actually not too restrictive.
If $Kp^2+nq^2 =o(\log n)$, then two nodes from the same cluster will have no common neighbor with probability $ (1-p^2)^{K} (1-q^2)^{n-K} \ge \exp[-\Theta(p^2K+q^2(n-K))]  = \exp[-o(\log n)]$, so Algorithm~\ref{alg:counting} cannot succeed with the probability specified in Theorem~\ref{thm:Simple}.
\end{remark}

Apart from some technical conditions,
Theorems~\ref{thm:Simple} and \ref{thm:SimpleConverse} show that the conditions~\eqref{eq:ConditionSimple} and~\eqref{eq:ConditionSimple2} are both sufficient and necessary. In particular, the counting algorithm cannot succeed outside the simple regime, and is indeed strictly weaker in separating different clusters as compared to the convexified MLE. Our proof reveals that the performance of the counting algorithm is limited by a \emph{variance barrier}: The RHS of~\eqref{eq:ConditionSimple} and~\eqref{eq:ConditionSimple2} are associated with the {variance} of the node degrees and common neighbors (i.e., $ d_i $ and $ S_{ij} $ in Algorithm~\ref{alg:counting}), respectively. There exist nodes whose degrees deviate from their expected value on the order of the standard deviation, and if the condition~\eqref{eq:ConditionSimple} does not hold, then the deviation will outweigh the difference between the expected degrees of the isolated nodes and those of the non-isolated nodes. A similar argument applies to the number of common neighbors.

\section{Main Results for Submatrix Localization}
\label{sec:submatrix}

In this section, we turn to the submatrix localization problem, sometimes known as bi-clustering~\cite{balakrishnan2011tradeoff}. We consider the following specific setting, which is defined by six parameters $ n_L,n_R, K_L, K_R, r \in \mathbb{N}$, and $ \mu \in\mathbb{R}_+ $ such that $n_L \ge r K_L$ and $n_R \ge r K_R$. We use the shorthand notation $n:= n_{L} \vee n_{R}$.

\begin{definition}[Submatrix Localization]\label{def:submatrix}
A random matrix $ A\in\mathbb{R}^{n_L\times n_R} $ is generated as follows. Suppose that $ rK_L $ rows of $A$ are partitioned into $r$ disjoint subsets $\left\{ C^\ast_{1},\ldots,C^\ast_{r}\right\} $ of equal size $K_L$, and  $ rK_R $ columns of $A$ are partitioned into $r$ disjoint subsets  $\{D^\ast_{1},\ldots,D^\ast_{r}\}$ of equal size $K_R$. For each $(i,j)$, we have $A_{ij}=\mu + \Delta_{ij}$ if $ (i,j) \in C^\ast_m \times D^\ast_m $ for some $ m\in[r] $ and $ A_{ij} =\Delta_{ij}$ otherwise, where  $ \mu>0 $ is a fixed number and $(\Delta_{ij}) $ are i.i.d. zero-mean sub-Gaussian random variables with parameter $ 1 $.\footnote{A random variable $ X $ is said to be sub-Gaussian with parameter $ 1 $ if $ \mathbb{E}[e^{tX}] \le e^{t^2/2}$ for all $ t\in\mathbb{R} $. } The goal is to recover the locations of the hidden submatrices $ \{ (C^\ast_m, D^\ast_m), m\in[r]\}  $ given the matrix~$A$.
\end{definition}

In the language of bi-clustering, the sets $\left\{ C^\ast_{1},\ldots,C^\ast_{r}\right\} $ are called \emph{left clusters} and $\{D^\ast_{1},\ldots,D^\ast_{r}\}$ are called
\emph{right clusters}. Row (column, resp.)\ indices which do not belong to any cluster are called \emph{isolated} left (right, resp.)\ nodes. One can think of $ A $ as the bipartite affinity matrix between the $ n_L $ left nodes and $n_R $ right nodes, and the goal is to recover the left and right clusters. Similarly as before, we define the \emph{bi-clustering matrix} $Y^{*}\in\left\{ 0,1\right\} ^{n_{L}\times n_{R}}$, where $Y^\ast_{ij}=1$ if and only if $ (i,j) \in C^\ast_m \times D^\ast_m$ for some $m \in [r] $. The problem reduces to recovering $Y^\ast$ given $A$.

As before, all the parameters $ \mu, K_L, K_R, r$  are allowed to scale with $ n_L $ and $ n_R $, and we assume that their values are known. Note that it is without loss of generality to assume the mean of $ A_{ij} $ is zero outside the submatrices and the variance of $ A_{ij} $ is one, because otherwise we can shift and rescale $ A $. The above model generalizes the previous submatrix localization/detection models~\cite{ma2013submatrix,butucea2011submatrix,arias2011anomalous} and bi-clustering models~\cite{kolar2011submatrix,balakrishnan2011tradeoff} which consider the special case with a \emph{single} submatrix (i.e., $ r=1 $).

In the next four subsections, we shall focus on the low-SNR setting $\mu^2 =O(\log n)$ and present theorems establishing the four regimes. These results parallel those for the planted clustering. In the high SNR setting $\mu^2 =\Omega(\log n)$,  the submatrices can be easily identified by naive element-wise thresholding, so we deal with this case separately in the last subsection.

\subsection{The Impossible Regime: Minimax Lower Bounds}
The following theorem gives conditions on $(n_L,n_R,K_L,K_R,\mu)$ under which the minimax error probability is large and thus it is informationally impossible to reliably locate the submatrices. With slight abuse of notation, we use $\mathcal{Y} \subset \left\{ 0,1\right\} ^{n_{L}\times n_{R}} $ to denote the set of all possible bi-clustering matrices corresponding to $r$ left (right, resp.) clusters of equal size $K_L$ ($K_R$, resp.).
\begin{theorem}[Impossible]\label{thm:Impossible_bi}
Under the submatrix localization model, suppose $\{A_{ij} \}$ are Gaussian random variables, $K_L \le n_L/2$, $K_R \le n_R /2$, and
$n_L, n_R \ge 128$. If
\begin{equation}\label{eq:impossible_bi}
\mu^{2}\le \frac{1}{12} \max\left\{ \frac{\log\left(n_{R}-K_{R}\right)}{K_{L}},\frac{\log\left(n_{L}-K_{L}\right)}{K_{R}}\right\} ,
\end{equation}
then
$
\inf_{\hat{Y}}\sup_{  Y^{\ast} \in \mathcal{Y} } \mathbb{P}\left[\hat{Y}\neq Y^{*}\right]\ge\frac{1}{2},
$
where the infimum ranges over all measurable functions of~$A$.
\end{theorem}
The regime where~\eqref{eq:impossible_bi} holds is called the \emph{impossible} regime, corresponding to an information barrier that no algorithm can break. We note the similarity between the impossible regimes for submatrix localization and planted clustering. In particular, if we assume the in/cross-cluster edges in planted clustering have comparable variance, i.e., $\frac{p(1-p)}{ q(1-q)} =\Theta(1)$, then the conditions~\eqref{eq:impossible_bi} and~\eqref{eq:impossiblesimle_1} coincide up to constant factors by setting $ n_L=n_R=n, K_L=K_R=K $ and $ \mu = \frac{p-q}{\sqrt{q(1-q)}}$. Such correspondence also exists in the next three regimes.

\paragraph*{Comparison to previous work} Theorem~\ref{thm:Impossible_bi} holds in the general high rank setting with arbitrary $ r $. In  $ r=1 $ case, our result  recovers the minimax lower bound in~\cite{kolar2011submatrix}.

\subsection{The Hard Regime: Optimal Algorithm}
Recall that $\mathcal{Y}$ is the set of all valid bi-clustering matrices. We consider the combinatorial optimization
problem given in Algorithm~\ref{alg:mle_bi}. In the setting where $\{\Delta_{ij}\}$ are Gaussian random variables, this can be shown to be the MLE of $ Y^* $.
\begin{algorithm}
\caption{Maximum Likelihood Estimator\label{alg:mle_bi}}
\begin{equation}
\hat{Y} = \arg\max_{Y\in\mathcal{Y}} \; \sum_{i,j} A_{ij}Y_{ij}.\label{eq:mle2}
\end{equation}
\end{algorithm}

Theorem~\ref{thm:MLE_bi} below provides a success condition for Algorithm~\ref{alg:mle_bi}.
\begin{theorem}
[Hard]\label{thm:MLE_bi} Suppose $K_{L},K_{R}\ge8$.  There exists a  constant $ c_1$ such that with probability at least $ 1-512en^{-1} $,  the optimal solution to the problem~\eqref{eq:mle2} is
unique and equals $ Y^{*}$ if
\begin{equation}\label{eq:hard_bi}
\mu^{2}\ge c_1 \frac{\log n}{K_{L}\wedge K_{R}}.
\end{equation}
\end{theorem}

We refer to the regime where the condition~\eqref{eq:hard_bi} holds and~\eqref{eq:CVX_cond_bi} fails as the \emph{hard} regime. Note that the bound~\eqref{eq:hard_bi} matches~\eqref{eq:impossible_bi} up to a constant factor, so they are minimax optimal. Therefore,  Theorems~\ref{thm:Impossible_bi} and~\ref{thm:MLE_bi} together establish  the minimax recovery boundary for submatrix localization at $ \mu^2 \asymp \frac{ \log n}{K_{L}\wedge K_{R}}. $

\paragraph*{Comparison with previous work}  Theorem~\ref{thm:MLE_bi} provides the first minimax-optimal achievability result when the number $ r $ of submatrices may grow with $ n_L$ and $n_R $. In particular, $ r $ is allowed to grow at a nearly linear rate $ r=O(n/\log n) $ assuming $ n_L=n_R=n $. In the special case with a single planted submatrix ($ r=1 $), Theorem~\ref{thm:MLE_bi} recovers the achievability result in~\cite{kolar2011submatrix}.

\subsection{The Easy Regime: Polynomial-Time Algorithms}

As previous, we obtain a convex relaxation of the combinatorial MLE formulation~\eqref{eq:mle2} by replacing the constraint $ Y\in \mathcal{Y} $ with the trace norm and linear constraints, for which we use the fact that the true $ Y^* $ satisfies $ \left\Vert Y^{*}\right\Vert _{*}=r\sqrt{K_L K_R} $. This is given as Algorithm~\ref{alg:cvx_bi}, which is a semidefinite program (SDP) and can be solved in polynomial time.
\begin{algorithm}[H]
\caption{Convexified Maximum Likelihood Estimator\label{alg:cvx_bi}}
\begin{eqnarray}
\max_{Y} &  & \sum_{i,j}A_{ij}Y_{ij}\label{eq:CVX_bi}\\
 &  & \left\Vert Y\right\Vert _{*}\le r\sqrt{K_LK_R},\label{eq:nuclear_bi} \\
 &  & \sum_{i,j}Y_{ij}=rK_{L}K_{R}, \quad 0\le Y_{ij}\le1,\forall i,j.\label{eq:linear_bi}
\end{eqnarray}
\end{algorithm}

The following theorem provides a sufficient condition for the success of Algorithm~\ref{alg:cvx_bi}.
\begin{theorem}
[Easy]\label{thm:CVX_bi} There exists a universal constant $ c_1 $ such that with probability at least $ 1-n^{-10} $, the optimal solution
to the program~\eqref{eq:CVX_bi}--\eqref{eq:linear_bi} in Algorithm~\ref{alg:cvx_bi} is unique and equals $Y^{*}$ if
\begin{equation}\label{eq:CVX_cond_bi}
\mu^{2} \ge c_1 \left(\frac{\log n}{K_{L}\wedge K_{R}}+\frac{ n }{K_{L}K_{R}}\right).
\end{equation}
\end{theorem}

When $ r=1 $, the \emph{easy regime} refers to where the condition~\eqref{eq:CVX_cond_bi} holds but~\eqref{eq:ConditionSimple1_bi} fails. When $ r>1 $, the easy regime is where the condition~\eqref{eq:CVX_cond_bi} holds but~\eqref{eq:ConditionSimple2_bi} fails. Suppose $n_{L}=n_{R}=n$ and $K_{L}=K_{R}=K$; the convexified MLE is guaranteed to succeed when $ \mu^2 \gtrsim \frac{K\log n + n }{K^2}$.

The following theorem provides a nearly matching converse to Theorem~\ref{thm:CVX_bi}.
\begin{theorem}[Easy, Converse]
\label{thm:Cvx_converse_bi} There exist positive universal constants $ c_1, c_2 $ such that the following holds. Under the submatrix localization model, suppose $ \mu \le 1/100 $, $n_{L}=n_{R}=n$, $K_{L}=K_{R}=K$, $ c_1 \log n \le  K \le \frac{n}{2} $, and $(\Delta_{ij})$ are Gaussian random variables. If
\begin{equation}
\mu^2 \le c_2  \frac{n}{K^2}, \label{eq:converse_bi_cond}
\end{equation}
then with probability at least $ 1-n^{-10} $, any optimal solution to the convex program~\eqref{eq:CVX_bi}--\eqref{eq:linear_bi} has a different support from $ Y^* $.
\end{theorem}

Theorems~\ref{thm:CVX_bi} and~\ref{thm:Cvx_converse_bi} together establish that the recovery boundary for the convexified MLE in Algorithm~\ref{alg:cvx_bi} is $ \mu^2 \asymp \frac{n}{K^2}$ ignoring logarithmic factors. There is a substantial gap from the minimax boundary $\mu^2 \asymp \frac{1}{K}$  established in the last two subsections (again ignoring logarithmic factors). Our analysis reveals that the performance of the convexified MLE is determined by a spectral barrier similar to that in planted clustering. In particular, the squared largest singular values of the \emph{signal matrix} $ Y^* $ and the \emph{noise matrix} $ A-\mathbb{E}A $ are $ \Theta(\mu^2 K^2) $ and $\Theta(n) $, respectively, so the condition $ \mu^2 \gtrsim \frac{n}{K^2} $ for the convexified MLE can be seen as an spectral SNR condition.

As in the planted clustering model,  we conjecture that no polynomial-time algorithm can achieve better statistical performance than the convexified MLE.
\begin{conjecture} \label{ConjectureHardnessBi}
For any constant $ \epsilon>0 $, there is no algorithm with running time  polynomial in $ n $ that, for all $ n $ and with probability at least $ 1/2$, outputs the true $Y^\ast$ for the submatrix localization problem with $ \mu \le 1$, $n_{L}=n_{R}=n$, $K_{L}=K_{R}=K \ge c_1 \log n$ and $$\mu^2 \le \frac{n^{1-\epsilon}}{K^2}.$$
\end{conjecture}

\paragraph*{Comparison with previous work} The achievability and converse results in Theorems~\ref{thm:CVX_bi} and~\ref{thm:Cvx_converse_bi} hold even when $ r $ grows with $ n $. In the special case with $r=1$, the work in~\cite{kolar2011submatrix} considers a convex relaxation of sparse singular value decomposition; they focus on the high SNR regime with $\mu^2 \gtrsim \log n$, and show that the performance of their convex relaxation is no better than a simple element-wise thresholding approach (cf.~Section~\ref{sec:elementwise}). Our convex program is different from theirs, and succeeds in the low SNR regime provided $\mu^2 \gtrsim \frac{K \log n+ n}{K^2}$. The work in~\cite{ames2012clustering} studies the success conditions of a convex formulation similar to~\cite{kolar2011submatrix}; with the additional assumption of bounded support of the distribution of $ A_{ij} $, they show that their approach succeeds under an order-wise more restricted condition~$ \mu^2 \gtrsim \frac{n\cdot r}{K^2} $.

\subsection{The Simple Regime: A Thresholding Algorithm}

We consider a simple thresholding algorithm as given in Algorithm~\ref{alg:counting_bi}. The algorithm computes the column and row sums of $ A $ as well as the correlation between the columns and rows. It is similar in spirit to the simple counting Algorithm~\ref{alg:counting} for the planted clustering problem.

\begin{algorithm}
\caption{A Simple Thresholding Algorithm\label{alg:counting_bi}}
\begin{compactenum}
\item (Identify isolated nodes) For each left node $i\in[n_{L}]$, declare it as isolated if the row sum $d_i := \sum_{j=1}^{n_{R}}A_{ij}\le\frac{\mu K_{R}}{2}$.
For each  right node $j\in[n_{R}]$, declare it as isolated if the column sum $d'_j : = \sum_{i=1}^{n_{L}}A_{ij}\le\frac{\mu K_{L}}{2}$.
\item (Identify clusters when $r>1$)
For each pair of non-isolated left nodes $i,i'\in[n_{L}]$, assign them to the same cluster if $S_{ii'} := \sum_{j=1}^{n_{R}}A_{ij}A_{i'j}\ge \frac{\mu^{2}K_{R}}{2}$. Declare error if inconsistency is found. Assign the non-isolated right nodes into clusters in a similar manner. Let $ \{C_k\} $ and $\{ D_k \}$ be the resulting left and right clusters.
\item (Associate left and right clusters) For each $k\in[r]$ and $l\in[r]$, associate the left cluster
$C_{k}$ with the right cluster $D_{l}$ if the block sum $ B_{kl} := \sum_{i\in C_{k},j\in D_{l}}A_{ij}\ge \mu K_L K_R/2 $.
\end{compactenum}
\end{algorithm}
Steps 1, 2 and 3 of the algorithm run in time $O(n_Ln_R)$, $O(n_L^2n_R + n_R^2 n_L)$ and $ O(n_Ln_R) $, respectively. We note that Step~1 is previously considered in~\cite{kolar2011submatrix} for locating a single submatrix. The following theorem provides success conditions for this simple algorithm.
\begin{theorem}
[Simple] \label{thm:Simple_bi} There exist universal constants $ c_1, c_2 $ such that Algorithm~\ref{alg:counting_bi}
identifies the isolated nodes with probability at least $ 1-en_L^{-1}-en_R^{-1}$ if
\begin{equation}
\mu^{2} \ge c_1 \max\left\{ \frac{n_{L} \log n_R }{K_{L}^{2}},\frac{n_{R} \log n_L }{K_{R}^{2}}\right\} ,\label{eq:ConditionSimple1_bi}
\end{equation}
and exactly recovers $Y^\ast$ with probability at least $1-e(rK_L)^{-1}-e(rK_R)^{-1}-en^{-1}$ if further
\begin{equation}
\mu^{4} \ge c_2 \max\left\{ \frac{n_{L} \log (r K_R) } {K_{L}^{2}},\frac{n_{R} \log (r K_L)}{K_{R}^{2}}\right\}. \label{eq:ConditionSimple2_bi}
\end{equation}
\end{theorem}
When $r=1$, we refer to the regime for which the condition~\eqref{eq:ConditionSimple1_bi} holds as the \emph{simple} regime. When $r>1$, the {simple} regime is where both  conditions~\eqref{eq:ConditionSimple1_bi} and~\eqref{eq:ConditionSimple2_bi} hold.

We provide a converse to Theorem~\ref{thm:Simple_bi}. The following theorem shows that the conditions~\eqref{eq:ConditionSimple1_bi} and~\eqref{eq:ConditionSimple2_bi} are also (nearly) necessary for the simple thresholding algorithm to succeed.
\begin{theorem}[Simple, Converse] \label{thm:SimpleConverse_bi}
Suppose that $K_L, K_R \ge \log n$. Under the submatrix localization model where the distributions of $ \{A_{ij}\} $ are Gaussian, there exist universal constants $ c_1, c_2 $ such that with probability at least $ 1-n^{-10} $, Algorithm \ref{alg:counting_bi} fails to correctly identify all the isolated nodes if
\begin{align}
\mu^{2} \le c_1 \max\left\{ \frac{n_L \log n_R }{K_{L}^{2} },\frac{ n_R \log n_L }{K_{R}^{2}}\right\},
\end{align}
and fails to correctly recover all the clusters if $n_L \ge rK_R$, $n_R \ge rK_L$ and
\begin{align}
\mu^{4} \le c_2 \max\left\{ \frac{ n_L \log (r K_R) }{K_{L}^{2}},\frac{n_R \log (r K_L) }{K_{R}^{2}}\right\}.
\end{align}
\end{theorem}
When $n_L=n_R=n$, $K_L=K_R=K$,  Theorems~\ref{thm:Simple_bi} and Theorem~\ref{thm:SimpleConverse_bi} establish that the recovery boundary for the simple thresholding algorithm is $ \mu^2 \asymp \frac{n \log n}{K^2} $ if $r=1$, and $ \mu^2 \asymp \frac{\sqrt{n \log n}}{K} $ if $r>1$ and $rK=\Theta(n)$. Comparing with the success condition~\eqref{eq:CVX_cond_bi} for the convex optimization approach, we see that the simple thresholding algorithm is order-wise less powerful in separating different submatrices. Similar to planted clustering, the performance is determined by the variance barrier associated with the variance of the quantities $ d_i $ and $ S_{ii'} $ computed in Algorithm~\ref{alg:counting_bi}.

\subsection{The High SNR Setting}\label{sec:elementwise}

As mentioned before, the high SNR setting with $ \mu^2 = \Omega(\log n) $ can be handled by a simple element-wise thresholding algorithm, which is given in Algorithm~\ref{alg:element_bi}.
\begin{algorithm}
\caption{\label{alg:element_bi}Element-wise Thresholding for Submatrix Localization}

For each $(i,j) \in [n_L]\times [n_R]$, set $\hat{Y}_{ij}=1$ if $A_{ij}\ge\frac{1}{2}\mu$,
and $\hat{Y}_{ij}=0$ otherwise.
Output  $\hat{Y}.$
\end{algorithm}

For the special case with one submatrix ($ r=1 $), the success of element-wise thresholding in the high SNR setting is proved in~\cite{kolar2011submatrix}. Their result can be easily extended to the general case with $ r\ge1 $. We record this result in Theorem~\ref{thm:element_bi} below. The theorem also shows that element-wise thresholding fails if $ \mu^2=o(\log n) $, so it is not very useful in the low SNR setting.
\begin{theorem}[Element-wise Thresholding]\label{thm:element_bi}
There exists a universal constant $c_1>4 $ such that the following holds. Algorithm~\ref{alg:element_bi} outputs $\hat{Y}=Y^{*}$ with probability at least $1-n^{-3}$ provided
\begin{equation}
\mu^2 > c_{1}\log n.\label{eq:element_cond}
\end{equation}
If the distributions of the $ A_{ij} $'s are Gaussian, and $ K_L \le n_L/2 $ \emph{or} $K_R \le n_R /2$, then with probability at least $ 1-n^{-3} $, the output of  Algorithm~\ref{alg:element_bi} satisfies $\hat{Y}\neq Y^{*}$ if
\begin{equation}
\mu^2\le 4 \log n. \label{eq:element_fail_cond}
\end{equation}
\end{theorem}

\section{Discussion and Future Work} \label{sec:conclusion}

In this paper, we show that the planted clustering problem and the submatrix localization problem admit successively faster algorithms with weaker statistical performance. We provide sufficient and necessary conditions for the success of the combinatorial MLE, the convexified MLE and the simple counting/thresholding algorithm, showing that they work in successively smaller regions of the model parameters. This represents a series of tradeoffs between the statistical and computational performance. Our results indicate that there may exist a large gap between the information limit and the computational limit, i.e., the information limit might not be achievable via polynomial-time algorithms. Our results hold in the high-rank setting with a growing number of clusters or submatrices.

Several future directions are of interest. Immediate goals include removing some of the technical assumptions in our theorems.
It is useful in practice to identify a finer spectrum, ideally close to a continuum, of computational-statistical tradeoffs. It is also interesting to extend to the settings with overlapping clusters and submatrices, and to the cases where the values of the model parameters are unknown. Proving our conjectures on the computational hardness in the hard regime is also interesting and such attempt has been pursued in \cite{HajekWuXu14}.

\section{Proofs for Planted Clustering}\label{sec:proof}
Throughout this section, we consider the  planted clustering model with $p>q$. Let $ n_1 := rK $ and $ n_2 := n-rK $ be the numbers of non-isolated and isolated nodes, respectively.

\subsection{Proof of Theorem \ref{thm:Impossible} and Corollary~\ref{cor:impossible}}\label{sec:proof_fano}

In the sequel we will make use of the following upper and lower bounds on the KL divergence $ D(u\Vert v) $ between two Bernoulli distributions with parameter $ u \in [0,1] $ and $  v \in [0,1] $. We have
\begin{align}
D\left(u \Vert v  \right):= u \log \frac{u}{v} + (1-u) \log \frac{1-u}{1-v}
\overset{(a)}{\le} u \frac{u-v}{v} + (1-u) \frac{v-u}{1-v}
 = \frac{(u-v)^2}{v(1-v)}, \label{eq:boundDivergence}
\end{align}
where $(a)$ follows from the inequality $\log x \le x-1, \forall x \ge 0$.
Moreover, viewing $D(x \Vert v)$ as a function of $x $ and using a Taylor's expansion, we can find some $\xi \in [u\wedge v,u\vee v]$ such that
\begin{align}
D\left(u \Vert v  \right)= D \left ( v \Vert v \right) + (u-v) D' \left( v \Vert v \right) + \frac{(u-v)^2}{2} D'' \left ( \xi \Vert v \right)
\overset{(b)}{\ge} \frac{(u-v)^2}{2 (u\vee v)[1-(u\wedge v)]}, \label{eq:lowerboundDivergence}
\end{align}
where $(b)$ follows because $D' \left( v \Vert v \right) =0$ and $D'' \left ( \xi \Vert v \right) = 1/[\xi(1-\xi)]$.

Theorem \ref{thm:Impossible}  is established through the following three lemmas, each of which provides a sufficient condition for having a large error probability.
\begin{lemma} \label{lmm:impossible1}
Suppose that $128 \le K\le n/2$. Let $ \alpha:=\frac{n_1(K-1)}{n(n-1)} $ and $ \beta:=\alpha p + (1-\alpha)q $. We have $\inf_{\hat{Y}} \sup_{Y^*\in \mathcal{Y}} \mathbb{P} \left[\hat{Y} \neq Y^*\right] \ge \frac{1}{2}$ if
\begin{align}
K\cdot D(p\Vert \beta ) + \frac{n^2}{n_1}(1-\alpha)\frac{(q-\beta)^2}{\beta(1-\beta)} &\le \frac{1}{4}  \log \frac{n}{K}, \label{eq:precise}
\end{align}
Moreover, \eqref{eq:precise} is implied by
\begin{align}
K(p-q)^2 &\le \frac{1}{4} q(1-q) \log \frac{n}{K}, \label{eq:impossible}
\end{align}
\end{lemma}
\begin{proof}
We use an information theoretical argument via Fano's inequality. Recall that $\mathcal{Y}$
is the set of cluster matrices corresponding to $r$ clusters of size $K.$
Let $ \mathbb{P}_{(Y^*,A)} $ be the joint distribution of $(Y^\ast,A)$
when $ Y^* $ is sampled from $ \mathcal{Y} $ uniformly at random and then $ A $ is generated according to the planted clustering model based on $Y^\ast$. Lower-bounding the supremum by the average, we have
\begin{align*} 
 \inf_{\hat{Y}} \sup_{Y^*\in \mathcal{Y}} \mathbb{P} \left[\hat{Y} \neq Y^*\right]
 \ge  \inf_{\hat{Y}} \mathbb{P}_{(Y^*,A)}\left[\hat{Y} \neq Y^*\right].
\end{align*}
It suffices to bound $\mathbb{P}_{(Y^*,A)}\left[\hat{Y} \neq Y^*\right]$
from below. Let $H(X)$ be the entropy of a random variable $X$ and $I(X;Z)$  the mutual information between  $X$ and $Z$. By Fano's inequality, we have for any $ \hat{Y} $,
\begin{align}
\mathbb{P}_{(Y^*,A)} \left[ \hat{Y} \neq Y^\ast \right] \ge 1- \frac{I(Y^\ast;A)+ 1 }{\log |\mathcal{Y}| }. \label{EqFano}
\end{align}
We first bound $\log |\mathcal{Y}|$.
Simple counting gives that $|\mathcal{Y}|  = \binom{n}{n_1} \frac{n_1!}{r! (K!)^r }$, where $n_1=\triangleq rK$.
Note that $ \binom{n}{n_1} \ge (\frac{n}{n_1})^{n_1}$ and $ \sqrt{n} (\frac{n}{e})^{n}\le n! \le e \sqrt{n} (\frac{n}{e})^{n}$. It follows that
\begin{align*}
|\mathcal{Y}| \ge  \left(n/n_1\right)^{n_1} \frac{\sqrt{n_1} (n_1/e)^{n_1}}{ e \sqrt{r} (r/e)^r e^r K^{r/2} (K/e)^{n_1} } \ge \left(\frac{n}{K} \right)^{n_1} \frac{1}{e (r \sqrt{K})^r }.
\end{align*}
This implies  $\log |\mathcal{Y}|  \ge  \frac{1}{2} n_1 \log \frac{n}{K}$ under the assumption that $ 8 \le K\le n/2 $ and $n \ge 32$.

Next we upper bound $I(Y^\ast;A)$. Note that $H(A) \le \binom{n}{2} H(A_{12})$ because the $ A_{ij} $'s are identically distributed by symmetry. Furthermore, $ A_{ij} $'s are independent conditioned on $Y^\ast$, so  $H(A|Y^\ast) = \binom{n}{2} H(A_{12}|Y_{12}^\ast)$. It follows that $I(Y^\ast;A) =H(A)-H(A|Y^\ast) \le \binom{n}{2} I(Y^\ast_{12} ; A_{12})$. We bound $I(Y^\ast_{12};A_{12})$ below. Simple counting gives
\[
\mathbb{P}(Y^\ast_{12}=1) = \frac{\binom{n-2}{K-2} \binom{n-K}{K} \cdots \binom{n-rK+K}{K} \frac{1}{(r-1)!} }{|\mathcal{Y}|  } = \frac{n_1(K-1)}{n(n-1)}=\alpha,
\]
and thus $\mathbb{P}(A_{12}=1)=\beta:=\alpha p + (1-\alpha) q$. Therefore $ I(Y^\ast_{12} ; A_{12}) = \alpha D \left( p  \Vert \beta \right) + (1-\alpha) D\left(q \Vert \beta \right). $  Using the upper bound~\eqref{eq:boundDivergence} on the KL divergence and condition~\eqref{eq:precise}, we obtain
\begin{align*}
 I(Y^\ast_{12} ; A_{12}) = \alpha D(p \Vert \beta)+(1-\alpha) D(q \Vert \beta) \le \alpha D(p\Vert \beta ) + (1-\alpha) \frac{(q-\beta)^2}{\beta(1-\beta)} \le \frac{n_1}{4n^2}\log \frac{n}{K}.
\end{align*}
It follows that $ I(Y^*;A) \le \binom{n}{2}I(Y^*_{12}\Vert A_{12}) \le \frac{n_1}{8}\log \frac{n}{K} $.
Substituting into~(\ref{EqFano}) gives
\begin{align*}
\mathbb{P}_{(Y^*,A)} \left[ Y \neq Y^\ast \right]  \ge  1- \frac{ \frac{n_1}{4}\log \frac{n}{K}  +2 } {n_1 \log \frac{n}{K}} =  \frac{3}{4} - \frac{2 } {n_1 \log \frac{n}{K}} \ge \frac{1}{2}, 
\end{align*}
where the last inequality holds because $K \le n/2$ and $n_1 \ge 32$. This proves the sufficiency of~\eqref{eq:precise}.

We turn to the second part of the lemma. Notice that
\begin{align*}
K\cdot D(p\Vert \beta ) + \frac{n^2}{n_1}(1-\alpha)\frac{(q-\beta)^2}{\beta(1-\beta)}
& \overset{(a)}{\le} K  \frac{(p-\beta)^2 }{\beta(1-\beta)} + \frac{K}{\alpha}(1-\alpha) \frac{(q-\beta)^2}{\beta(1-\beta)} \\
&= K \frac{\alpha(1-\alpha) (p-q)^2}{\beta (1-\beta)} \overset{(b)}{\le}
 \frac{K (p-q)^2}{q(1-q)},
\end{align*}
where $(a)$ holds due to $\alpha \le \frac{n_1K}{n^2}$ and~\eqref{eq:boundDivergence};  $(b)$ holds because $\beta(1-\beta) \ge \alpha p(1-p) + (1-\alpha) q (1-q) \ge (1-\alpha)q(1-q)$ thanks to the concavity of $x(1-x)$. Combining the last displayed equation with~\eqref{eq:impossible}
implies~\eqref{eq:precise}.
\end{proof}

\begin{lemma} \label{lmm:impossible3}
Suppose  $128 \le K\le n/2$. We have $\inf_{\hat{Y}} \sup_{Y^*\in \mathcal{Y}} \mathbb{P} \left[\hat{Y} \neq Y^*\right] \ge \frac{1}{2}$ if
\begin{align}
K \max \left \{ D( p \Vert q ), D( q \Vert p ) \right \}  \le \frac{1}{24} \log (n -K ). \label{eq:impossible3}
\end{align}
\end{lemma}
\begin{proof}
Let $\bar{M}=n-K$, and $\bar{\mathcal{Y}}=\left\{ Y_{0},Y_{1},\ldots,Y_{\bar{M}}\right\} $
be a subset of $\mathcal{Y}$ with cardinality $\bar{M}+1$, which is specified
later. Let $\bar{\mathbb{P}}_{(Y^*,A)}$ denote the joint distribution of
$(Y^{\ast},A)$ when $Y^{\ast}$ is sampled from $\bar{\mathcal{Y}}$
uniformly at random and then $A$ is generated according to the planted clustering model based on $ Y^* $.
By Fano's inequality, we have
\begin{align}\label{EqTesting2}
 \sup_{Y^*\in \mathcal{Y}} \mathbb{P} \left[\hat{Y} \neq Y^*\right ]
 \ge   \bar{\mathbb{P}}_{(Y^*,A)}\left[\hat{Y} \neq Y^*\right]
 \ge 1- \frac{I(Y^\ast;A)+ 1 }{\log |\bar{\mathcal{Y}}| } .
\end{align}
We construct $\bar{\mathcal{Y}}$ as follows. Let $Y_{0}$ be the
cluster matrix such that the clusters $\{C_{l}\}_{l=1}^{r}$ are
given by $C_{l}=\left\{ (l-1)K+1,\ldots,lK\right\} $. Informally,
each $Y_{i}$ with $i\ge1$ is obtained from $Y_{0}$ by swapping
the cluster memberships of node $K$ and $K+i$. Formally, for each $i\in[\bar{M}]$: (1)
if node $(K+i)$ belongs to cluster $C_{l}$ for some $l$,
then $Y_{i}$ is the cluster matrix such that the first  cluster consists of nodes $\{1,2,\ldots,K-1,K+i\}$
and the $l$-th cluster is given by $C_{l}\setminus\{K+i\}\cup\{K\}$, and
all the other clusters identical to those according to $Y_0$; (2) if node $(K+i)$
is an isolated node in $Y_{0}$ (i.e., does not belong to any cluster),
then $Y_{i}$ is the cluster matrix such that the first
cluster consists of nodes $\{1,2,\ldots,K-1,K+i\}$ and node $K$ is an isolated
node, and all the other clusters identical to those according to $Y_{0}$.

Let $\mathbb{P}_{i}$ be the distribution of the graph $A$ conditioned
on $Y^{*}=Y_{i}$. Note that each $\mathbb{P}_{i}$ is the product of
$\frac{1}{2}n(n-1)$ Bernoulli distributions.
We have the following chain of inequalities:
\begin{align*}
I(Y^{*};A)
\overset{(a)}{\le}\frac{1}{( \bar{M} +1)^{2}}\sum_{i,i'=0}^{\bar{M}}D\left(\mathbb{P}_{i}\|\mathbb{P}_{i'}\right)
\overset{(b)}{\le}3K \cdot D\left( p \Vert q \right)+3K \cdot D\left( q \Vert p \right),
\end{align*}
where (a) follows from the convexity of KL divergence, and (b) follows
by our construction of $\left\{ Y_{i}\right\} $.
If~\eqref{eq:impossible3} in the lemma holds, then $I(Y^\ast;A)\le\frac{1}{4}\log(n-K)=\frac{1}{4}\log\left|\bar{\mathcal{Y}}\right|.$
Since $\log (n-K) \ge \log (n/2) \ge 4$ if $n \ge 128$, it follows from \eqref{EqTesting2} that the minimax error probability
is at least $1/2$.
\end{proof}

\begin{lemma} \label{lmm:impossible2}
Suppose $128 \le K\le n/2$. The  $\inf_{\hat{Y}} \sup_{Y^*\in \mathcal{Y}} \mathbb{P} \left[\hat{Y} \neq Y^*\right] \ge \frac{1}{4}$ if
\begin{align}
K p & \le  \frac{1}{8} \min\{\log (rK/2),K\},  \label{eq:impss_cond_K} \\
\underline{\textbf{or}}\quad K (1-q) & \le \frac{1}{4} \log K \label{eq:impss_cond_K2}.
\end{align}
\end{lemma}
\begin{proof}
First assume condition~\eqref{eq:impss_cond_K} holds. We call a node a \emph{disconnected node} if it is not connected to any other node in its own cluster. Let $E$ be the event that there exist two disconnected nodes from two different clusters. Suppose $Y^\ast$ is uniformly distributed over $\mathcal{Y}$ and let $\rho:=\mathbb{P}[E]$. We claim that $\mathbb{P} \left[\hat{Y} \neq Y^*\right] \ge \rho/2$. To see this, consider the maximum likelihood estimate of $ Y^* $ (MLE) given by $\hat{Y}_{\text{ML}}(a):=\arg \max_{y}  \mathbb{P} [ A=a |Y^\ast =y]$ with tie broken uniformly at random. It is a standard fact that the MLE minimizes the error probability under the uniform prior, so for all $ \hat{Y} $ we have
\begin{align}
\mathbb{P} \left[\hat{Y} \neq Y^*\right]
\ge \frac{1}{|\mathcal{Y}|}  \sum_{y\in\mathcal{Y}} \sum_{a \in \{0,1\}^{n \times n}} \mathbb{P} \left[ \hat{Y}_{\text{ML}}(a) \neq y \right ] \mathbb{P} \left[ A=a |Y^\ast =y \right]. \label{eq:MLEErrorbound}
\end{align}
Let $\mathcal{A}_y\subseteq{0,1}^{n\times n}$ denote the set of adjacency matrices with at least two disconnected nodes with respect to the clusters defined by $y \in \mathcal{Y}$. For each $a \in \mathcal{A}_y$, let $y'(a)$ denote the cluster matrix obtain by swapping the two rows and columns of $y$ corresponding to the two disconnected nodes in $a$. It is easy to check that for each $a \in \mathcal{A}_y$, the likelihood satisfies
 $ \mathbb{P} [ A=a| Y^\ast =y] \le \mathbb{P} [ A=a|Y^\ast =y'(a)]$ and therefore $\mathbb{P}[ \hat{Y}_{\text{ML}}(a) \neq y ] \ge 1/2$.
 It follows from~\eqref{eq:MLEErrorbound} that
 \begin{align*}
 \mathbb{P} \left[\hat{Y} \neq Y^*\right]  \ge  \frac{1}{|\mathcal{Y}|}  \sum_{y} \sum_{a\in\mathcal{A}_y} \frac{1}{2}\cdot\mathbb{P}[ A=a | Y^\ast =y ] = \frac{1}{2}  \rho,
 \end{align*}
 where the last equality holds because $\mathbb{P}[ \mathcal{A}_y | Y^\ast =y ]\equiv\rho:=\mathbb{P}[E]$ independently of $y$.

Since the maximum error probability is lower bounded by the average error probability, it suffices to show $\rho \ge 1/2$. Without loss of generality, suppose $r$ is even and the first $ rK/2 $ nodes $i \in [rK/2]$ form $r/2$ clusters. For each $i \in [rK/2]$, let $\xi_{i}$ be the indicator random variable for node $i$ being a disconnected node. Then $\rho_1:=\mathbb{P}\left[\sum_{i=1}^{rK/2}\xi_{i}\ge1\right]$ is the probability that there exits at least one disconnected node among the first $ rK/2 $ nodes. We use a second moment argument~\cite{durrett2007random} to lower-bound $\rho_1$. Observe that $\xi_1, \ldots, \xi_{rK/2}$ are (possibly dependent) Bernoulli variables with mean $\mu=(1-p)^{K-1}$. For $i \neq j$, we have
\begin{equation*}
\mathbb{E}\left[\xi_{i}\xi_{j}\right]=\mathbb{P}\left[\xi_{i}=1,\xi_{j}=1\right]
 \le \left(1-p\right)^{2K-3}=\frac{1}{1-p} \mu^2 .
\end{equation*}
Therefore, we have
\begin{align*}
 \text{Var} \left[ \sum_{i=1}^{rK/2}\xi_{i} \right]
 &\le \frac{1}{2}rK \mu (1-\mu) + \frac{1}{2}rK(rK/2-1) \left( \frac{1}{1-p} -1 \right) \mu^2 \\
 & \le \frac{1}{2} rK \mu + \frac{1}{4} r^2K^2\mu^2 \frac{p}{1-p}.
\end{align*}
By the assumptions~\eqref{eq:impss_cond_K} we have $p \le 1/8$ and
\begin{align}
\mu = (1-p)^{K-1} \overset{(a)}{\ge} e^{-2(K-1) p} \ge (rK/2)^{-1/4}, \label{eq:mubound}
\end{align}
where $(a)$ uses the inequality $1-x\ge e^{-2x}, \forall x \in [0,\frac{1}{2}]$.
Applying Chebyshev's inequality, we get
\begin{equation}
\mathbb{P}\left[\left|\sum_{i=1}^{rK/2}\xi_{i}-rK \mu/2 \right| \ge rK \mu /2 \right]\le\frac{ \frac{1}{2} rK \mu + \frac{1}{4} (rK \mu)^2 \frac{p}{1-p}}{ r^2K^2 \mu^2/4 } \le \frac{2}{ rK \mu} + \frac{p}{1-p}  \le \frac{1}{4}, \label{eq:chebyshev}
\end{equation}
where the last inequality holds due to~\eqref{eq:mubound} and $p \le 1/8$.
It follows that $ \rho_{1} \ge \frac{3}{4}$. If we let $\rho_2$ denote the probability that there exits a disconnected node among
the next $ rK/2 $ nodes $rK/2+1, \ldots, rK$, then by symmetry $\rho_2 \ge \frac{3}{4}$. Therefore $\rho=\rho_1\rho_2\ge 1/2$, proving the sufficiency of~\eqref{eq:impss_cond_K}.

We next assume the  condition~\eqref{eq:impss_cond_K2} holds and bound the error probability using a similar strategy.  For $ k=1,2 $, we call a node in cluster $ k $ a \emph{betrayed node} if it is connected to all nodes in cluster $3-k$. Let $E'$ be the event of having  a betrayed node in each of cluster $1$ and $ 2 $, and let $\mathbb{P}[E']:=\rho'$. Suppose $Y^\ast$ is uniformly distributed over $\mathcal{Y}$; again we can show that $\mathbb{P} [\hat{Y} \neq Y^*] \ge \rho'/2$ for any $ \hat{Y} $. Suppose cluster 1 is $[K] $. For each $i\in[K]$, let $\xi'_{i}$ be the indicator for node $ i $ being is a betrayed node. Then $\rho'_1:=\mathbb{P}\left[\sum_{i=1}^{K}\xi'_{i}>0\right]$ is the probability having a betrayed node in cluster $1$. We have
\begin{align*}
\mathbb{P}\left[\sum_{i=1}^{K}\xi'_{i}=0\right] = \left( 1- q^K \right)^K \le \exp \left( - K q^K \right) \overset{(a)}{\le} \exp( -K^{1/2} ) \le 1/4,
\end{align*}
where $(a)$ follows from~\eqref{eq:impss_cond_K2} and $q^K=\left( 1-(1-q) \right)^K \ge \exp \left(-2(1-q)K \right) $ since $1-q \le 1/2$. Let $\rho'_2$ be the probability of having a betrayed node in cluster $2$ and by symmetry $\rho'_2 \ge 3/4$. We thus get $\rho' =\rho'_1\rho'_2\ge 1/2$, proving the sufficiency of~\eqref{eq:impss_cond_K2}.
\end{proof}

We can now prove Theorem \ref{thm:Impossible} by combining the above three lemmas.

\begin{proof}[of Theorem \ref{thm:Impossible}]
Since $256\le 2K \le n$, we have the following  relations between the log terms:
\begin{align}
\log (n-K) \ge \log (n/2) \ge \frac{1}{2} \log n, \; \log (rK/2) \ge \frac{1}{2} \log (rK). \label{eq:upboundlogn}
\end{align}
Our goal is to show that  if the condition~\eqref{eq:KLimpossible1} or~\eqref{eq:KLimpossible2} holds, then the minimax error probability is large.

First assume~\eqref{eq:KLimpossible1} holds.
By~\eqref{eq:lowerboundDivergence} we know condition~\eqref{eq:KLimpossible1} implies
\begin{align}
K(p-q)^2 \le \frac{1}{96} p(1-q) \left( \log (rK) \wedge K \right). \label{eq:upperbound1}
\end{align}
$(i)$ If $ p \le 2q$, then~\eqref{eq:upperbound1} implies $K(p-q)^2 \le \frac{1}{48} q (1-q) \log (rK)$; it follows from~\eqref{eq:boundDivergence} and~\eqref{eq:upboundlogn} that $K D(p \Vert q) \le \frac{1}{48} \log (rK) \le \frac{1}{24} \log (n-K)$ and thus Lemma \ref{lmm:impossible3} proves the conclusion.
$(ii)$ If $p> 2q$,
\eqref{eq:upperbound1} implies  $Kp \le \frac{1}{24} \log (rK) \wedge K \le \min\{\frac{1}{24}K,\frac{1}{12} \log (\frac{rK}{2})\}$ and Lemma \ref{lmm:impossible2} proves the conclusion.

Next assume the condition~\eqref{eq:KLimpossible2} holds.
By the lower-bound~\eqref{eq:lowerboundDivergence} on the KL divergence, we know~\eqref{eq:KLimpossible2} implies
\begin{align}
K(p-q)^2 \le \frac{1}{96} p(1-q) \log n. \label{eq:upperbound2}
\end{align}
$(i)$ If $ 1-q \le 2 (1-p)$, then~\eqref{eq:upperbound2} implies that $K(p-q)^2 \le \frac{1}{48} p (1-p) \log n$; it follows from~\eqref{eq:boundDivergence} and~\eqref{eq:upboundlogn}  that  $K D(q\Vert p) \le \frac{1}{48} \log n \le \frac{1}{24} \log (n-K)$ and thus Lemma \ref{lmm:impossible3} implies the conclusion.
$(ii)$ If $1-q > 2(1-p)$ and $K \ge \log n$, then~\eqref{eq:upperbound2} implies
\begin{align}
K(1-q) \le \frac{1}{24} \log n \le \frac{1}{12} \max \left \{ \log \frac{n}{K}, \log K \right\}. \label{eq:upperbound3}
\end{align}
We divided the analysis into two subcases.

Case $(ii.1)$: $K \ge \log n$. It follows from~\eqref{eq:upperbound3} that $1-q \le \frac{1}{24}$, i.e., $q \ge \frac{23}{24}$ and thus $(p-q)^2 \le 2q (1-q)^2$. Therefore, \eqref{eq:upperbound3} implies either the condition~\eqref{eq:impossible} in Lemma \ref{lmm:impossible1}  or the condition~\eqref{eq:impss_cond_K2} in Lemma \ref{lmm:impossible2}, which proves the conclusion.

Case $(ii.2)$: $ K<\log n $.  It follows that $ \delta =\frac{n_1(K-1)}{n(n-1)}\le \frac{1}{10} $ and $\log \frac{n}{K} \ge \frac{1}{2} \log n$.
Note that $ \bar{p}=\delta p+(1-\delta) q \ge \max\{\delta p,q\}$ and $1-\bar{p} \ge \frac{9}{10}(1-q) $. Therefore, we have
 \begin{align}
 \frac{n^2(q-\bar{p})^2}{n_1\bar{p}(1-\bar{p})} = \frac{n^2\delta^2(p-q)^2}{n_1\bar{p}(1-\bar{p})} \le \frac{2n^2\delta(p-q)^2}{n_1p(1-q)}  \overset{(a)}{\le} 4K D(p\Vert q) \overset{(b)}{\le}  \frac{1}{24}\log \frac{n}{K}, \label{eq:smallKbound1}
 \end{align}
where we use~\eqref{eq:lowerboundDivergence} in $ (a) $ and~\eqref{eq:KLimpossible2} in $ (b) $. On the other hand, we have
\begin{align}
D(p\Vert \bar{p}) &= p\log\frac{p}{\bar{p}} + (1-p)\log\frac{1-p}{1-\bar{p}} \le p\log\frac{p}{q} + (1-p)\log\frac{10(1-p)}{9(1-q)} \nonumber \\
 &\le D(p\Vert q) + (1-q)\log \frac{10}{9} \le \frac{1}{6K}\log \frac{n}{K}, \label{eq:smallKbound2}
\end{align}
where the last inequality follows from~\eqref{eq:KLimpossible2} and~\eqref{eq:upperbound3}.
Equations~\eqref{eq:smallKbound1} and~\eqref{eq:smallKbound2}  imply assumption~\eqref{eq:precise} in Lemma~\ref{lmm:impossible1}, and therefore the conclusion follows.
\end{proof}

\subsubsection{Proof of Corollary~\ref{cor:impossible}}
The corollary is derived from Theorem~\ref{thm:Impossible} using the upper bound~\eqref{eq:boundDivergence} on the KL divergence. In particular, Condition~\eqref{eq:impossiblesimle_1} in the corollary implies Condition~\eqref{eq:KLimpossible2} in Theorem~\ref{thm:Impossible} in view of~\eqref{eq:boundDivergence}. Similarly, Condition~\eqref{eq:impossiblesimle_2} implies Condition~\eqref{eq:KLimpossible1} because $D(q \Vert p) \le \frac{p}{1-p}$ in view of~\eqref{eq:boundDivergence} and $p \le \frac{1}{193}$; Condition~\eqref{eq:impossiblesimle_3} implies Condition~\eqref{eq:KLimpossible2} because $D(p \Vert q) \le p \log \frac{p}{q}$ by definition.

\subsection{Proof of Theorem \ref{thm:MLE} and Corollary~\ref{cor:hard}}
Let $\langle X, Y \rangle := \text{Tr}(X^\top Y) $ denote the inner product between two matrices. For any feasible solution $ Y \in \mathcal{Y}$ of~(\ref{MLE}), we define $\Delta(Y):= \langle A, Y^\ast - Y \rangle$ and $ d(Y):=\langle Y^\ast, Y^\ast -Y \rangle  $. To prove the theorem, it suffices to show that $\Delta(Y)>0$ for all feasible $Y$ with $ Y\neq Y^* $. For simplicity, in this proof we use a different convention that $Y^\ast_{ii}=0$ and $Y_{ii}=0$ for all $i \in V$.  Note that $\mathbb{E}[A]= q J + (p-q) Y^\ast -q {I}$, where $ J $ is the $n \times n$ all-one matrix and ${I}$ is the $ n\times n $ identity matrix. We may decompose $ \Delta(Y) $ into an expectation term and a fluctuation term:
\begin{equation}\label{EqDeltaY}
\Delta(Y)=\langle \mathbb{E}[A], Y^\ast- Y \rangle + \langle A-\mathbb{E}[A], Y^\ast-Y \rangle = (p-q)d(Y) +\langle A-\mathbb{E}[A], Y^\ast-Y \rangle,
\end{equation}
where the second equality follows from $\sum_{i,j}Y_{ij}=\sum_{i,j}Y^\ast_{ij}$ by feasibility of $ Y $.
For the second fluctuation term above, observe that
\begin{align}
\langle A-\mathbb{E}[A],  Y^\ast - Y \rangle  =  2\underbrace{\sum_{(i<j):\substack{Y^\ast_{ij}=1\\ Y_{ij}=0} } (A_{ij}- p )}_{T_1(Y)} - 2\underbrace{\sum_{(i<j ):\substack{Y^\ast_{i j}=0\\ Y_{i j}=1}}  (A_{ij}-q)}_{T_2(Y)}. \nonumber
\end{align}
Here each of $T_1(Y)$ and $ T_2(Y) $ is the sum of $\frac{1}{2} d(Y)$ i.i.d.\ centered Bernoulli random variables with parameter $p$ and $ q $, respectively.
Using the Chernoff bound, we can bound the fluctuation for each fixed $ Y \in \mathcal{Y}$:
\begin{align*}
\mathbb{P} \left\{  T_1(Y) \le - \frac{ p-q }{4} d(Y) \right\} &  \le \exp \left( - \frac{1}{2} d(Y) D\left(\frac{p+q}{2} \Big| \Big| p \right) \right) \\
\mathbb{P} \left\{  T_2(Y) \ge  \frac{ p-q}{4} d(Y) \right\} & \le \exp \left( - \frac{1}{2} d(Y) D \left(\frac{p+q}{2} \Big| \Big| q \right) \right).
\end{align*}
We need to control the fluctuation uniformly over $ Y \in \mathcal{Y}$. Define the equivalence class $[Y]= \{ Y' \in \mathcal{Y}: Y'_{ij}=Y_{ij}, \forall (i,j) \text{ s.t. } Y^\ast_{ij}=1\}$. Notice that all cluster matrices in the equivalence class $[Y]$ have the same value $T_1(Y)$. The following combinatorial lemma upper bounds the number of $Y$'s and $[Y]$'s such that $d(Y)=t$. Note that $2(K-1) \le d(Y) \le r K^2$ for any feasible $ Y\neq Y^* $.
\begin{lemma} \label{lem:countY}
For each integer $t \in[K,rK^2]$, we have
\begin{align*}
|\{Y\in\mathcal{Y}:d(Y)=t\}| &\le \left( \frac{16 t^2}{K^2} \right)^2 n^{32t/K}, 
\\
| \{[Y]: d(Y)=t \} | & \le \frac{16 t^2}{K^2}(r K)^{16t/K} . 
\end{align*}
\end{lemma}
We also need the following lemma to upper bound $D \left(\frac{p+q}{2} \Vert q \right)$ and $D\left(\frac{p+q}{2} \Vert p \right)$
 using $D\left(p \Vert q \right)$
and $D\left(q \Vert p \right)$, respectively.
\begin{lemma} \label{lmm:kldivergence}
\begin{align}
 D \left(\frac{p+q}{2} \Big| \Big| q \right) \ge \frac{1}{36} D\left(p \Vert q \right)  \label{eq:bounddivergence1}\\
D\left(\frac{p+q}{2} \Big| \Big| p \right) \ge \frac{1}{36} D\left(q \Vert p \right)  \label{eq:bounddivergence2}
 \end{align}
\end{lemma}
We prove the lemmas in the appendix. Using the union bound and Lemma \ref{lem:countY} and Lemma \ref{lmm:kldivergence}, we obtain
\begin{align*}
   &  \mathbb{P} \left\{ \exists [Y]: Y \neq Y^*, T_1(Y) \le - \frac{ p-q }{4} d(Y) \right\} \\
 \le & \sum_{t=K}^{rK^2} \mathbb{P} \left\{ \exists [Y] : d(Y)=t, T_1(Y) \le - \frac{ p-q }{4} t  \right\}  \\
\le &   \sum_{t=K}^{rK^2} | \{\exists [Y]: d(Y)=t  \} | \mathbb{P} \left\{ T_1(Y) \le - \frac{ p-q }{4} t  \right\}  \\
\le &  \sum_{t=K}^{rK^2} \frac{16t^2}{K^2} (rK)^{16t/K} \exp \left(  - \frac{1}{72}t D(q \Vert p ) \right) \\
\overset{(a)}{\le} &  16 \sum_{t=K}^{rK^2}  (rK)^2 (\gamma rK)^{-5t/K} \le 16(\gamma rK)^{-1},
\end{align*}
where (a) follows from the theorem assumption that $ D(q \Vert p)  \ge c_1 \log (\gamma rK)/K$ for a large constant $c_1$. Similarly,
\begin{align*}
    & \mathbb{P} \left\{ \exists Y \in \mathcal{Y}: Y \neq Y^*, T_2(Y) \ge  \frac{ p-q}{4} d(Y) \right\} \\
\le & \sum_{t=K}^{rK^2} \mathbb{P} \left\{  \exists Y \in \mathcal{Y} : d(Y)=t, T_2(Y) \ge  \frac{ p-q}{4} t  \right\}  \\
\le &   \sum_{t=K}^{rK^2} \left| \{ Y \in \mathcal{Y}: d(Y)=t  \} \right| \cdot \mathbb{P} \left\{ T_2(Y) \ge  \frac{ p-q}{4} t  \right\}  \\
\le &   \sum_{t=K}^{rK^2} \frac{256t^4}{K^4} n^{32t/K} \cdot \exp \left(  - \frac{1}{72}t D(p \Vert q ) \right)
\overset{(a)}{\le}   256n^{-1},
\end{align*}
where (a) follows from the theorem assumption that $D(p \Vert q ) \ge c_1 \log n/K$ for a large constant $c_1$.
Combining the above two bounds with~\eqref{EqDeltaY}, we obtain
\begin{align}\label{EqProbFixY}
\mathbb{P} \left\{  \exists Y \in \mathcal{Y}:  \Delta(Y) \le 0 \right\} \le  16 (\gamma rK)^{-1}+ 256 n^{-1}
\end{align}
and thus $ Y^*$ is the unique optimal solution with high probability. This proves the theorem.

\subsubsection{Proof of Corollary~\ref{cor:hard}}
The corollary is derived from Theorem~\ref{thm:MLE} using the lower bound~\eqref{eq:lowerboundDivergence} on the KL divergence.
In particular, first assume $e^2 q \ge p$. Then $K(p-q)^2 \gtrsim q (1-q) \log n$ implies condition~\eqref{eq:ConditionHard} in view of~\eqref{eq:lowerboundDivergence}. Next assume $e^2 q <p$. It follows that $\log \frac{p}{q} \le 2 \log \frac{p}{eq}$. By definition,
$D(p \Vert q ) \ge p \log \frac{p}{q} + (1-p) \log (1-p) \ge p \log \frac{p}{eq}$. Hence, $K p \log \frac{p}{q} \gtrsim \log n$ implies $K D(p \Vert q) \gtrsim \log n$. Furthermore, $D(q \Vert p) \ge \frac{1}{2} (1- 1/e^2) p$ in view of~\eqref{eq:lowerboundDivergence} and $p>e^2 q$. Therefore, $K p \gtrsim \log (rK) $ implies $K D( q \Vert p) \gtrsim \log (rK)$.

\subsection{Proof of Theorem \ref{thm:CVX}}\label{sec:proof_CVX}
We prove Theorem \ref{thm:CVX} and Theorem~\ref{thm:CVX_bi} (for submatrix localization) together in this section. Our proof only relies on two standard concentration results for the adjacency matrix $ A $ (see Proposition~\ref{lem:basic} below).
We need some unified notations. For both models we use $ n_L $ and $ n_R $ to denote the problem dimensions, with the understanding that $ n_L=n_R=\max\{n_L,n_R\}=n $ for planted clustering. Similarly, for planted clustering the left and right clusters are identical and $ K_L= K_R = K$. Let $U\in\mathbb{R}^{n_L\times r}$ and $V\in\mathbb{R}^{n_R\times r}$ be the normalized characteristic matrices
of the left and right clusters, respectively; i.e.,
$$U_{ik} = \begin{cases}
\frac{1}{\sqrt{K_L}} & \text{if the left node $i$ is in the $ k $-th left cluster} \\
0 & \text{otherwise,}
\end{cases}$$
and similarly for $ V $. Here $ U=V $ for planted clustering.
The true cluster matrix $ Y^* $ has the rank-$ r $ singular value
decomposition given by $Y^{\ast}=\sqrt{K_{L}K_{R}}UV^{\top}$. Define the projections $\mathcal{P}_T(M) = UU^\top M+MVV^\top-UU^\top MVV^\top $ and $\mathcal{P}_{T^\perp}(M) = M - \mathcal{P}_T(M)$.
Several matrix norms are used: the spectral norm $\|X\|$ (the largest singular value of $ X $), the nuclear norm  $\|X\|_\ast$ (the sum of the singular values), the $ \ell_1 $ norm $\|X\|_1=\sum_{i,j}|X_{ij}|$ and the  $\ell_\infty$ norm $\|X\|_\infty=\max_{i,j} |X_{ij}|$.

We define a quantity $ \nu>0 $ and a matrix $ \bar{A} \in \mathbb{R}^{n_L \times n_R} $, which roughly correspond to the signal strength and the population version of $ A $. For planted clustering, let $ \nu :=  p-q $ and $ \bar{A} := q J+ (p-q)Y^* $, where $J $ is the all-ones matrix. For submatrix localization, let  $ \nu := \mu $ and $ \bar{A} := \mu Y^* $. The proof hinges on the following concentration property of the random matrix $ A-\bar{A} $.

\begin{proposition}\label{lem:basic}
Under the condition~\eqref{eq:Condition} for planted clustering or the condition~\eqref{eq:CVX_cond_bi} for submatrix localization, the following holds  with probability at least $ 1-n^{-10} $:
\begin{align}
\Vert A- \bar{A}\Vert & \le \frac{1}{8}\nu \sqrt{K_L K_R}, \label{eq:basic1}\\
\Vert \mathcal{P}_T(A-\bar{A}) \Vert_\infty & \le \frac{1}{8}\nu. \label{eq:basic2}
\end{align}
\end{proposition}
We prove the proposition in Section~\ref{sec:proof_basic} to follow. In the rest of the proof we assume~\eqref{eq:basic1} and~\eqref{eq:basic2} hold. To establish the theorems, it suffices to show that $\langle Y^\ast - Y, A \rangle>0$ for all feasible solution $ Y $ of the convex program with $Y \neq Y^\ast$. For any feasible $ Y $, we may write
\begin{align}
    \langle Y^\ast - Y, A \rangle
    &= \langle \bar{A}, Y^\ast -Y \rangle + \langle A-\bar{A}, Y^\ast-Y \rangle\nonumber\\
    &= \nu \langle Y^\ast, Y^\ast -Y \rangle + \langle A-\bar{A}, Y^\ast-Y \rangle= \frac{\nu}{2} \|Y^\ast - Y\|_1 + \langle A-\bar{A}, Y^\ast-Y \rangle, \label{eq:boundmean}
  \end{align}
where the first equality follows from the definition of $ \bar{A} $, and the second equality holds because $ Y $ obeys the linear constraints of the convex programs $\sum_{i,j}Y_{ij}=\sum_{i,j}Y^*_{ij}$ and $Y_{ij} \in [0,1],\forall i,j$.
On the other hand, we have $\|Y^\ast\|_\ast \ge \Vert Y\Vert_*$ thanks to the constraint~\eqref{eq:norm} or~\eqref{eq:nuclear_bi}. Let $ W := \frac{8(A-\bar{A}) }{\nu\sqrt{K_LK_R}} $. By~\eqref{eq:basic1} we have  $ \Vert \mathcal{P}_{T^\perp}(W)\Vert \le \Vert W \Vert \le 1 $, so  $UV^\top+\mathcal{P}_{T^\perp}(W)$ is a subgradient of $f(X):=||X||_*$ at $X = Y^\ast$. It follows that
 \begin{align*}
 0
 \ge \|Y\|_\ast\!-\!\|Y^\ast \|_\ast \ge \left\langle UV^\top \!+\mathcal{P}_{T^\perp}(W), Y\!- Y^\ast \right\rangle
 = \langle W, Y\!- Y^\ast \rangle + \left\langle UV^\top\!-\mathcal{P}_{T}(W), Y\!- Y^\ast \right\rangle.
 \end{align*}
Rearranging terms and using the definition of $ W $ gives
 \begin{align}
 \left\langle A-\bar{A}, Y^\ast\!-Y \right\rangle = \frac{\nu\sqrt{K_LK_R}}{8} \left\langle W, Y^\ast\!-Y \right\rangle  \ge  \frac{\nu\sqrt{K_LK_R}}{8} \left\langle - UV^\top\!+\mathcal{P}_{T}(W), Y^\ast \!- Y  \right\rangle. \label{eq:boundinnerproduct}
 \end{align}
Assembling~\eqref{eq:boundmean} and~\eqref{eq:boundinnerproduct}, we obtain that for any feasible $ Y $,
 \begin{align}
  \langle Y^\ast - Y, A \rangle &\ge  \frac{\nu}{2} \|Y^\ast - Y\|_1 + \frac{\nu\sqrt{K_LK_R}}{8}\left\langle  - UV^\top + \mathcal{P}_{T} (W), Y^\ast-Y\right\rangle \nonumber \\
  & \ge  \left( \frac{\nu}{2} - \frac{\nu\sqrt{K_LK_R}}{8}  \|UV^\top\|_\infty -  \|\mathcal{P}_{T} (A-\bar{A})\|_{\infty}  \right)   \|Y^\ast - Y\|_1, \nonumber
  \end{align}
where the last inequality follows from duality between the $ \ell_1 $ and $ \ell_\infty $ norms. Using~\eqref{eq:basic2} and the fact that $\|UV^\top \|_\infty= 1/\sqrt{K_LK_R}$, we get
\begin{align}
  \langle Y^\ast - Y, A \rangle \ge
  \left(  \frac{\nu}{2 } - \frac{\nu }{8} - \frac{\nu }{8}  \right)  \|Y^\ast - Y\|_1
  =   \frac{\nu}{4 }  \|Y^\ast - Y\|_1, \label{eq:DeltaYLowerBound}
\end{align}
which is positive for all $ Y\neq Y^* $. This completes the proof of Theorems~\ref{thm:CVX} and~\ref{thm:CVX_bi}.

\subsubsection{Proof of Proposition~\ref{lem:basic}}\label{sec:proof_basic}

We first prove~\eqref{eq:basic2}. By definition of $\mathcal{P}_T$, we have
\begin{align}
\Vert\mathcal{P}_{T}(A-\bar{A})\Vert_{\infty}
\leq &\Vert UU^{\top}(A-\bar{A})\Vert _{\infty}+\Vert (A-\bar{A})VV^{\top}\Vert _{\infty}+\Vert UU^{\top}(A-\bar{A})VV^{\top}\Vert _{\infty}\nonumber\\
\leq & 3\max\left(\Vert UU^{\top}(A-\bar{A})\Vert _{\infty},\Vert (A-\bar{A})VV^{\top}\Vert _{\infty}\right).\label{eq:triangle}
\end{align}
Suppose the left node $i$ belongs to the left cluster $k$. Then
\begin{equation}\label{eq:UU}
    (U U^\top (A-\bar{A}))_{ij}
    = \frac{1}{K_L} {\sum_{l\in C^\ast_k}} (A-\bar{A})_{lj}
    = \frac{1}{K_L} {\sum_{l\in C^\ast_k}} (A-\mathbb{E}A)_{lj} + \frac{1}{K_L} {\sum_{l\in C^\ast_k}} (\mathbb{E}A-\bar{A})_{lj}.
\end{equation}
To proceed, we consider the two models separately.

\emph{Planted clustering:} The entries of the matrix $ A-\mathbb{E}A $ are centered Bernoulli random variables with variance bounded by $ p(1-q)  $ and mutually independent up to symmetry with respect to the diagonal. The first term of~\eqref{eq:UU} is the average of $K_L$ such random variables; by Bernstein's inequality (Theorem~\ref{thm:Bernstein}) we have with probability at least $1-n^{-13} $ and for some universal constant $ c_2 $,
\begin{align*}
    \left|\textstyle{\sum_{l\in C^\ast_k}} (A-\mathbb{E}A)_{l j} \right| \leq \sqrt{26 p(1-q) K \log n}+9\log n \le c_2  \sqrt{p(1-q)K \log n},
\end{align*}
where the last inequality follows because $K p(1-q) > c_1 \log n$ in view of the condition~\eqref{eq:Condition}. By definition of $ \bar{A} $, $ \mathbb{E}[A] -\bar{A} $ is a diagonal matrix with diagonal entries equal to $ -p $ or $ -q $, so the second term of~\eqref{eq:UU} has magnitude at most $ 1/K $. By the union bound over all $ (i,j) $ and substituting back to~\eqref{eq:triangle}, we have with probability at least $1-2n^{-11}$,
\begin{align*}
\Vert\mathcal{P}_T(A-\bar{A})\Vert_\infty \le  3c_2 \sqrt{p(1-q)\log n/K} + 3/K \le (p-q)/8 = \nu/8
\end{align*}
where the last inequality follows from the condition~\eqref{eq:Condition}. This proves~\eqref{eq:basic2} in the proposition.

\emph{Submatrix localization:} We have $ \bar{A} = \mathbb{E}A $ by definition, so the second term of~\eqref{eq:UU} is zero. The first term is the average of $K_{L}$ independent centered random variables with unit sub-Gaussian norm. By a standard sub-Gaussian concentration inequality (e.g., Proposition 5.10 in~\cite{vershynin2010nonasym}), we have  for some universal constant $ c_3 $ and with probability at least $1-n^{-13}$,
\[
\left|\textstyle{\sum_{l\in\mathcal{C}_{k}}}  (A-\mathbb{E}A)_{lj}\right|\le c_{3}\sqrt{K_{L}\log n}.
\]
So $\Vert UU^{\top}(A-\mathbb{E}A)\Vert _{\infty}\leq c_{3}\sqrt{\log n/K_{L}}$
with probability at least $1-n^{-11}$ by the union bound. Similarly, $\Vert (A-\mathbb{E}A)VV^{\top}\Vert _{\infty}\leq c_{2}\sqrt{\log n/K_{R}}$ with the same probability. Combining with~\eqref{eq:triangle}~gives
\begin{align*}
\Vert\mathcal{P}_T(A-\bar{A})\Vert_\infty \le \sqrt{\log n/\min\{K_L,K_R\}} \le \nu/8 =\mu/8,
\end{align*}
where the last inequality holds under the condition~\eqref{eq:CVX_cond_bi}. This proves~\eqref{eq:basic2} in the proposition.

We now turn to the proof of ~\eqref{eq:basic1} in the proposition, separately for the two models.
\begin{itemize}
\item \emph{Planted clustering:}
Note that $ \| A- \bar{A} \| \le \|A-\mathbb{E}[A] \| + \|\bar{A}-\mathbb{E} [A] \| \le \|A-\mathbb{E}[A] \| +1$. Under the condition~\eqref{eq:Condition}, $K p(1-q) \ge c_1 \log n$. The spectral norm term is bounded below.
\begin{lemma}
If $Kp(1-q) \ge c_1 \log n $, then there exists  some universal constant $c_4$ such that $\| A- \mathbb{E}[A] \| \le c_4 \sqrt{  p(1-q) K \log n + q(1-q) n }$  with probability at least $1-n^{-10}.$
\label{thm:AConcentration}
\end{lemma}
We prove the lemma in Section~\ref{sec:AConcentration} to follow. Applying the lemma, we obtain
\begin{align*}
 \| A- \bar{A} \|
 \le c_4 \sqrt{ p(1-q) K \log n + q(1-q) n } + 1
 \le \frac{K(p-q)}{8} = \frac{K\nu}{8},
\end{align*}
where the second inequality holds under the condition~\eqref{eq:Condition}.

\item \emph{Submatrix localization:} The matrix $ A-\bar{A} = A-\mathbb{E}
A $ has i.i.d.\ sub-Gaussian entries. Using a standard concentration bound on the spectral norm of such matrices  (Theorem~5.39 in~\cite{vershynin2010nonasym}), we get that for a universal constant $ c_5 $ and with probability at least $1-n^{-10},$
\[
\|A-\mathbb{E}A \|\le c_5 \sqrt{n} \le \frac{\mu}{8} \sqrt{K_L K_R}= \frac{\nu}{8}\sqrt{K_L K_R},
\]
where the second inequality holds  under the condition~\eqref{eq:CVX_cond_bi}.
\end{itemize}

 \subsubsection{Proof of Lemma~\ref{thm:AConcentration}}\label{sec:AConcentration}

Let $R:=\text{support}(Y^\ast)$ and $\mathcal{P}_{R}(\cdot): \mathbb{R}^{ n \times n} \to \mathbb{R}^{n \times n}$ be the operator which sets the entries outside of $R$ to zero. Let $B_1=\mathcal{P}_R (A-\mathbb{E}[A] )$ and $B_2 = A-\mathbb{E}[A]-B_1$. Then $B_1$ is a block-diagonal symmetric matrix with $r$ blocks of size $K \times K$ and its upper-triangular entries are independent with zero mean and  variance bounded by $p(1-q)$.  Applying the matrix Bernstein inequality~\cite{tropp2010matrixmtg} and using the assumption that $K p(1-q)  \ge c_1 \log n$ in the lemma, we get that there exits some universal constant $c_6$ such that $\|B_1\| \le c_6 \sqrt{ p(1-q) K \log n}$ with probability at least $1-n^{-11}.$

On the other hand, $B_2$ is symmetric and its upper-triangular entries are independent centered Bernoulli random variables with variance bounded by
$\sigma^2: = \max \{ q(1-q), c_7\log n /n \}$ for any universal constant $ c_7 $.  If $\sigma^2 \ge \frac{\log^7 n}{n}$, then Theorem 8.4 in~\cite{Chattergee12}
implies that $\| B_2\| \le 3 \sigma \sqrt{n }$ with probability at least $1-n^{-11}$.
If $ c_7  \frac{\log n}{n}  \le \sigma^2 \le  \frac{\log^7 n}{n}$ for a sufficiently large constant $ c_7$, then Lemma 2 in~\cite{mastom11} implies that  $\| B_2 \|  \le c_8 \sigma \sqrt{n }$ with probability at least $1-n^{-11}$ for some universal constant $ c_8 \ge 3 $. (See Lemma 8 in~\cite{vu2014clustering} for a similar derivation.)
It follows that with probability at least $1-2n^{-11}$,
\begin{align}
  \|A-\mathbb{E}[A] \| \le \|B_1\|+ \|B_2\| & \le c_6 \sqrt{ p(1-q) K \log n} + c_8 \max \{ \sqrt{ q (1-q) n }, \sqrt{\log n}  \} \nonumber \\
    & \le c_4 \sqrt{  p(1-q)K \log n + q(1-q) n }, \nonumber
\end{align}
where the last inequality holds because $K p(1-q)  \ge c_1 \log n$ by assumption. This proves the~lemma.

\subsection{Proof of Theorem \ref{thm:cvx_converse}}
Observe that if any feasible solution $ Y $ has the same support as $ Y^* $, then the constraint~\eqref{eq:linear} implies that $ Y $ must be exactly equal to $ Y^* $. Therefore, it suffices to show that $ Y^* $ is not an optimal solution.

We first claim that $K (p-q)  \le c_2  \sqrt{ Kp  + q n}$ implies $K(p-q) \le c_2 \sqrt{2qn}$ under the assumption that $K \le n/2$ and $qn \ge c_1 \log n$.
In fact, if $Kp \le q n$, then the claim trivially holds. If $Kp > qn $, then $q < Kp/n \le p/2$. It follows that
\[ Kp/2 < K (p-q) \le c_2 \sqrt{ Kp + qn } \le c_2 \sqrt{ 2K p } .\]
Thus, $Kp < 8 c_2^2$ which contradicts the assumption that $Kp > qn \ge c_1 \log n$. Therefore, $Kp >qn$ cannot hold.
Hence, it suffices to show that if $K(p-q) \le c_2 \sqrt{2qn}$,
then $Y^\ast$ is not an optimal solution. We do this by deriving a contradiction assuming the optimality of $ Y^* $.

Let $J$ be the $n\times n$ all-ones matrix. Let $\R:=\textrm{support}(Y^{*})$ and $\A:=\textrm{support}(A)$.
Recall the cluster characteristic matrix $ U $ and the  projection  $\mathcal{P}_T(M) = UU^\top M+MUU^\top-UU^\top MUU^\top$ defined in Section~\ref{sec:proof_CVX}, and that $Y^\ast=K U U^\top$ is the SVD of $Y^\ast$.
Consider the Lagrangian
\[
L(Y;\lambda,\mu,F,G):=- \left\langle A,Y\right\rangle + \lambda \left(\left\Vert Y\right\Vert _{*}-\left\Vert Y^{*}\right\Vert _{*}\right) + \eta \left(\left\langle J,Y\right\rangle -rK^{2}\right)- \left\langle F,Y\right\rangle + \left\langle G,Y-J\right\rangle,
\]
where the Lagrangian multipliers are $\lambda,\eta\in\mathbb{R}$ and $F,G\in\mathbb{R}^{n\times n}$.
Since  $Y=\frac{rK^2}{n^2} J$ is strictly feasible, strong
duality holds by Slater's condition. Therefore, if $Y^{*}$ is an optimal
solution, then there must exist some $ F,G \in \mathbb{R}^{n\times n}$ and $ \lambda $ for which the KKT conditions hold:
\begin{align*}
\left.
\begin{aligned} 0 &  \in \frac{\partial L(Y;\lambda,\mu,F,G)}{\partial Y} \bigg\vert_{Y=Y^{*}}
\end{aligned}
\quad \right\}  & \text{Stationary condition}\\
\left.
\begin{aligned}F_{ij}\ge0, G_{ij}\ge 0,\; & \forall(i,j)\\
\lambda\ge0
\end{aligned}
\quad\right\}  & \text{Dual feasibility}\\
\left.\begin{aligned}F_{ij}=0, \; & \forall(i,j)\in\R\\
G_{ij}=0, \; & \forall(i,j)\in\R^{c}
\end{aligned}
\quad\right\}  & \text{Complementary slackness}.
\end{align*}
Recall that $M\in \mathbb{R}^{n\times n}$ is a sub-gradient of $\Vert X\Vert_*$ at $X = Y^\ast$ if and only if $\mathcal{P}_T(M) = UU^\top $ and $\Vert M-\mathcal{P}_{T}(M)\Vert\leq 1$. Let $H=F-G$; the KKT conditions imply that there exist some numbers $\lambda\ge 0$, $\eta\in\mathbb{R}$ and matrices
$W$, $H$ obeying
\begin{align}
&A- \lambda \left(UU^{\top}+W\right)-\eta J+H  =0;\label{eq:subgrad}\\
&\mathcal{P}_{T}W  =0;\qquad
\left\Vert W\right\Vert  \le1;\label{eq:W}\\
&H_{ij}  \le0, \; \forall(i,j)\in\R; \qquad
H_{ij}  \ge0, \; \forall(i,j)\in\R^{c}.\label{eq:sign}
\end{align}

Now observe that $UU^{\top}WUU^{\top}=0$ by~\eqref{eq:W}. We
left  and right multiply~\eqref{eq:subgrad} by $UU^{\top}$ to obtain
\[
\breve{A} -\lambda UU^{\top} -\eta J+\breve{H} =0, \label{eq:subgrad2}
\]
where for any $X\in\mathbb{R}^{n\times n}$, $\breve{X}:=UU^\top XUU^\top$ is
the matrix obtained by averaging each $K\times K$ block
of $X$.  Consider the last display equation on the entries in $\R$ and $\R^{c}$
respectively. By the Bernstein inequality (Theorem~\ref{thm:Bernstein}) for each entry $\breve{A}_{ij}$, we have with probability at least $1-2n^{-11}$,
\begin{align}
 p -\frac{\lambda}{ K}-\eta  + \breve{H}_{ij} &\ge - \frac{c_{3}\sqrt{p(1-p) \log n }}{K} - \frac{c_4\log n}{2K^2} \overset{(a)}{\ge} -\frac{\epsilon_0}{8} , \quad \forall (i,j) \in \R \label{eq:equal1}\\
q - \eta + \breve{H}_{ij} &\le \frac{c_{3}\sqrt{q(1-q) \log n }}{K}+\frac{c_4\log n}{2K^2} \overset{(b)}{\le} \frac{\epsilon_0}{8}, \quad \forall (i,j) \in \R^c \label{eq:equal2}
\end{align}
for some universal constants $c_3,c_4>0$, where $(a)$ and $ (b) $ follow from the assumption $K \ge c_1 \log n$ with the universal constant $ c_1 $ sufficiently large. In the rest of the proof, we assume~\eqref{eq:equal1} and~\eqref{eq:equal2} hold.
Using~\eqref{eq:sign}, we get that
\begin{equation}\label{eq:etabound}
\begin{aligned}
\eta & \ge   q - \frac{c_3 \sqrt{q(1\!-\!q) \log n}}{K}- \frac{c_4\log n}{2K^2} \ge q - \frac{\epsilon_0}{8}\\
\eta &\le p + \frac{c_3 \sqrt{p(1\!-\!p) \log n }}{K} + \frac{c_4\log n}{2K^2} -\frac{\lambda}{ K} \le p+\frac{\epsilon_0}{8} -\frac{\lambda}{ K}.
\end{aligned}
\end{equation}
It follows that
\begin{align}
\lambda &\le K(p-q) + c_3 (\sqrt{p(1-p)\log n} + \sqrt{q(1-q) \log n} ) +  \frac{c_4\log n}{K}  \nonumber\\
& \le 4 \max \left\{ K(p-q), c_3 \sqrt{p(1-p) \log n}, c_3 \sqrt{q(1-q)\log n},  \frac{c_4}{c_1} \right\}. \label{eq:larger}
\end{align}
On the other hand, \eqref{eq:W} and \eqref{eq:subgrad} imply
\begin{align*}
\lambda^2 & = \left\Vert \lambda ( UU^{\top} +W ) \right\Vert^2 \ge \frac{1}{n} \left\Vert  \lambda ( UU^{\top}+W) \right\Vert _{F}^2 \nonumber \\
 & =\frac{1}{n} \left\Vert A-\eta J+ H\right\Vert _{F}^2 \nonumber
 \ge \frac{1}{n}\left\Vert A_{\R^{c}}-\eta J_{\R^c}+H_{\R^{c}}\right\Vert_{F}^2 \nonumber
 \ge \frac{1}{n}\sum_{(i,j) \in \R^c} \left(1-\eta \right)^{2} A_{ij},\nonumber
\end{align*}
where $ X_{\R^c} $ denotes that matrix obtained from $ X $ by setting the entries outside $ \R^c $ to zero. Using~\eqref{eq:etabound}, $ \lambda\ge0 $ and the assumption $p \le 1-\epsilon_0$, we obtain $\eta \le 1- \frac{7}{8}\epsilon_0$ and therefore
\begin{equation}
\lambda^2   \ge \frac{49}{64n}\epsilon_0^2 \sum_{ (i,j) \in \R^c } A_{ij}. \label{eq:fro}
\end{equation}
Note that $\sum_{(i,j) \in \R^c} A_{ij}$ equals two times the sum of $\binom{n}{2}-r\binom{K}{2}$ i.i.d.\ Bernoulli random variables with parameter $q$. By the Chernoff bound of Binomial distributions and the assumption that $qn \ge c_1 \log n$, with probability at least $1-n^{-11}$,
$\sum_{(i<j) \in \R^c } A_{ij} \ge c_5 q n^2$ for some universal constant $c_5 $. It follows from~\eqref{eq:fro} that $ \lambda^2 \ge \frac{1}{2}\epsilon_0^2 c_5 q n$.
Combining with~\eqref{eq:larger} and the assumption that $ qn \ge c_1 \log n  $, we conclude that with probability at least $1-3n^{-11}$, $K^2(p-q)^2 \ge \frac{1}{32} \epsilon^2 c_5 qn$. Choosing $c_2$ in the assumption sufficiently small such that $2c_2^2 < \frac{1}{32} \epsilon^2 c_5$, we have $K(p-q) > c_2 \sqrt{2qn}$, which leads to the contradiction. This completes the proof of the theorem.

\subsection{Proof of Theorem \ref{thm:Simple}}\label{sec:proofSimple}
Let $\text{Bin}(n, p)$ denote the binomial distribution with $n$ trials and success probability $p$. For each non-isolated node $i \in V_1$,
its degree $d_i$ is the sum of two independent binomial random variables  distributed as $\text{Bin}(K-1, p)$ and $\text{Bin} (n-K,q)$ respectively. For each isolated node $i \in V_2$, its degree $d_i$ is distributed as $\text{Bin} (n-1,q)$.
It follows that $\mathbb{E}[d_i]=(n-1)q+(K-1)(p-q)$ if $i \in V_1$ and $\mathbb{E}[d_i] = (n-1)q$ if $ i\in V_2 $. Define $\sigma_1^2:= Kp(1-q)+ n q(1-q)$, then $\text{Var}[d_i] \le \sigma_1^2$ for all $ i $. Set $t_1:=\frac{1}{2}(K-1)|p-q|\le \sigma_1^2$; the Bernstein inequality (Theorem~\ref{thm:Bernstein}) gives
\begin{align*}
\mathbb{P}\left\{ |d_i - \mathbb{E}[d_i] | \ge t_1\right\} \le 2\exp \left( -\frac{t_1^2}{2 \sigma_1^2 + 2 t_1/3 } \right) \le  2 \exp \left(-\frac{(K-1)^2 (p-q)^2 }{12 \sigma_1^2 }\right) \le 2 n^{-2},
\end{align*}
where the last inequality follows from the assumption (\ref{eq:ConditionSimple}). By the union bound, with probability at least $1-2n^{-1}$,  we have $d_i>  \frac{(p-q)K}{2} + qn$ for all nodes $i \in V_1$ and $ d_i < \frac{(p-q)K}{2} + qn$ for all nodes $i \in V_2$. On this event, all nodes in $V_2$ are correctly declared to be isolated.

For two nodes $i$ and $j$ in the same cluster, the number of their common neighbors $S_{ij}$ is the sum of two independent binomial random variables distributed as $\text{Bin}(K-2, p^2)$ and  $\text{Bin} (n-K,q^2)$, respectively. Similarly, for two nodes $i, j$ in two different clusters, $S_{ij}$ is the sum of two independent binomial variables $\text{Bin}(2(K-1),pq)$ and $\text{Bin}(n-2K,q^2)$.
Hence, $\mathbb{E}[S_{ij}]$ equals $(K-2)p^2+(n-K)q^2$ if $i$ and $j$ are in the same cluster and $2(K-1)pq+(n-2K)q^2$ otherwise.
The difference of the expectations equals $K(p-q)^2-2p(p-q)$. Let $\sigma_2^2:=2K p^2(1-q^2) +n q^2 (1-q^2)$, then $\text{Var}[S_{ij}] \le \sigma_2^2$. Set $t_2:=K(p-q)^2/3 \le \sigma_2^2$ for all $ (i,j) $. The assumption (\ref{eq:ConditionSimple2}) implies $t>2p(p-q)$. Applying the Bernstein inequality (Theorem~\ref{thm:Bernstein}), we obtain
\begin{align*}
\mathbb{P}\{ |S_{ij} - \mathbb{E}[S_{ij}] | \ge t_2\} \le 2\exp \left( -\frac{t_2^2}{2 \sigma_2^2 + 2 t_2/3 } \right) \le 2 \exp \left(- \frac{K^2(p-q)^4}{27 \sigma_2^2 }\right) \le 2 n^{-3},
\end{align*}
where the last inequality follows from the assumption (\ref{eq:ConditionSimple2}). By the union bound, with probability at least $1-2 n^{-1}$, $S_{ij}>\frac{(p-q)^2K}{3} + 2Kpq + q^2 n$ for all nodes $i, j$ from the same cluster and $S_{ij}<\frac{(p-q)^2K}{3} + 2Kpq + q^2 n$ for all nodes $i,j$ from two different clusters. On this event the algorithm correctly identifies the true clusters.

\subsection{Proof of Theorem~\ref{thm:SimpleConverse}}\label{sec:proofSimpleConverse}
 For simplicity we assume $K$ and $n_2$ are even numbers. We partition the non-isolated nodes $V_1$ into two equal-sized subsets $V_{1+}$ and $V_{1-}$ such that half of the nodes in  each cluster are in $ V_{1+} $. Similarly, the isolated nodes $V_2$ are partitioned into two equal-sized subsets $V_{2+}$ and $V_{2-}$. The idea is to use the following large-deviation \emph{lower} bound to the $ d_i $'s and $ S_{ij} $'s.
\begin{theorem}[Theorem 7.3.1 in~\cite{matousek2008}]\label{thm:Anti}
Let $ X_1,\ldots,X_N $ be independent random variables such that $ 0\le X_i\le 1 $ for all $ i $. Suppose $ X=\sum_{i=1}^N X_i $ and $ \sigma^2 := \sum_{i=1}^N \textrm{Var}[X_i] \ge 200$. Then for all $ 0\le \tau\le \sigma^2/100 $ and some universal constant $ c_3>0 $, we have
$$
\mathbb{P}\left[X \ge \mathbb{E}[X] + \tau\right] \ge c_3e^{-\tau^2/(3\sigma^2)}.
$$
\end{theorem}
The main hurdle is that the graph adjacency matrix $ A $ are not completely independent due to the symmetry of $ A $, so we need to take care of the dependence between the $ d_i $'s and $ S_{ij} $'s before we can apply the above theorem.

\paragraph*{Identifying isolated nodes.} For each node $i$ in $V_{1+} \cup V_{2+}$, let $d_{i+}$ and $d_{i-}$ be the numbers of its neighbors in $V_{1+}\cup V_{2+}$ and $ V_{1-} \cup V_{2-}$, respectively, so its total degree is $d_i=d_{i+}+d_{i-}$. Let $\text{Bin}(N, \alpha)$ denote the binomial distribution with $N$ trials and probability $\alpha$.
We consider two cases.

{\bf Case 1}:$( Kp + (n-K) q ) \log n_1 \ge n q \log n_2$. In this case, it follows from~\eqref{eq:simple_fail_1} that
\begin{align}
(K-1)^2 (p-q)^2 \le 2 c_2 ( K p + n q  ) \log n_1. \label{eq:conditioncase1}
 \end{align}
 For each node $i \in V_{1+}$, $d_{i-}$ is a sum of two independent Binomial random variables  distributed as
$\text{Bin}(K/2,p)$ and  as $\text{Bin}((n-K)/2,q)$, respectively. Define
 \begin{align*}
 t&:=(K-1) (p-q)  + 2, \qquad \gamma_d^-:= \mathbb{E}[d_{i-}] - t = \frac{1}{2}nq + \frac{1}{2} K (p-q) -t, \\
 \sigma_d^2 &:= \text{Var}[d_{i-}] = \frac{1}{2} K p (1-p) + \frac{1}{2} (n-K) q (1-q).
 \end{align*}
By assumption, $K \le n/2$, $q \le p \le 1-c_0$ and $Kp + nq \ge Kp^2+nq^2 \ge  c_1 \log n$. Therefore $\sigma_d^2 \ge \frac{1}{4}c_0 c_1 \log n \ge 200$ by choosing the constant $c_1$ in the assumption sufficiently large. Furthermore, it follows from~\eqref{eq:simple_fail_1} that by choosing $c_1$ sufficiently large and $c_2$ sufficiently small,
\begin{align*}
\sigma_d^4  \ge  \frac{1}{4} c_0 (Kp + nq) \sigma_d^2 \ge \frac{c_0}{8c_2} \frac{(K-1)^2 (p-q)^2 }{\log n_1} \times  \frac{1}{4} c_0 c_1 \log n  \ge 100^2 t^2.
\end{align*}
We can thus apply~Theorem~\ref{thm:Anti} with~\eqref{eq:conditioncase1} to get
\begin{align*}
\mathbb{P} \left[ d_{i-}  \le \gamma_d^- \right]  \ge c_3\exp\left(-\frac{t^2}{3\sigma_d^2}\right) = \exp \left(-\frac{ ((K-1) (p-q)  + 2)^2  }{3 (K p(1-p) + (n-K)q (1-q)) } \right) \ge c_3 n_1^{-c_4},
\end{align*}
for some universal constant $c_4>0$ that can be made arbitrarily small by choosing $c_2$ in the assumption sufficiently small. Let $i^\ast:=\arg \min_{i \in V_{1+} } d_{i-}$. Since the random variables $\{ d_{i-}: i \in V_{1+} \}$ are mutually independent, we have
\begin{align*}
\mathbb{P} \left[ d_{i^\ast-}  > \gamma_d^-  \right]
 = \prod_{i \in V_{1+}} \mathbb{P} \left[ d_{i-}  > \gamma_d^-\right]
\le (1- c_3n_1^{-c_4})^{n_1/2}
\le \exp \left( - c_3 n_1^{1-c_4}/2 \right)
\le 1/4,
\end{align*}
where the last equality follows by letting $c_4$ sufficiently small and $n_1$ sufficiently large.
On the other hand, for each $ i\in V_{1+} $, $d_{i+}$ is the sum of two independent Binomial random variables distributed as $\text{Bin}(K/2-1, p)$
and  $\text{Bin}((n-K)/2,q)$, respectively. Since the median of $\text{Bin}(N,\alpha)$ is at most $N\alpha+1$, we know that with probability at least $1/2$, we have $d_{i+} \le \gamma_d^+:=nq/2 + K(p-q)/2 -p+2$.
Now observe that the two sets of random variables $ \{d_{i+}, i\in V_{1+}\} $ and $ \{d_{i-}, i\in V_{1+} \}$ are independent of each other, so $ d_{i+} $ is independent of $ i^\ast $ for each $ i\in V_{1+} $. It follows that
\begin{align*}
\mathbb{P}\left[d_{i^\ast+} \le \gamma_d^+\right]
= \sum_{i\in V_{1+} } \mathbb{P}\left[d_{i+} \le \gamma_d^+ \vert i^\ast = i\right]\mathbb{P}\left[i^\ast = i\right]
= \sum_{i\in V_{1+}} \mathbb{P}\left[d_{i+} \le \gamma_d^+\right]\mathbb{P}\left[i^\ast = i\right]
\ge \frac{1}{2}.
\end{align*}
Combining with the union bound, we obtain that with probability at least $1/4$,
\[ d_{i^\ast} = d_{i^\ast-}+d_{i^\ast+} \le \gamma_d^-+\gamma_d^+ =(n-1)q. \]On this event the node $i^\ast$ will be incorrectly declared as an isolated node.

{\bf Case 2}: $(Kp +n  q ) \log n_1 \le n q  \log n_2$. In this case we have $(K-1)^2 (p-q)^2 \le 2 c_2 n q \log n_2$ in view of~\eqref{eq:simple_fail_1}. Define $i^\ast=\arg \max_{i \in V_{2+} } d_{i-}$. Following the same argument as in Case 1 and using the assumption that $nq \ge c_1 \log n$, we can show that $d_{i^\ast} \ge nq +K(p-q) $ with probability at least $1/4$, and on this event node $i^\ast$ will incorrectly be declared as a non-isolated node.

\paragraph*{Recovering clusters.}
For two nodes $i,j \in V_1$, let $S_{ij+}$ be the number of their common neighbors in $V_{1+}\cup V_{2+}$ and $S_{ij-}$ be the number of their common neighbors in $ V_{1-} \cup V_{2-}$, so  the total number of their common neighbors is $S_{ij}=S_{ij+}+S_{ij-}$.

For each pair of nodes $i, j$ in $ V_{1+} $ from the same cluster, $S_{ij-}$ is the sum of two independent Binomial random variables distributed as $\text{Bin}(K/2, p^2)$  and $\text{Bin} ((n-K)/2, q^2)$, respectively. Define
\begin{align*}
t'&:=K (p-q)^2  + 4, \qquad \gamma_S^- := \mathbb{E}[S_{ij-}] - t' =  nq^2/2  + K(p^2-q^2)/2 - t' , \; \\
\sigma^2_S &:= \text{Var}[S_{ij-}] =\frac{1}{2} K p^2(1-p^2) + \frac{1}{2}(n-K)q^2 (1-q^2).
\end{align*}
By assumption, $K \le n/2$, $q \le p \le 1-c_0$ and $ Kp^2+nq^2 \ge  c_1 \log n$, and therefore $\sigma_S^2 \ge 200$ and $\sigma_S^2 \ge 100t'$.
Theorem~\ref{thm:Anti} with~\eqref{eq:ConditionSimple2} implies that
\begin{align*}
\mathbb{P} \left[ S_{ij-}  \le \gamma_S^- \right] \ge c_3 \exp\left(-\frac{t'^2}{3\sigma_S^2}\right)
\ge c_3 n_1^{-c_5},
\end{align*}
where the universal constant $c_5>0$ can be made sufficiently small by choosing  $c_2$ sufficiently small in~\eqref{eq:ConditionSimple2}. Without loss of generality, we may re-label the nodes such that $ V_{1+} =\{1,2,\ldots,n_1/2\} $ and for each $ k=1,\ldots,n_1/4 $, the nodes $ 2k-1 $ and $ 2k $ are in the same cluster. Note that the random variables $\{S_{(2k-1)2k-}: k=1,2,\ldots,n_1/4 \}$ are mutually independent. Let $i^\ast :=-1+ 2\arg \min_{k=1,2,\ldots,n_1/4 } S_{(2k-1)2k-}$ and $ j^\ast := i^\ast+1 $; it follows that
\begin{align*}
\mathbb{P} \left[  S_{i^\ast j^\ast-} \ge \gamma_S^-  \right]
&  \le (1- c_3n_1^{-c_5})^{n_1/4}
\le \exp ( - c_3 n_1^{1-c_5}/4 ) \le 1/4.
\end{align*}
On the other hand, since $S_{ij+}$ is the sum of two independent Binomial random variables $\text{Bin}(K/2-2, p^2)$ and $\text{Bin}((n-K)/2,q^2)$, we use a median argument similar to the one above to show that for all $ i,j $, $S_{i j +} \le \gamma_S^+:=nq^2/2 + K(p^2-q^2)/2 -2p^2+2$ with probability at least $1/2$.
Because $\{ S_{ij+}, i,j\in V_{1+}\} $ only depends on the edges between $ V_{1+} $ and $ V_{1+}\cup V_{2+} $, and $ (i^\ast,j^\ast) $ only depends on the edges between $ V_{1+} $ and $ V_{1-}\cup V_{2-}  $, we know $\{ S_{ij+}, i,j\in V_{1+}\} $ and $ (i^\ast,j^\ast) $ are independent of each other. It follows that  $ S_{i^\ast j^\ast +} \le \gamma_S^+ $ with probability at least $ 1/2 $. Applying the union bound, we get that with probability at least $1/4$, $$S_{i^\ast j^\ast} = S_{i^\ast j^\ast -} + S_{i^\ast j^\ast +} \le \gamma_S^- + \gamma_S^+ = 2(K-1) pq + (n-2K)q^2; $$ on this event the nodes $i^\ast, j^\ast$ will be incorrectly assigned to two different clusters.

\section{Proofs for Submatrix Localization}\label{sec:proof_bi}

\subsection{Proof of Theorem~\ref{thm:Impossible_bi}}
We prove the theorem using Fano's inequality. Our arguments
extend those used in~\cite{kolar2011submatrix}.
Recall that $\mathcal{Y}$ is the set of all valid bi-clustering
matrices. Let $M=n_{R}-K_{R}$ and $\bar{\mathcal{Y}}=\left\{ Y_{0},Y_{1},\ldots,Y_{M}\right\} $
be a subset of $\mathcal{Y}$ with cardinality $M+1$, which is specified later.
Let $\mathbb{P}_{(Y^\ast, A)}$ denote the joint distribution of $(Y^\ast,A)$ when
$Y^\ast$ is sampled from $\bar{\mathcal{Y}}$ uniformly at random and then $A$ is generated according to the submatrix localization model.
The minimax error probability can be bounded using the average error probability and Fano's inequality:
\begin{equation}
\inf_{\hat{Y}}\sup_{ Y^{\ast} \in \mathcal{Y}  }  \mathbb{P}\left[\hat{Y}\neq Y^{*}\right]
\ge \inf_{\hat{Y}} \mathbb{P}_{\rm U}\left[\hat{Y}\neq Y^{*}\right] \ge 1-\frac{I(Y^{\ast};A)+1}{\log|\bar{\mathcal{Y}}|}, \label{Bi_EqFano}
\end{equation}
where the last inequality the mutual information is
defined under the distribution $\mathbb{P}_{(Y^\ast, A)}$.

We construct $\bar{\mathcal{Y}}$ as follows. Let $Y_0$ be the bi-clustering matrix such that
the left clusters $\{C_k \}_{k=1}^r $ are
$C_k = \left\{ (k-1)K_{L}+1,\ldots,kK_{L}\right\} $ and the right clusters $\{D_l\}_{l=1}^r$ are  $D_l=\{ (l-1)K_{R}+1,\ldots, l K_{R}\} $. Informally, each $Y_{i}$ with $i\ge1$ is obtained from $Y_{0}$
by keeping the left clusters and swapping two right nodes in two different right clusters. More specifically, for each $i \in [M]$:
(1) $Y_i$ has the same left clusters as $Y_0$; (2) if right node $K_R+i \in D_l$, then $Y_i$ has the same right clusters as $Y_0$ except that the
first right cluster is $\{1,2,\ldots,K_{R}-1,K_{R}+i\}$ and the $l$-th right cluster is $D_l \setminus \{K_R+i \} \cup
\{ K_R \}$ instead; (3) if right node $K_R+i$ does not belong to any $D_l$, then $Y_i$ has the same right clusters as $Y_0$ except that the
first right cluster is $\{1,2,\ldots,K_{R}-1,K_{R}+i\}$ instead.

Let $\mathbb{P}_{i}$ be the distribution of $A$ conditioned on $Y^{*}=Y_{i}$, and $D\left(\mathbb{P}_{i} \| \mathbb{P}_{i'}\right)$
the KL divergence between  $\mathbb{P}_{i} $ and $ \mathbb{P}_{i'}$.
Since each $\mathbb{P}_{i}$ is a product of $n_{L}\times n_{R}$ Gaussian distributions, we have
\begin{align*}
I(Y^{*};A) & \le\frac{1}{(M+1)^{2}}\sum_{i,i'=0}^{M}D\left(\mathbb{P}_{i} \| \mathbb{P}_{i'}\right)
\\
 & \le 3 K_{L} \left[D\left(\mathcal{N}(\mu_{1},\sigma^{2})\Vert \mathcal{N}(\mu_{2},\sigma^{2})\right)+ D\left(\mathcal{N}(\mu_{2},\sigma^{2})\Vert \mathcal{N}(\mu_{1},\sigma^{2})\right)\right]
  = 3 K_{L}\frac{(\mu_{1}-\mu_{2})^{2}}{\sigma ^{2}},
\end{align*}
where we use the convexity of KL divergence the first inequality, the definition of $Y_i$ in the third inequality, and KL divergence between two Gaussian distributions in the equality. If
$(\mu_{1}-\mu_{2})^{2}  \le  \frac{\sigma^{2}\log(n_{R}-K_{R})}{12K_{L}}$, then
$I(Y;A)\le \frac{1}{2}\log(n_{R}-K_{R})=\frac{1}{2}\log\left|\bar{\mathcal{Y}}\right|.$
Since $\log (n_R-K_R) \ge \log (n_R/2) \ge 4$ if $n_R \ge 128$,
It follows from \eqref{Bi_EqFano} that the minimax error probability
is at least $1/2$.

Alternatively, we can construct $Y_{i}$ with $i\ge1$ from $Y_{0}$
by keeping the right clusters and swapping two left nodes in two different left clusters.
A similar argument shows that if $(\mu_{1}-\mu_{2})^{2}\le \frac{\sigma^{2}\log(n_{L}-K_{L})}{12K_{R}}$,
the minimax error probability is at least $1/2$.

\subsection{Proof of Theorem \ref{thm:MLE_bi}}
Let $\langle X,Y\rangle:=\text{Tr}(X^{\top}Y)$ denote the inner product
between two matrices. For any feasible solution $Y\in\mathcal{Y}$
of~(\ref{eq:mle2}), we define $\Delta(Y):=\langle A,Y^{\ast}-Y\rangle$ and $d(Y):=\langle Y^{\ast},Y^{\ast}-Y\rangle$.
To prove the theorem, it suffices to show that $\Delta(Y)>0$ for
all feasible $Y$ with $Y\neq Y^{*}$.  We may write
\begin{equation}
\Delta(Y)=\langle\mathbb{E}[A],Y^{\ast}-Y\rangle+\langle A-\mathbb{E}[A],Y^{\ast}-Y\rangle
= \mu d(Y) + \langle A-\mathbb{E}[A],Y^{\ast}-Y\rangle\label{EqDeltaY_bi}
\end{equation}
since $\mathbb{E}[A]=\mu Y^{\ast}$.
The second term above can be written as
\[
\langle A-\mathbb{E}[A],Y^{\ast}-Y\rangle=
\underbrace{\sum_{(i,j):Y_{ij}^{\ast}=1,Y_{ij}=0}(A_{ij}-\mu)}_{T_{1}(Y)}
+ \underbrace{\sum_{(i,j):Y_{ij}^{\ast}=0,Y_{ij}=1}(-A_{ij})}_{T_{2}(Y)}.
\]
Here each of $T_{1}(Y)$ and $T_{2}(Y)$ is the sum of $d(Y)$
i.i.d. centered sub-Gaussian random variables with parameter $ 1 $. By the
sub-Gaussian concentration inequality given in Proposition 5.10 in~\cite{vershynin2010nonasym}, we obtained that for each $ i=1,2 $ and each fixed $Y\in\mathcal{Y}$,
\begin{align*}
\mathbb{P}\left\{ T_{i}(Y)\le-\frac{\mu}{2}d(Y)\right\}  & \le  e \exp\left(-C \mu^2d(Y)\right),
\end{align*}
where $C>0$ is an absolute constant.
Combining with the union bound and~\eqref{EqDeltaY_bi}, we get
\begin{align}
\mathbb{P}\left\{ \Delta(Y)\le0\right\} \le2e \exp\left(-C \mu^2 d(Y)\right),\quad\text{for each }Y\in\mathcal{Y}.\label{EqProbFixY_bi}
\end{align}
Define the equivalence class $[Y]= \{ Y' \in \mathcal{Y}: Y'_{ij}=Y_{ij}, \forall (i,j) \text{ s.t. } Y^\ast_{ij}=1\}$. The following combinatorial lemma (proved in the appendix) upper-bounds the number of $Y$'s and $[Y]$'s with a fixed value of $ d(Y) $. Note that $K_{L}\wedge K_{R}\le d(Y)\le rK_{L}K_{R}$ for any feasible
$Y\neq Y^{*}$.

\begin{lemma} \label{lem:countY_bi}
For each integer $t \in[K_{L}\wedge K_{R},rK_{L}K_{R}]$, we have
\begin{align}
|\{Y\in\mathcal{Y}:d(Y)=t\}| &\le \left( \frac{16 t^2}{K_L K_R} \right)^2 n_{L}^{16t/K_{R}} n_{R}^{16t/K_{L}} \label{EqCardinality_bi} \\
| \{[Y]: d(Y)=t \} | & \le \frac{16 t^2}{K_L K_R}(r K_L)^{8t/K_{R} } (r K_R)^{8 t/K_{L}}. \label{EqCardinality_bi2}
\end{align}
\end{lemma}
Combining Lemma~\ref{lem:countY_bi} with~\eqref{EqProbFixY_bi} and the union bound, we obtain
\begin{align*}
 & \mathbb{P}\left\{ \exists Y\in\mathcal{Y}:Y\neq Y^{*},\Delta(Y)\le0\right\} \\
\le & \sum_{t=K_{L}\wedge K_{R}}^{rK_{L}K_{R}}\mathbb{P}\left\{ \exists Y\in\mathcal{Y}:d(Y)=t,\Delta(Y)\le0\right\} \\
\le & 2e \sum_{t=K_{L}\wedge K_{R}}^{rK_{L}K_{R}}\left|\{ Y\in\mathcal{Y}:d(Y)=t\}\right| \cdot \mathbb{P}\left\{ d(Y)=t,\Delta(Y)\le0\right\}\\
\le & 2e \sum_{t=K_{L}\wedge K_{R}}^{rK_{L}K_{R}}  \left( \frac{16t^2}{K_L K_R} \right)^2 n_{L}^{16t/K_{R}}n_{R}^{16t/K_{L}} \cdot \exp\left(-C \mu^2 t \right)\\
\overset{(a)}{\le} & 2e \sum_{t=K_{L}\wedge K_{R}}^{rK_{L}K_{R}} 256 n^4 n^{- 7 t /  (K_{L}\wedge K_{R}) }
\le  512e K_{L}K_{R} r n^{-3}\le 512e n^{-1},
\end{align*}
where (a) follows from the assumption that $\mu^{2} \left(K_{L}\wedge K_{R}\right)\ge C^{'} \sigma^{2}\log n$ for a sufficiently large constant $C^{'}$.
This means $Y^{*}$ is the unique optimal solution with high probability.

\subsection{Proof of Theorem \ref{thm:CVX_bi}}\label{sec:proof_easy_bi}
We proved the theorem in Section~\ref{sec:proof_CVX}.

\newif\ifcommentabcd
\commentabcdfalse
\ifcommentabcd
\fi

\subsection{Proof of Theorem~\ref{thm:Cvx_converse_bi}}\label{sec:proof_easy_converse_bi}

Observe that if any feasible solution $ Y $ has the same support as $ Y^* $, then the constraint~\eqref{eq:linear_bi} implies that $ Y $ must be exactly equal to $ Y^* $. Therefore, it suffices to show that $ Y^* $ is not an optimal solution.

The theorem assumes $ n=n_L=n_R $ and $ K=K_L=K_R $. Let $ J$ be the $n\times n$ all-one matrix, $\R:=\textrm{support}(Y^{*})$ and $\A:=\textrm{support}(A)$.
Recall that  $U, V \in \mathbb{R}^{n\times r}$ are the cluster characteristic matrices  defined in Section~\ref{sec:proof_CVX}, and $Y^\ast= K U V^\top $ is the SVD of $Y^\ast$. We may assume $U=V$.

Suppose $Y^\ast$ is an optimal solution to the program. Then by the same  argument used in the proof of Theorem \ref{thm:cvx_converse}, there
must exist some $\lambda\ge 0$, $\eta,$ $W$ and $H$ obeying the KKT conditions~\eqref{eq:subgrad}--\eqref{eq:sign}.
Since $UU^{\top}WUU^{\top}=0$ by \eqref{eq:W}, we can
left and right multiply~\eqref{eq:subgrad} by $UU^{\top}$ to obtain
\[
\bar{A}-\lambda UU^{\top}-\eta J+\bar{H}=0,
\]
where for any matrix $X\in\mathbb{R}^{n\times n}$, we define the block-averaged matrix $\bar{X}:=UU^{\top}XUU^{\top}$. Consider the last display equation on $\R$ and $\R^{c}$
respectively. By the Gaussian probability tail bound, there exists a universal  constant $c_3>0$ such that
with probability at least $1-2n^{-11}$,
\begin{align}
\mu-\frac{\lambda}{K}-\eta+\bar{H}_{ij}\ge-\frac{c_{3}\sqrt{\log n}}{K},\forall(i,j)\in\R\label{eq:equal1_bi}\\
-\eta+\bar{H}_{ij}\le\frac{c_{3}\sqrt{\log n}}{K},\forall(i,j)\in\R^{c}.\label{eq:equal2_bi}
\end{align}
Combining the last two display equations with~\eqref{eq:sign}, we get that
\begin{align*}
-\frac{c_{3}\sqrt{\log n}}{K}\le\eta \le \mu +\frac{c_{3}\sqrt{\log n}}{K}-\frac{\lambda}{K}.
\end{align*}
It follows that
\begin{align}
\lambda & \le K \mu +2c_{3}\sqrt{\log n}
 \le4\max\left\{ K \mu ,c_{3}\sqrt{\log n}\right\} .\label{eq:larger_bi}
\end{align}
Furthermore, due to~\eqref{eq:equal1_bi}, \eqref{eq:equal2_bi}
and $\lambda\ge0$, we have
\begin{align}
\bar{H}_{ij} & \le \mu +\frac{2c_{3}\sqrt{\log n}}{K} \le \mu+\frac{1}{40},\forall(i,j)\in\R^{c},\label{eq:H_bi}
\end{align}
where the last inequality holds when $K\ge c_1 \log n$.

On the other hand, \eqref{eq:W} and~\eqref{eq:subgrad} imply that
\begin{align}
\lambda^{2} & =\left\Vert \lambda(UU^{\top}+W)\right\Vert ^{2}\nonumber
 \ge\frac{1}{n}\left\Vert \lambda(UU^{\top}+W)\right\Vert _{F}^{2}\nonumber
 =\frac{1}{n}\left\Vert A-\eta J+H\right\Vert _{F}^{2}\nonumber \\
 & =\frac{1}{n}\left(\left\Vert A_{\R}-\eta J_{\R}+H_{\R}\right\Vert _{F}^{2}+\left\Vert A_{\R^{c}}-\eta J_{\R^{c}}+H_{\R^{c}}\right\Vert _{F}^{2}\right).\label{eq:fro_bi}
\end{align}
We now lower bound the RHS of~\eqref{eq:fro_bi}. For each $(i,j)$,
define the Bernoulli random variables $\bar{b}_{ij}=\mathbf{1}(A_{ij}-\mathbb{E}A_{ij}\ge1)$
and $\underline{b}_{ij}=\mathbf{1}(A_{ij}-\mathbb{E}A_{ij}\le-1)$,
where $\mathbf{1}(\cdot)$ is the indicator function. By tail bounds
of the standard Gaussian distribution, we have
\[
\mathbb{P}\left(\bar{b}_{ij}=1\right)=\mathbb{P}\left(\underline{b}_{ij}=1\right)\ge\rho:=\frac{1}{2\sqrt{2\pi}}e^{-1/2}.
\]
Note that $\rho\ge\frac{1}{12}$. By Hoeffdding's inequality, we know
that with probability at least $1-2n^{-11}$,
\begin{align}
\sum_{i,j\in\R^{c}}\bar{b}_{ij}\ge\frac{1}{2}\rho |\R^{c}|, & \quad\sum_{i,j\in\R^{c}}\underline{b}_{ij}\ge\frac{1}{2}\rho |\R^{c}|.\label{eq:anti}
\end{align}
We consider two cases below.
\begin{itemize}
\item Case 1: $\eta\ge40\mu $. By~\eqref{eq:H_bi} and the Markov inequality, there is
at most a fraction of $ $$\frac{1}{30}$ of the $(i,j)$ in $\R^{c}$
which satisfies $H_{ij}>30\left(\mu+\frac{1}{40}\right)$.
Let $\mathcal{D}$ denote the set of entries $(i,j)$ satisfying both $H_{ij}\le30\left( \mu +\frac{1}{40}\right)$
and $A_{ij} \le -1$.
In view of the second inequality in~\eqref{eq:anti}, $|\mathcal{D}|/|\R^c| \ge \rho/2-1/30 \ge 1/150$.
For $(i,j)\in\mathcal{D}$, we have $-\eta +H_{\R^{c}} \le-10 \mu+\frac{3}{4} $,
and thus
\begin{align*}
\left\Vert A_{\R^{c}}-\mu J_{\R^{c}}+H_{R^{c}}\right\Vert _{F}^{2} & \ge \sum_{(i,j)\in\mathcal{D}}\left\Vert A_{\R^{c}} -\eta J_{\R^{c}}+H_{\R^{c}}\right\Vert _{F}^{2} \\
 & \ge\sum_{(i,j)\in\mathcal{D}}\left(-1-10\mu+\frac{3}{4}\right)^{2}
  \ge\frac{1}{150}|\R^c| \cdot\frac{1}{16}.
\end{align*}

\item Case 2: $\eta \le40\mu$. Since $\mu \le\frac{1}{100}$ by assumption, we have $\eta \le 1/2$.
Then
\begin{align*}
\left\Vert A_{\R^{c}}-\eta J_{\R^{c}}+H_{R^{c}}\right\Vert _{F}^{2} & \ge \sum_{(i,j)\in\R^{c},\bar{b}_{ij}=1}\left\Vert A_{\R^{c}} -\eta J_{\R^{c}}+H_{\R^{c}}\right\Vert _{F}^{2}\\
 & \ge\sum_{(i,j)\in\R^{c}:\bar{b}_{ij}=1}(1- \eta)^{2}
  \ge\frac{1}{2}\rho  |\R^{c}| \cdot\frac{1}{4}.
\end{align*}

\end{itemize}
Combining the two cases and substituting into~\eqref{eq:fro_bi}, we obtain $\lambda^{2}\ge c_4 |\R^{c}|/n $ for a constant $c_4>0 $. Since $|\R^{c}|=n^2-rK^2 \ge n (n-K) \ge n^2/2$, we have $\lambda^2 \ge c_4 n/2$.
It follows from~\eqref{eq:larger_bi} that
\[
\max\left\{ K \mu ,c_{3} \sqrt{ \log n}\right\} \ge \frac{\sqrt{c_4}}{4\sqrt{2}} \sqrt{n }.
\]
Since $n\ge K\ge c_1\log n$ with a sufficiently large constant $c_1$, we must have $K \mu \ge \frac{\sqrt{c_4}}{4\sqrt{2}} \sqrt{n}$.
This violates the condition~\eqref{eq:converse_bi_cond} in the theorem statement by choosing the universal constant $c_2$ sufficiently small.
Therefore,  $Y^{*}$ is not an optimal solution of the convex program.

\subsection{Proof of Theorem \ref{thm:Simple_bi}\label{sec:proofSimple_bi}}

We prove that with high probability, each of the three steps of the simple thresholding algorithm succeeds and thus $Y^\ast$ is exactly recovered.

\paragraph{Identifying isolated nodes.}
Recall that $d_i=\sum_{j=1}^{n_R} A_{ij}$ is the row sum corresponding to left node~$i$. Observe that
$d_{i}-\mathbb{E}[d_i]$ is the sum of $n_R$ independent centered sub-Gaussian random variables with parameter $ 1 $.
Moreover, $\mathbb{E}[d_i]=K_R \mu$ if node $i$ is non-isolated; otherwise, $\mathbb{E}[d_i]=0$.
By Proposition 5.10 in~\cite{vershynin2010nonasym}, there exists a  universal constant $c_3>0$ such that
\begin{align*}
\mathbb{P}\{ |d_{i}-\mathbb{E}[d_{i}]|\ge K_{R}\mu/2\}\le e \exp\left(  -\frac{c_3 K_{R}^2\mu^{2}}{n_R}\right) \le e n_L^{-2},
\end{align*}
where the last inequality follows from the assumption (\ref{eq:ConditionSimple1_bi}) by choosing the universal constant $c_1$ sufficiently large.
By the union bound, with probability at least $1-en_L^{-1}$, we have $d_{i}> \mu K_{R}/2$
for all non-isolated left nodes $i$ and $d_{i}<\mu K_{R}/2$
for all isolated left nodes $i$, and therefore all isolated left nodes are correctly identified in Step 1 of the algorithm.
A similar argument shows that all isolated right nodes are correctly identified with probability at least $1-en_R^{-1}$.

\paragraph{Recovering clusters}
Recall that $S_{ii'}=\sum_{j=1}^{n_R} A_{ij} A_{i'j}$ is the inner product of two rows of $A$ corresponding to the left nodes $i,i'.$
If the two left nodes $i,i'$ are in the same cluster, then $\mathbb{E}[S_{ii'}]=K_{R}\mu^{2}$ and otherwise
$\mathbb{E}[S_{ii'}]=0$. Moreover, $A_{ij} A_{i'j}$ is a product of two independent sub-Gaussian random variables. We use $\|X \|_{\psi_2}$ and $\| X \|_{\psi_1}$ to denote the
sub-Gaussian norm and sub-exponential norm of a random variable $ X $.\footnote{The sub-exponential norm and sub-Gaussian norm of a random variable $ X $ are defined as $ \Vert X\Vert_{\psi_i} =\sup_{p\ge 1} p^{-1/i} \left(\mathbb{E}|X|^p\right)^{1/p} $ for $ i=1,2 $, respectively~\cite{vershynin2010nonasym}. Up to a universal positive constant, they are equal to the sub-exponential and sub-Gaussian parameters of $ X $, respectively. }
 It follows from the definition that
\begin{align*}
& \left\| A_{ij} A_{i'j} - \mathbb{E}[A_{ij}]\mathbb{E}[A_{i'j}] \right\|_{\psi_1} \\
\overset{(a)}{\le} &  \left\| ( A_{ij}- \mathbb{E}[A_{ij}] ) (A_{i'j}- \mathbb{E}[A_{i'j}] ) \right\|_{\psi_1} +  \left\| ( A_{ij}- \mathbb{E}[A_{ij}] ) \mathbb{E}[A_{i'j}] \right\|_{\psi_1} + \left\| ( A_{ij}- \mathbb{E}[A_{ij}]) \mathbb{E}[A_{i'j}] \right\|_{\psi_1} \\
\overset{(b)}{\le} &   2 \left\|A_{ij}\!-\! \mathbb{E}[A_{ij}] \right\|_{\psi_2} \left\|A_{i'j}\!-\! \mathbb{E}[A_{i'j}] \right\|_{\psi_2} + 2 \mu \left\|A_{ij}\!-\! \mathbb{E}[A_{ij}]\right\|_{\psi_2}  + 2\mu \left\|A_{i'j}\!-\! \mathbb{E}[A_{i'j}] \right\|_{\psi_2}
\overset{(c)}{\le} c'(4\mu+2),
\end{align*}
where $(a)$ and $ (b) $ follow from $\|X+Y\|_{\psi_1} \le \|X\|_{\psi_1} + \|Y\|_{\psi_1} $ and $\|XY \|_{\psi_1} \le 2 \|X\|_{\psi_2} \|Y\|_{\psi_2}$ for any random variables $X,Y$, and $(c)$ holds for some universal constant $ c'>0 $ because $A_{ij}- \mathbb{E}[A_{ij}]$ is sub-Gaussian with parameter $ 1 $ for each $i,j$.
By the Bernstein inequality for sub-exponential random variables given in Proposition 5.16 in \cite{vershynin2010nonasym}, there  exists some universal constant $c_4>0$ such that
\[
\mathbb{P}\left\{|S_{i i'}-\mathbb{E}[S_{ii' }]|\ge K_{R}\mu^{2}/2 \right\}
\le e \exp\left[-c_4 \min\left(\frac{K_{R}^{2}\mu^{4}}{ n_R c^{'2}(4\mu +2)^2},\frac{K_{R}\mu ^{2}}{c'(4\mu + 2)}\right)\right]
\le e (rK_L)^{-3},
\]
where the last inequality follows from the conditions~(\ref{eq:ConditionSimple2_bi})
and (\ref{eq:ConditionSimple1_bi}). By the union bound, with probability
at least $1-e (rK_L)^{-1}$, $S_{ii'}>\frac{\mu^{2}K_{R}}{2} $
for all left nodes $i,i'$ from the same left cluster and $S_{ii'}<\frac{\mu^2 K_{R}}{2}$
for all left nodes $i,i'$ from two different left clusters, and on this event
Step 2 of the algorithm returns the true left clusters. A similar argument shows that the algorithm also returns the true right clusters with probability at least $1-e (rK_R)^{-1}$.

\paragraph{Associating left and right clusters.}
Recall that $B_{kl}=\sum_{i \in C_k, j \in D_l} A_{ij}$ is the block sum of $A$ with left clusters given by $\{C_k\}_{k=1}^r$ and right clusters given by $\{D_l\}_{l=1}^r$.
By model assumptions, $B_{kl} -\mathbb{E}[B_{kl}]$ is a sum of $K_L K_R$ independent centered sub-Gaussian random variables with parameter $ 1 $.
Moreover, $\mathbb{E}[B_{kl}]=\mu K_L K_R$ if $k=l$ and  $\mathbb{E}[B_{kl}]=0$ otherwise.
By the standard sub-Gaussian concentration inequality given in Proposition 5.10 in~\cite{vershynin2010nonasym}, there exists some universal constant $c_5>$ such that
\begin{align*}
\mathbb{P}\{ | B_{kl}-\mathbb{E}[B_{kl} ] |\ge \mu K_LK_R/2 \}\le e \exp\left(  -\frac{c_5 \mu^{2} K_L^2 K_R^2 }{K_LK_R} \right) \le e n^{-3},
\end{align*}
where the last inequality holds because $\mu^2K_L K_R \ge c_1 \log n$ in view of~\eqref{eq:ConditionSimple1_bi}.
By the union bound, with probability at least $1-en^{-1}$, $B_{kl} <\mu K_LK_R/2$
for all $k=l$ and $B_{kl} > \mu K_LK_R/2$ for all $k \neq l$. On this event, Step 3 of the algorithms correctly associate left and right clusters.

\subsection{Proof of Theorem~\ref{thm:SimpleConverse_bi}\label{sec:proofSimpleConverse_bi}}
We focus on identifying left isolated nodes and left clusters. The proof for the right nodes is identical. We will show that some of the $ d_i $ and $ S_{ii'} $'s will have large deviation from their expectation.

\paragraph{Identifying isolated nodes.}
Assume  $rK_L \ge n_L/2$ first.  We will show that if $K_R^2\mu^2 \le c_1 n_R  \log n_L $ for a sufficiently small universal constant $c_1$,
then with high probability there exists a non-isolated left node $i^\ast$  that is incorrectly declared as isolated.
Recall that $d_{i}=\sum_{j=1}^{n_R}A_{ij}$ is the row sum corresponding to the left node $i$. If the left node $i$ is non-isolated, then $d_i$ is Gaussian
with mean $K_R \mu$ and variance $n_R$. For a standard Gaussian random variable $ Z $, its tail probability is lower bounded as $Q(t) := \mathbb{P}\left[Z \ge t\right] \ge \frac{1}{\sqrt{2\pi} } \frac{t}{t^2+1} \exp(-t^2/2 )$. It follows that for a non-isolated left node $i$, there exists two positive universal constants $c_3,c_4$ such that
\begin{align*}
\mathbb{P} \left[d_i-\mathbb{E}[d_i] \le K_R\mu/2 \right] \ge c_3 \exp \left(  -\frac{c_4 K_R^2\mu^2 }{n_R} \right) \ge c_3 n_L^{-c_1c_4}.
\end{align*}
Let $i^\ast$ be the non-isolated left node with the minimum $d_i$. Since $\{d_i\}_{i=1}^{n_L}$ are mutually independent,
\begin{align*}
\mathbb{P} \left[d_{i^\ast}> \frac{ K_R \mu }{2} \right]  \le \left(1-c_3 n_L^{-c_1 c_4 }\right)^{rK_L} \le \exp\left(- \frac{1}{2} c_3 n_L^{1-c_1c_4 }\right),
\end{align*}
where the last inequality holds because $rK_L \ge n_L/2$. By choosing $c_1$ sufficiently small,  with high probability the non-isolated left node $i^\ast$ will be incorrectly declared as an isolated node.

If $rK_L \le n_L/2$, then we can similarly show that if $K_R^2\mu^2 \le c_1 n_R  \log n_L$ for a small $c_1$,
then with high probability there exists a isolated left node $i^{**}$ incorrectly declared as non-isolated.

\paragraph{Recovering clusters.}
We will show that if
\begin{align}
K_R^2 \mu^4 \le c_2 n_R \log (rK_L), \label{eq:conditionconversesimple1}
\end{align}
for a sufficiently small constant  $c_2$, then there exist two left nodes $i_1,i_2 $ in two different clusters
which will be incorrectly assigned to the same cluster. Since $K_L,K_R \ge \log n$, it follows from~\eqref{eq:conditionconversesimple1} that $K_R \mu^2 \le c_2 n_R$ and $K_R \mu^3 \le c_2^{3/4} n_R$.

Recall that $S_{i i'}=\sum_{j=1}^{n_{R} }A_{i j }A_{i' j}$. For two left nodes $i,i'$ from two different clusters, we have
\begin{align*}
\mathbb{E}[S_{ii'}]=0, \quad \text{Var}[S_{ii'}] =2K_R\mu^2+n_R \le (2c_2+1)n_R, \\
\sum_{j=1}^{n_R} \mathbb{E}[|A_{i j }A_{i' j}|^3] \le c_5(K_R\mu^3 + n_R) \le c_5(c_2^{3/4}+1) n_R,
\end{align*}
where $c_5$ is some universal positive constant.
By the Berry-Esseen theorem, there exists a positive universal constant $c_6$ such that
\begin{align*}
\mathbb{P}\left[ S_{ii'} \ge  \frac{\mu^2K_R}{2}  \right]
&\ge Q\left( \frac{\mu^2 K_R}{2 \sqrt{2K_R \mu^2 +n_R} } \right)- \frac{c_6(K_R\mu^3 + n_R)}{ ( 2K_R\mu^2+n_R)^{3/2}} \\
& \overset{(a)}{\ge} Q\left( \frac{\mu^2 K_R}{ \sqrt{n_R}}  \right)-\frac{c_6 c_5(c_2^{3/4}+1) }{\sqrt{n_R}}  \\
& \overset{(b)}{\ge} Q\left( \sqrt{c_2 \log (rK_L) }  \right) -  \frac{c_6 c_5(c_2^{3/4}+1) }{\sqrt{rK_L}} \\
& \overset{(c)}{\ge} c_3 (rK_L)^{-c_4 c_2}- c_6 c_5(c_2^{3/4}+1) (rK_L)^{-1/2} \overset{(d)}{\ge} c_7 (rK_L)^{-c_4 c_2},
\end{align*}
where $(a)$ holds because $Q(t)$ is non-increasing in $t$, $(b)$ holds in view of~\eqref{eq:conditionconversesimple1} and the assumption that $n_R \ge r K_L$,
$(c)$ follows because $Q(t) \ge c_3 \exp(-c_4t^2)$, and $(d)$ holds for some universal constant $c_7>0$ by choosing $c_2$ sufficiently small.

Define $(i_1, i_2):=\arg \max_{(i,i') \in W} S_{ii'}$, where $W$ is the maximal set of node pairs $(i,i')$ such that $ (i) $ $i,i'$ are from two different clusters, and $ (ii) $ for any $(i,i'),(j,j') \in W$, $i,i',j,j'$ are all distinct. Then $|W|\ge rK_L/4$ and
$\{S_{ii'}: (i,i')\in W \}$ are mutually independent. It follows that
\begin{align*}
\mathbb{P} \left[ S_{i_1i_2} < \frac{\mu^2K_R}{2}   \right]
\le \left(1-c_7(rK_L)^{-c_4c_2} \right)^{rK_L/4}
\le \exp\left(- \frac{1}{4} c_7 (r K_L)^{1-c_4c_2} \right).
\end{align*}
Therefore, with probability at least $1- \exp\left(- \frac{1}{4 } c_7 (r K_L)^{1-c_4c_2} \right)$, we have $S_{i_1 i_2} \ge  \frac{\mu^2 K_R}{2} $. On this event, $(i_1,i_2)$ will be incorrectly assigned to the same cluster.

\subsection{Proof of Theorem~\ref{thm:element_bi}}
We prove the first part of the theorem. Since $A_{ij}$ are sub-Gaussian, there exists a universal constant $c_{1}>0$ such that
$\mathbb{P}\left(\left|A_{ij}-\mathbb{E}A_{ij}\right|\le\frac{1}{2}\sqrt{c_{1} \log n}\right)\ge1-n^{-12}$
for each $(i,j)$. Recall that $\R=\textrm{support}(Y^{*})$. By the union bound over all $(i,j)$, we obtain
that with probability at least $1-n^{-3}$,
\begin{align*}
\min_{i,j\in\R } A_{ij} >\mu-\frac{1}{2}\sqrt{ c_{1}\log n} \overset{(a)}{>} \frac{1}{2} \mu,\qquad
\max_{i,j\in\R^{c}}A_{ij}  < \frac{1}{2}\sqrt{c_{1} \log n} \overset{(b)}{<} \frac{1}{2} \mu,
\end{align*}
where $(a)$ and $(b)$ holds in view of the assumption~\eqref{eq:element_cond}.
Therefore, the algorithm sets $\hat{Y}_{ij}=1$ for $(i,j)\in\R$ and $\hat{Y}_{ij}=0$ for $(i,j)\in\R^{c}$, which implies $\hat{Y}=Y^{*}.$

For  the second part of the theorem, note that $\left\{ A_{ij}\right\} $ are Gaussian variables obeying the tail bound
\begin{align*}
\mathbb{P}\left(A_{ij}\ge\mathbb{E}A_{ij}+\sqrt{\log n}\right) & \ge\frac{1}{\sqrt{2\pi n\log n}}.
\end{align*}
By the independency of $\left\{ A_{ij}\right\} $, we obtain
\[
\mathbb{P}\left(\max_{i,j\in\R^{c}}A_{ij} < \sqrt{\log n}\right) \le \left(1-\frac{1}{\sqrt{2\pi n\log n}}\right)^{|\mathcal{R}^c|} \le \exp \left( - \frac{1}{2 \sqrt{ 2\pi} } \sqrt{\frac{n}{\log n}} \right),
\]
where the last inequality holds because $|\mathcal{R}^c| \ge \frac{1}{2} n_L n_R \ge \frac{1}{2} n$.
In view of the assumption~\eqref{eq:element_fail_cond}, we conclude that
$
\max_{i,j\in\R^{c}}A_{ij}\ge \sqrt{\log n}\ge\frac{1}{2}\mu
$ with probability at least $1- \exp \left( - \frac{1}{2 \sqrt{ 2\pi} } \sqrt{\frac{n}{\log n}} \right)$. On this event the algorithm will incorrectly set $\hat{Y}_{ij}=1$ for some $(i,j)\in\R^{c}$.

\section{Acknowledgement}
The authors would like to thank Sivaraman Balakrishnan, Bruce Hajek and Martin J. Wainwright for inspiring discussions.
J. Xu acknowledges the support of the National Science Foundation under Grant ECCS
10-28464.


\appendix

\section{Proof of Lemmas \ref{lem:countY} and \ref{lem:countY_bi}}

Notice that Lemma \ref{lem:countY} is a special case of Lemma \ref{lem:countY_bi} with $n_L=n_R$, $K_L=K_R$ and the left clusters identical to the right clusters. Hence we only need to prove Lemma \ref{lem:countY_bi}.

Recall that $C_{1}^{\ast},\ldots,C_{r}^{\ast}$ ($D_{1}^{*},\ldots,D_{r}^{*}$, resp.) denote the true left (right, resp.) clusters associated with $Y^{*}$. The nodes in $V_L \setminus \left( \cup_{k=1}^r C_k^\ast \right)$ do not belong to any left clusters and are called isolated left nodes. Isolated right nodes are similarly defined.

Fix a $Y\in\mathcal{Y}$ with $d(Y):=\langle Y^\ast, Y-Y^\ast \rangle = t$. Based on $ Y $, we construct a new ordered partition  $(C_{1},\ldots,C_{r+1})$
of $V_{L}$ and a new ordered partition $(D_{1},\ldots,D_{r+1})$ of $V_{R}$ as follows.
\begin{enumerate}
\item Let $C_{r+1}:=\{i:Y_{ij}=0,\forall j\}$ and $D_{r+1}:=\{j:Y_{ij}=0,\forall i\}$.
\item The left nodes in $V_{L}\setminus C_{r+1}$ are further partitioned
into $r$ new left clusters of size $K_L$, such that left nodes $i$ and $i'$ are
in the same cluster if and only if the $i$-th and $ i' $-th rows of $Y$ are identical.
Similarly, the right nodes in $V_{R}\setminus D_{r+1}$ are partitioned into
$r$ new right clusters of size $K_R$ according to the columns of $ Y $.
We now define an ordering $C_{1},\ldots,C_{r}$ of these $r$ new left clusters and an ordering $D_1, \ldots, D_r$ for the right clusters using the following procedure.
\begin{enumerate}
\item For each new left cluster $C$, if there exists a $k\in[r]$ such that
$|C\cap C_{k}^{\ast}|>K_L/2$, then we label this new left cluster as $C_{k}$;
this label is unique because the left cluster size is $K_L$. The corresponding right cluster $\{j: Y_{ij}=1, \forall i \in C_k \}$ is labeled as as $D_{k}$.
\item For each remaining unlabeled right cluster $D$, if there exists a $k \in [r]$ such that
$| D \cap D_{k}^\ast |>K_R/2$, then we label this new right cluster as $D_{k}$; again this label is unique. We label the corresponding left cluster $\{i:Y_{ij}=1, \forall j \in D_k \}$ as $C_k$.
\item The remaining unlabeled left clusters are labeled arbitrarily. For each remaining unlabeled right cluster, we label it according to $D_k:=\{j: Y_{ij}=1, \forall i \in C_k \}$.
\end{enumerate}
\end{enumerate}
For each $ (k,k')\in [r]\times[r+1] $, we use  $ \alpha_{kk'}:=|C^*_{k}\cap C_{k'}| $ and $ \beta_{kk'}:=|D^*_{k}\cap D_{k'}| $ to denote the sizes of intersections of the true and new clusters. We observe that the new clusters $(C_{1},\ldots,C_{r+1},D_{1},\ldots,D_{r+1})$ have the following three properties:
\begin{itemize}
\item[(A0)] $(C_{1},\ldots,C_{r},C_{r+1})$ is a partition of $V_{L}$ with
$|C_{k}|=K_L$ for all $k\in[r]$; $(D_{1},\ldots,D_{r},D_{r+1})$ is a partition of $V_{R}$ with $|D_{k}|=K_R$ for all $k\in[r]$.
\item[(A1)] For each $k\in[r]$, exactly one of the following is true: (1)  $\alpha_{kk}>K_L/2$;
(2) $\alpha_{kk'}\le K_L/2$ for all $k^{\prime}\in[r]$ and $\beta_{kk} >K_R/2$; (3) $\alpha_{kk'}\le K_L/2$  and $\beta_{kk'}\le K_R/2$ for all $k^{\prime}\in[r]$.
\item[(A2)] We have
\begin{align*}
\sum_{k=1}^{r}\left( \alpha_{k(r+1)}\beta_{k(r+1)}+\sum_{k^{\prime},k^{\prime\prime}:
k^{\prime}\neq k^{\prime\prime
}
}\alpha_{kk'}\beta_{kk''}\right)=t;
\end{align*}
here and henceforth, all the summations involving $ k' $ or $ k'' $ (as the indices of the new clusters) are over the range $ [r+1] $ unless defined otherwise.

\end{itemize}
Here, Property (A0) holds due to $Y \in \mathcal{Y}$; Property (A1) is direct consequence of how we label
the new clusters, and Property (A2) follows from the following:
\begin{align*}
t=d(Y)
= & \sum_{k=1}^{r}|\{(i,j):(i,j)\in C_{k}^{*}\times D_{k}^{*},Y_{ij}=0\}|\\
= & \sum_{k=1}^{r}|\{(i,j):(i,j)\in C_{k}^{*}\times D_{k}^{*},(i,j)\in C_{r+1}\times D_{r+1}\}|\\
 & +\sum_{k=1}^{r}\sum_{(k',k''):
k'\neq k''}|\{(i,j):(i,j)\in C_{k}^{*}\times D_{k}^{*},(i,j)\in C_{k^{\prime}}\times D_{k^{\prime\prime}}\}|.
\end{align*}
Since a different $Y$ corresponds to a different ordered partition, and the ordered partition
for any given $Y$ with $d(Y)=t$ must satisfy the above three properties, we obtain the following bound on the cardinality of the set of interest:
\begin{align}
|\{Y\in\mathcal{Y}:d(Y)=t\}|\le\vert\{(C_{1},\ldots,C_{r+1},D_{1},\ldots,D_{r+1}):\text{it satisfies (A0)--(A2)}\}\vert.\label{EqCountingBound2}
\end{align}
It remains to upper-bound the right hand side of (\ref{EqCountingBound2}).

Fix any ordered partition $(C_{1},\ldots,C_{r},C_{r+1},D_{1},\ldots,D_{r},D_{r+1})$
with properties (A0)--(A2). Consider the first true left cluster $C_{1}^{*}$.
Define $m_{1}^{(L)}:=\sum_{k':k^{\prime}\neq1}\alpha_{1k'}$,
which can be considered as the number of nodes in $C_{1}^{*}$ that are misclassified by $ Y $. Analogously define $m_{1}^{(R)}: =\sum_{k'':
k''\neq1} \beta_{1k''} $. We consider the following two cases for the values of $\alpha_{11}  $.
\begin{itemize}
\item If $\alpha_{11}>K_{L}/4$, then
\[
\sum_{(k',k''):
k'\neq k''}\alpha_{1k'}\beta_{1k''}
\ge \alpha_{11}\sum_{k'':k''\neq1} \beta_{1k''}>\frac{1}{4}m_{1}^{(R)}K_{L}.
\]
\item If $\alpha_{11}\le K_{L}/4$, then  $m_{1}^{(L)}\ge3K_{L}/4$, and we must also have
$\alpha_{1k'}\le K_{L}/2$ for all
$1\le k'\le r$ by Property (A1). Hence,
\begin{align*}
 & \sum_{(k',k''):k'\neq k''
}\alpha_{1k'}\beta_{1k''}+\alpha_{1(r+1)}\beta_{1(r+1)}\\
\ge&   \sum_{(k',k''):k'\neq k''}\mathbf{1}\left\{ k^{\prime}\neq1\right\} \mathbf{1}\left\{ k^{\prime\prime}\neq1\right\}\alpha_{1k'}\beta_{1k''}+\alpha_{1(r+1)}\beta_{1(r+1)}\\
=&  m_{1}^{(L)}m_{1}^{(R)}-\sum_{2\le k^{\prime}\le r}\alpha_{1k'}\beta_{1k'}
\ge   m_{1}^{(L)}m_{1}^{(R)}-\frac{1}{2} K_{L} m_{1}^{(R)}
\ge \frac{1}{4}m_{1}^{(R)}K_{L}.
\end{align*}
Similarly, we consider the following three cases for the values of $  \beta_{11}$.
\item If $\beta_{11}>K_{R}/4$, then
\[
\sum_{(k^{\prime},k^{\prime\prime}):
k^{\prime}\neq k^{\prime\prime}}\alpha_{1k'}\beta_{1k''}\ge \beta_{11}\sum_{k':k'\neq1}\alpha_{1k'}>\frac{1}{4}m_{1}^{(L)}K_{R}.
\]

\item If $\beta_{11}\le K_{R}/4$ and $\beta_{1k''}\le K_{L}/2$ for all
$1< k''\le r$, then similarly to the second case for $ \alpha_{11} $ above, we have
\[
\sum_{(k',k''):k'\neq k''}\alpha_{1k'}\beta_{1k''}+\alpha_{1(r+1)}\beta_{1(r+1)}
\ge\frac{1}{4}m_{1}^{(L)}K_{R}.
\]
\item If $\beta_{11}\le K_{R}/4$ and $\beta_{1k_0}> K_{R}/2$ for some $1<k_0 \le r$, then by Property (A1)
we must have $\alpha_{11} > K_L/2$. It follows that $m_1^{(L)} < K_L/2$ and
\[
\sum_{(k',k''):k'\neq k''}
\alpha_{1k'}\beta_{1k''}\ge \alpha_{11}\beta_{1k_0} > K_{L} K_R /4  \ge \frac{1}{2} m_1^{(L)} K_R.
\]
\end{itemize}
Combining the above five cases, we conclude that we always have
\[
\sum_{(k',k''):k'\neq k''}\alpha_{1k'}\beta_{1k''}+\alpha_{1(r+1)}\beta_{1(r+1)}\ge\frac{1}{4}\left(m_{1}^{(L)}K_{R}\vee m_{1}^{(R)}K_{L}\right).
\]
This inequality continue to hold if we replace $\alpha_{1k'}$, $\beta_{1k''}$, $ m_{1}^{(L)} $ and $m_{1}^{(R)}$ respectively by $\alpha_{kk'}$, $\beta_{kk''}$,   $m_{k}^{(L)}$ and
 $m_{k}^{(R)}$ (defined in a similar manner) for each $k\in[r]$.
Summing these inequalities over $ k\in[r] $ and using Property (A2), we obtain
\begin{align*}
t = \sum_{k=1}^{r}\left \{ \alpha_{k(r+1)}\beta_{k(r+1)}+\sum_{(k',k''):k'\neq k''}\alpha_{kk'}\beta_{kk''}\right \}
\ge \left(\frac{K_{L}}{4}\sum_{k=1}^{r}m_{k}^{(R)}\right)\vee\left(\frac{K_{R}}{4}\sum_{k=1}^{r}m_{k}^{(L)}\right).
\end{align*}
In other words, we have $\sum_{k\in[r]}m_{k}^{(L)}\le4t/K_{R}$
and $\sum_{k\in[r]}m_{k}^{(R)}\le4t/K_{L}$, i.e., the total number
of misclassified non-isolated left (right, resp.) nodes is upper
bounded by $4t/K_{R}$ ($4t/K_{L}$, resp.). This means that the total
number of misclassified isolated left (right, resp.) nodes is
also upper bounded by $4t/K_{R}$ ($4t/K_{L}$, resp.), because by the cluster size constraint in Property (A0), one misclassified isolated  node must produce
one misclassified non-isolated  node.

We can now upper-bound the right hand side of~\eqref{EqCountingBound2} using the above relation between the value of $ t $ and the misclassified nodes. For a $ Y $ with $ d(Y)=t $, the pair of numbers of misclassified left nodes (isolated and non-isolated) can take at most $\left(4t/K_{R}\right)^2$  different values; similarly for the right nodes with the bound $\left(4t/K_{L} \right)^2$. Given these numbers of misclassified nodes, there are at most $n_L^{8t/K_R}n_R^{8t/K_L}$ different ways to choose the identity of these misclassified nodes.
Each misclassified non-isolated left  node can then be assigned to one of $r-1 \le n_L$ different left  clusters
or leave isolated, and each misclassified isolated left node can be
assigned to one of $r \le n_L$ different left clusters; an analogous statement holds for the right nodes.
Hence, the right hand side of (\ref{EqCountingBound2}) is upper bounded by $\left(\frac{16t^2}{K_LK_R}\right)^2 n_{L}^{16t/K_{R}}n_{R}^{16t/K_{L} }$. This proves the first part of the lemma.

To count the number of possible equivalence classes $[Y]$, we use a similar argument but only need to consider the misclassified \emph{non-isolated} nodes. The number of misclassified  non-isolated left (right, resp.) nodes can take at most $4t/K_{R}$ ($4t/K_{L}$, resp.) different values. Given these numbers, there are at most $(rK_L)^{4t/K_R}(rK_R)^{4t/K_L}$ different ways to choose the identity of the misclassified non-isolated nodes. Each misclassified non-isolated left (right, resp.) node then can be assigned to one of $r-1 $ different left (right, resp.) clusters or leave isolated.
Therefore, the number of possible equivalence classes $[Y]$ with $ d(Y)=t $ is upper bounded by $\frac{16 t^2}{K_L K_R} (r K_L)^{8t/K_{R} } (r K_R)^{8 t/K_{L}}$.

\section{Proof of Lemma \ref{lmm:kldivergence}}
Notice that the inequality~\eqref{eq:bounddivergence2} follows from \eqref{eq:bounddivergence1} by replacing $p=1-q'$ and $q=1-p'$,
so it suffices to prove~\eqref{eq:bounddivergence1}. If $u \ge v$, then
\begin{align}
D\left(u \Vert v  \right) & =  u \log \frac{u}{v} + (1-u) \log \frac{1-u}{1-v} \le  u \log \frac{u}{v} \label{eq:divergenceupbound} \\
 D\left(u \Vert v  \right) & \ge u \log \frac{u}{v} + (1-u) \log (1-u) \overset{(a)}{\ge} u \log \frac{u}{e v}, \label{eq:divergencelowbound}
\end{align}
where $(a)$ follows from the inequality $x \log x \ge x-1, \forall x \in [0,1]$.
We divide the analysis into two cases:
\begin{itemize}
\item Case 1: $p \le 8 q$. In view of \eqref{eq:boundDivergence} and \eqref{eq:lowerboundDivergence},
$ D\left(p \Vert q \right) \le \frac{(p-q)^2}{q(1-q)}$ and $  D \left(\frac{p+q}{2} \Vert q \right) \ge \frac{(p-q)^2}{4(p+q) (1-q)}$.
Since $p \le 8q$, it follows that $D \left(\frac{p+q}{2} \Vert q \right) \ge \frac{(p-q)^2}{36q (1-q)} \ge \frac{1}{36}D\left(p \Vert q \right)$.

\item Case 2: $p > 8q$. In view of \eqref{eq:divergenceupbound} and \eqref{eq:divergencelowbound}, $ D\left(p \Vert q \right) \le p \log \frac{p}{q}$
and $  D \left(\frac{p+q}{2} \Vert q \right) \ge \frac{p+q}{2} \log \frac{p+q}{2 e q}$. Since $p >8q$ and $8 >e^2$,
it follows that $\log \frac{p}{q} > \frac{6}{5} \log (2 e)$ and thus $  D \left(\frac{p+q}{2} \Vert q \right) \ge \frac{p}{12} \log \frac{p}{q} \ge \frac{1}{12} D\left(p \Vert q \right).$
\end{itemize}

\section{The Bernstein Inequality}
\begin{theorem}[Bernstein] \label{thm:Bernstein}
Let $X_1, \ldots, X_N$ be independent random variables such that $| X_i | \le M$ almost surely. Let  $\sigma^2= \sum_{i=1}^N \text{Var}(X_i)$, then for any $ t\ge 0 $,
\begin{align}
\mathbb{P} \left[ \sum_{i=1}^N X_i \ge t \right] \le \exp \left(  \frac{-t^2}{ 2 \sigma^2 + \frac{2}{3} M t }\right). \nonumber
\end{align}
A consequent of the above inequality is
$
\mathbb{P} \left[ \sum_{i=1}^N X_i \ge \sqrt{2 \sigma^2 u} + \frac{2M u }{3}  \right]\le e^{-u} \nonumber
$
for any $ u>0 $.
\end{theorem}

\bibliographystyle{abbrv}
\bibliography{PlantedXXX}

\begin{thebibliography}{10}

\bibitem{Abbe14}
E.~Abbe, A.~S. Bandeira, and G.~Hall.
\newblock Exact recovery in the stochastic block model.
\newblock {\em Arxiv preprint arXiv:1405.3267}, 2014.

\bibitem{Abbe15}
E.~Abbe and C.~Sandon.
\newblock Community detection in general stochastic block models: fundamental
  limits and efficient recovery algorithms.
\newblock {\em arXiv:1503.00609}, 2015.

\bibitem{ailon2013breaking}
N.~Ailon, Y.~Chen, and H.~Xu.
\newblock Breaking the small cluster barrier of graph clustering.
\newblock In {\em Proceedings of the 30th International Conference on Machine
  Learning}, pages 995--1003, 2013.

\bibitem{alon2007testing}
N.~Alon, A.~Andoni, T.~Kaufman, K.~Matulef, R.~Rubinfeld, and N.~Xie.
\newblock Testing k-wise and almost k-wise independence.
\newblock In {\em Proceedings of the thirty-ninth annual ACM symposium on
  Theory of computing}, pages 496--505. ACM, 2007.

\bibitem{alon1997coloring3}
N.~Alon and N.~Kahale.
\newblock A spectral technique for coloring random 3-colorable graphs.
\newblock {\em SIAM Journal on Computing}, 26(6):1733--1748, 1997.

\bibitem{Alon98}
N.~Alon, M.~Krivelevich, and B.~Sudakov.
\newblock Finding a large hidden clique in a random graph.
\newblock {\em Random Structures and Algorithms}, 13(3-4):457--466, 1998.

\bibitem{ames2012clustering}
B.~P.~W. Ames.
\newblock Guaranteed clustering and biclustering via semidefinite programming.
\newblock {\em Mathematical Programming}, pages 1--37, 2013.

\bibitem{ames2010kclique}
B.~P.~W. Ames and S.~Vavasis.
\newblock Convex optimization for the planted k-disjoint-clique problem.
\newblock {\em Mathematical Programming}, 143(1--2):299--337, 2014.

\bibitem{ames2011plantedclique}
B.~P.~W. Ames and S.~A. Vavasis.
\newblock Nuclear norm minimization for the planted clique and biclique
  problems.
\newblock {\em Mathematical programming}, 129(1):69--89, 2011.

\bibitem{amini2013pseudo}
A.~Amini, A.~Chen, P.~J. Bickel, and E.~Levina.
\newblock Pseudo-likelihood methods for community detection in large sparse
  networks.
\newblock {\em The Annals of Statistics}, 41(4):2097--2122, 2013.

\bibitem{amini2009sparsePCA}
A.~A. Amini and M.~J. Wainwright.
\newblock High-dimensional analysis of semidefinite relaxations for sparse
  principal components.
\newblock {\em The Annals of Statistics}, 37(5):2877--2921, 2009.

\bibitem{anandkumar2013tensormixed}
A.~Anandkumar, R.~Ge, D.~Hsu, and S.~M. Kakade.
\newblock A tensor spectral approach to learning mixed membership community
  models.
\newblock {\em Journal of Machine Learning Research}, 15:2239--2312, June 2014.

\bibitem{arias2011anomalous}
E.~Arias-Castro, E.~J. Cand{\`e}s, and A.~Durand.
\newblock Detection of an anomalous cluster in a network.
\newblock {\em The Annals of Statistics}, 39(1):278--304, 2011.

\bibitem{arias2013community}
E.~Arias-Castro and N.~Verzelen.
\newblock Community detection in random networks.
\newblock {\em arXiv preprint arXiv:1302.7099}, 2013.

\bibitem{balakrishnan2011tradeoff}
S.~Balakrishnan, M.~Kolar, A.~Rinaldo, A.~Singh, and L.~Wasserman.
\newblock Statistical and computational tradeoffs in biclustering.
\newblock In {\em NIPS 2011 Workshop on Computational Trade-offs in Statistical
  Learning}, 2011.

\bibitem{balakrishnan2011threshold}
S.~Balakrishnan, M.~Xu, A.~Krishnamurthy, and A.~Singh.
\newblock Noise thresholds for spectral clustering.
\newblock In {\em Advances in Neural Information Processing Systems 25}, 2011.

\bibitem{bansal2004correlation}
N.~Bansal, A.~Blum, and S.~Chawla.
\newblock Correlation clustering.
\newblock {\em Machine Learning}, 56(1):89--113, 2004.

\bibitem{berthet2013lowerSparsePCA}
Q.~Berthet and P.~Rigollet.
\newblock Complexity theoretic lower bounds for sparse principal component
  detection.
\newblock {\em Journal of Machine Learning Research: Workshop and Conference
  Proceedings}, 30:1046--1066, 2013.

\bibitem{bhamidi2012energy}
S.~Bhamidi, P.~S. Dey, and A.~B. Nobel.
\newblock Energy landscape for large average submatrix detection problems in
  gaussian random matrices.
\newblock {\em arXiv preprint arXiv:1211.2284}, 2012.

\bibitem{Bicke09}
P.~J. Bickel and A.~Chen.
\newblock A nonparametric view of network models and {Newman}--{Girvan} and
  other modularities.
\newblock {\em Proceedings of the National Academy of Sciences},
  106(50):21068--21073, 2009.

\bibitem{bollobas2004maxcut}
B.~Bollob{\'a}s and A.~Scott.
\newblock Max cut for random graphs with a planted partition.
\newblock {\em Combinatorics, Probability and Computing}, 13(4-5):451--474,
  2004.

\bibitem{butucea2011submatrix}
C.~Butucea and Y.~I. Ingster.
\newblock Detection of a sparse submatrix of a high-dimensional noisy matrix.
\newblock {\em Bernoulli}, 19(5B):2652--2688, 2013.

\bibitem{cai2014robust}
T.~Cai and X.~Li.
\newblock Robust and computationally feasible community detection in the
  presence of arbitrary outlier nodes.
\newblock {\em arXiv preprint arXiv:1404.6000}, 2014.

\bibitem{chandrasekaran2013tradeoff}
V.~Chandrasekaran and M.~I. Jordan.
\newblock Computational and statistical tradeoffs via convex relaxation.
\newblock {\em Proceedings of the National Academy of Sciences},
  110(13):E1181--E1190, 2013.

\bibitem{Chattergee12}
S.~Chatterjee.
\newblock Matrix estimation by universal singular value thresholding.
\newblock {\em The Annals of Statistics}, 43(1):177--214, 2014.

\bibitem{ChaudhuriGT12}
K.~Chaudhuri, F.~Chung, and A.~Tsiatas.
\newblock Spectral clustering of graphs with general degrees in the extended
  planted partition model.
\newblock In {\em Proceedings of the 25th Annual Conference on Learning Theory
  (COLT)}, pages 35.1--35.23, 2012.

\bibitem{chen2013incoherence_arxiv}
Y.~Chen.
\newblock Incoherence-optimal matrix completion.
\newblock {\em IEEE Transactions on Information Theory, to appear. Available on
  arXiv:1310.0154}, 2015.

\bibitem{Chen13}
Y.~Chen, A.~Jalali, S.~Sanghavi, and H.~Xu.
\newblock Clustering partially observed graphs via convex optimization.
\newblock {\em Journal of Machine Learning Research}, 15:2213--2238, June 2014.

\bibitem{Chen12}
Y.~Chen, S.~Sanghavi, and H.~Xu.
\newblock Clustering sparse graphs.
\newblock In {\em Proceedings of the {Neural Information Processing Systems
  Conferece}}, pages 2204--2212, 2012.

\bibitem{chen2014improved}
Y.~Chen, S.~Sanghavi, and H.~Xu.
\newblock Improved graph clustering.
\newblock {\em IEEE Transactions on Information Theory}, 60(10):6440--6455,
  2014.

\bibitem{ChenXu14}
Y.~Chen and J.~Xu.
\newblock Statistical-computational tradeoffs in planted problems and submatrix
  localization with a growing number of clusters and submatrices.
\newblock {\em arXiv:1402.1267. Presented at the 31st International Conference
  on Machine Learning, Beijing, China, June, 2014.}, 2014.

\bibitem{coja2004coloringSemirandom}
A.~Coja-Oghlan.
\newblock Coloring semirandom graphs optimally.
\newblock {\em Automata, Languages and Programming}, pages 383--395, 2004.

\bibitem{Condon01}
A.~Condon and R.~M. Karp.
\newblock Algorithms for graph partitioning on the planted partition model.
\newblock {\em Random Struct. Algorithms}, 18(2):116--140, Mar 2001.

\bibitem{Decelle11}
A.~Decelle, F.~Krzakala, C.~Moore, and L.~Zdeborova.
\newblock Asymptotic analysis of the stochastic block model for modular
  networks and its algorithmic applications.
\newblock {\em Physics Review E}, 84:066106, 2011.

\bibitem{Dekel10}
Y.~Dekel, O.~Gurel-Gurevich, and Y.~Peres.
\newblock Finding hidden cliques in linear time with high probability.
\newblock {\em Combinatorics, Probability and Computing}, 23(01):29--49, 2014.

\bibitem{Deshpande12}
Y.~Deshpande and A.~Montanari.
\newblock Finding hidden cliques of size $\sqrt{ N/e }$ in nearly linear time.
\newblock {\em Foundations of Computational Mathematics}, pages 1--60,
  September 2013.

\bibitem{durrett2007random}
R.~Durrett.
\newblock {\em Random Graph Dynamics}.
\newblock Cambridge University Press, New York, NY, 2007.

\bibitem{DyerFrieze89}
M.~E. Dyer and A.~M. Frieze.
\newblock The solution of some random {NP}-hard problems in polynomial expected
  time.
\newblock {\em Journal of Algorithms}, 10(4):451 -- 489, 1989.

\bibitem{Feldman2012statAlg}
V.~{Feldman}, E.~{Grigorescu}, L.~{Reyzin}, S.~{Vempala}, and Y.~{Xiao}.
\newblock {Statistical Algorithms and a Lower Bound for Planted Clique}.
\newblock {\em ArXiv e-prints}, Jan. 2012.

\bibitem{Fortunato10}
S.~Fortunato.
\newblock Community detection in graphs.
\newblock {\em Physics Reports}, 486(3):75--174, 2010.

\bibitem{grimmett1975colouring}
G.~R. Grimmett and C.~J.~H. McDiarmid.
\newblock On colouring random graphs.
\newblock In {\em Mathematical Proceedings of the Cambridge Philosophical
  Society}, volume~77, pages 313--324, 1975.

\bibitem{HajekWuXu14SDP}
B.~Hajek, Y.~Wu, and J.~Xu.
\newblock Achieving exact cluster recovery threshold via semidefinite
  programming.
\newblock {\em arXiv preprint arXiv:1412.6156}, 2014.

\bibitem{HajekWuXu14}
B.~Hajek, Y.~Wu, and J.~Xu.
\newblock Computational lower bounds for community detection on random graphs.
\newblock {\em arXiv preprint arXiv:1406.6625}, 2014.

\bibitem{HajekWuXu14SDP2}
B.~Hajek, Y.~Wu, and J.~Xu.
\newblock Achieving exact cluster recovery threshold via semidefinite
  programming: Extensions.
\newblock {\em arXiv:1502.07738}, 2015.

\bibitem{Hazan2011Nash}
E.~Hazan and R.~Krauthgamer.
\newblock How hard is it to approximate the best nash equilibrium?
\newblock {\em SIAM Journal on Computing}, 40(1):79--91, 2011.

\bibitem{Holland83}
P.~W. Holland, K.~B. Laskey, and S.~Leinhardt.
\newblock Stochastic blockmodels: First steps.
\newblock {\em Social Networks}, 5(2):109--137, 1983.

\bibitem{jalali2012maxnorm}
A.~Jalali and N.~Srebro.
\newblock Clustering using max-norm constrained optimization.
\newblock In {\em Proceedings of the 29th International Conference on Machine
  Learning}, pages 481--488, 2012.

\bibitem{Juel00cliqueCrypto}
A.~Juels and M.~Peinado.
\newblock Hiding cliques for cryptographic security.
\newblock {\em Designs, Codes and Cryptography}, 20(3):269--280, 2000.

\bibitem{koiran2012rip}
P.~Koiran and A.~Zouzias.
\newblock Hidden cliques and the certification of the restricted isometry
  property.
\newblock {\em IEEE Transactions on Information Theory}, 60(8):4999--5006,
  2014.

\bibitem{kolar2011submatrix}
M.~Kolar, S.~Balakrishnan, A.~Rinaldo, and A.~Singh.
\newblock Minimax localization of structural information in large noisy
  matrices.
\newblock In {\em Advances in Neural Information Processing Systems}, 2011.

\bibitem{Vilenchik13}
R.~Krauthgamer, B.~Nadler, and D.~Vilenchik.
\newblock Do semidefinite relaxations really solve sparse {PCA}?
\newblock {\em arXiv preprint arXiv:1306.3690}, 2013.

\bibitem{Kucera95}
L.~Ku\v{c}era.
\newblock Expected complexity of graph partitioning problems.
\newblock {\em Discrete Appl. Math.}, 57(2-3):193--212, Feb. 1995.

\bibitem{Lelarge13}
M.~Lelarge, L.~Massouli{\'e}, and J.~Xu.
\newblock Reconstruction in the labeled stochastic block model.
\newblock In {\em IEEE Information Theory Workshop (ITW)}, pages 1--5, 2013.

\bibitem{Leskovec08}
J.~Leskovec, K.~J. Lang, A.~Dasgupta, and M.~W. Mahoney.
\newblock Statistical properties of community structure in large social and
  information networks.
\newblock In {\em Proceedings of the 17th international conference on World
  Wide Web}, pages 695--704. ACM, 2008.

\bibitem{ma2013submatrix}
Z.~Ma and Y.~Wu.
\newblock Computational barriers in minimax submatrix detection.
\newblock {\em arXiv preprint arXiv:1309.5914}, 2013.

\bibitem{Massoulie13}
L.~Massouli{\'e}.
\newblock Community detection thresholds and the weak {Ramanujan} property.
\newblock In {\em Proceedings of the 46th Annual ACM Symposium on Theory of
  Computing}, pages 694--703. ACM, 2014.

\bibitem{mastom11}
L.~Massouli\'e and D.~Tomozei.
\newblock Distributed user profiling via spectral methods.
\newblock {\em Stochastic Systems}, 4:1--43, 2014.

\bibitem{mathieu}
C.~Mathieu and W.~Schudy.
\newblock Correlation clustering with noisy input.
\newblock In {\em Proceedings of the 21st Annual ACM-SIAM Symposium on Discrete
  Algorithms}, pages 712--728. SIAM, 2010.

\bibitem{matousek2008}
J.~Matou\v{s}ek and J.~Vondr\'{a}k.
\newblock The probabilistic method, lecture notes.
\newblock {\em Available at
  \url{http://kam.mff.cuni.cz/~matousek/prob-ln-2pp.ps.gz}}, 2008.

\bibitem{McSherry01}
F.~McSherry.
\newblock Spectral partitioning of random graphs.
\newblock In {\em 42nd IEEE Symposium on Foundations of Computer Science},
  pages 529 -- 537, Oct. 2001.

\bibitem{Mossel12}
E.~Mossel, J.~Neeman, and A.~Sly.
\newblock Stochastic block models and reconstruction.
\newblock {\em available at: http://arxiv.org/abs/1202.1499}, 2012.

\bibitem{Mossel13}
E.~Mossel, J.~Neeman, and A.~Sly.
\newblock A proof of the block model threshold conjecture.
\newblock {\em arxiv:1311.4115}, 2013.

\bibitem{Mossel14}
E.~Mossel, J.~Neeman, and A.~Sly.
\newblock Consistency thresholds for binary symmetric block models.
\newblock {\em Arxiv preprint arXiv:1407.1591}, 2014.

\bibitem{nadakuditi}
R.~R. Nadakuditi and M.~E.~J. Newman.
\newblock Graph spectra and the detectability of community structure in
  networks.
\newblock {\em Physical Review Letters}, 108(18):188--701, 2012.

\bibitem{Newman04}
M.~E.~J. Newman and M.~Girvan.
\newblock Finding and evaluating community structure in networks.
\newblock {\em Physical Review E}, 69:026113, Feb 2004.

\bibitem{oymak2012simultaneously}
S.~Oymak, A.~Jalali, M.~Fazel, Y.~C. Eldar, and B.~Hassibi.
\newblock Simultaneously structured models with application to sparse and
  low-rank matrices.
\newblock {\em arXiv preprint arXiv:1212.3753}, 2012.

\bibitem{Yu11}
K.~Rohe, S.~Chatterjee, and B.~Yu.
\newblock Spectral clustering and the high-dimensional stochastic blockmodel.
\newblock {\em The Annals of Statistics}, 39(4):1878--1915, 2011.

\bibitem{rossman2010clique}
B.~Rossman.
\newblock {\em Average-case complexity of detecting cliques}.
\newblock PhD thesis, Massachusetts Institute of Technology, 2010.

\bibitem{shabalin2009submatrix}
A.~A. Shabalin, V.~J. Weigman, C.~M. Perou, and A.~B. Nobel.
\newblock Finding large average submatrices in high dimensional data.
\newblock {\em The Annals of Applied Statistics}, pages 985--1012, 2009.

\bibitem{sun2013anova}
X.~Sun and A.~B. Nobel.
\newblock On the maximal size of large-average and {ANOVA}-fit submatrices in a
  gaussian random matrix.
\newblock {\em Bernoulli}, 19(1):275--294, 2013.

\bibitem{tropp2010matrixmtg}
J.~A. Tropp.
\newblock User-friendly tail bounds for sums of random matrices.
\newblock {\em Foundations of Computational Mathematics}, 12(4):389--434, 2012.

\bibitem{vershynin2010nonasym}
R.~Vershynin.
\newblock Introduction to the non-asymptotic analysis of random matrices.
\newblock In Y.~C. Eldar and G.~Kutyniok, editors, {\em Compressed Sensing},
  pages 210--268. Cambridge University Press, 2012.

\bibitem{verzelen2013sparse}
N.~Verzelen and E.~Arias-Castro.
\newblock Community detection in sparse random networks.
\newblock {\em arXiv preprint arXiv:1308.2955}, 2013.

\bibitem{vinayak2013sharp}
R.~K. Vinayak, S.~Oymak, and B.~Hassibi.
\newblock Sharp performance bounds for graph clustering via convex
  optimization.
\newblock In {\em 38th International Conference on Acoustics, Speech, and
  Signal Processing (ICASSP)}, 2014.

\bibitem{vu2014clustering}
V.~Vu.
\newblock A simple {SVD} algorithm for finding hidden partitions.
\newblock {\em arXiv preprint arXiv:1404.3918}, 2014.

\bibitem{vu2013fantope}
V.~Q. Vu, J.~Cho, J.~Lei, and K.~Rohe.
\newblock Fantope projection and selection: A near-optimal convex relaxation of
  sparse pca.
\newblock In {\em Advances in Neural Information Processing Systems}, pages
  2670--2678, 2013.

\bibitem{Hajek13}
J.~Xu, R.~Wu, K.~Zhu, B.~Hajek, R.~Srikant, and L.~Ying.
\newblock Jointly clustering rows and columns of binary matrices: Algorithms
  and trade-offs.
\newblock In {\em SIGMETRICS}, pages 29--41, 2014.

\bibitem{yun2014adaptive}
S.~Yun and A.~Proutiere.
\newblock Community detection via random and adaptive sampling.
\newblock In {\em Proceedings of The 27th Conference on Learning Theory}, 2014.

\end{thebibliography}

\end{document}